\documentclass{article}

\usepackage{microtype}
\usepackage{graphicx}
\usepackage{subfigure}
\usepackage{booktabs} 
\usepackage{diagbox}

\usepackage{fullpage}

\usepackage{parskip}

\usepackage{natbib}
\usepackage[backref = page]{hyperref}

\usepackage{algorithm, algorithmic}

\usepackage[T1]{fontenc}
\usepackage{amsmath,amssymb,amsfonts,amsthm,textcomp,graphicx,nicefrac,mathtools,mathrsfs, dsfont} 
\usepackage{enumitem}
\usepackage[utf8]{inputenc}
\usepackage{enumitem}

\usepackage[capitalize,noabbrev]{cleveref}
\usepackage{xcolor}
\definecolor{mydarkblue}{rgb}{0,0.08,0.45}
  \hypersetup{ %
    pdftitle={},
    pdfauthor={},
    pdfsubject={},
    pdfkeywords={},
    pdfborder=0 0 0,
    pdfpagemode=UseNone,
    colorlinks=true,
    linkcolor=mydarkblue,
    citecolor=mydarkblue,
    filecolor=mydarkblue,
    urlcolor=mydarkblue,
    }

\usepackage{wrapfig}

\newtheoremstyle{dotless}{}{}{\itshape}{}{\bfseries}{}{ }{}
\theoremstyle{dotless}

\theoremstyle{plain}

\newtheorem{myth}{Theorem}

\newtheorem{myprop}[myth]{Proposition}
\newtheorem{mylem}[myth]{Lemma}

\newtheoremstyle{named}{}{}{\itshape}{}{\bfseries}{.}{.5em}{#1 #3}
\theoremstyle{named}
\newtheorem*{namthm*}{Theorem}

\usepackage{chngcntr}

\crefname{myth}{Theorem}{Theorems} 

\newcounter{parentnumber}

\usepackage{enumitem}
\usepackage{graphicx}

\newcommand{\indi}{\mathds{1}}

\providecommand{\argmax}{\mathop\mathrm{arg\, max}}

\newcommand{\Thomp}{\textsc{TS} }
\newcommand{\UCB}{\textsc{UCB} }
\newcommand{\KLUCB}{\textsc{KL-UCB} }
\newcommand{\UCBV}{\textsc{UCB\textsubscript{V}} }
\newcommand{\BayesUCB}{\textsc{BayesUCB} }

\newcommand{\ThompNS}{\textsc{TS}}

\newcommand{\KLUCBNS}{\textsc{KL-UCB}}

\newcommand{\BayesUCBNS}{\textsc{BayesUCB}}

\newcommand{\good}[2]{\mathcal{G}^{#1}_{#2}}

\newcommand{\bbP}{\mathbb{P}}

\newcommand{\nuhat}{\widehat{\nu}^k_t}
\newcommand{\muhat}{\widehat{\mu}^k_t}

\newcommand{\Perm}{\Pi}

\newcommand{\hist}{\mathscr{H}}

\usepackage{csquotes}

\usepackage[textsize=tiny]{todonotes}

\title{Strategies for Safe Multi-Armed Bandits\\ with Logarithmic Regret and Risk }
\author{Tianrui Chen\\Boston University\\\texttt{trchen@bu.edu} 
\and
Aditya Gangrade\\Carnegie Mellon University\footnote{The bulk of this work was done whilst A.G. was a graduate student at Boston University}\\\texttt{agangra2@andrew.cmu.edu}
\and
Venkatesh Saligrama\\ Boston University\\\texttt{srv@bu.edu}}

\date{\vspace{-2\baselineskip}}
\begin{document}

\maketitle

\begin{abstract}

We investigate a natural but surprisingly unstudied approach to the multi-armed bandit problem under safety risk constraints. Each arm is associated with an unknown law on safety risks and rewards, and the learner's goal is to maximise reward whilst not playing unsafe arms, as determined by a given threshold on the mean risk. 

We formulate a pseudo-regret for this setting that enforces this safety constraint in a per-round way by softly penalising any violation, regardless of the gain in reward due to the same. This has practical relevance to scenarios such as clinical trials, where one must maintain safety for each round rather than in an aggregated sense.

We describe doubly optimistic strategies for this scenario, which maintain optimistic indices for both safety risk and reward. We show that schema based on both frequentist and Bayesian indices satisfy tight gap-dependent logarithmic regret bounds, and further that these play unsafe arms only logarithmically many times in total. This theoretical analysis is complemented by simulation studies demonstrating the effectiveness of the proposed schema, and probing the domains in which their use is appropriate.

\end{abstract}

\section{Introduction}
We consider the safety constrained multi-armed bandit problem, where each \emph{arm}, $k \in [1:K]$ is modelled by a tuple, consisting of a stochastic \emph{reward}, of mean $\mu^k,$ and an associated stochastic \emph{safety-risk}, of mean $\nu^k$. Upon playing an arm, the learner observes noisy instances of the reward and safety-risk. The learner is provided with a \emph{tolerated risk level}, denoted $\alpha,$ and the goal of the \emph{safe bandit problem} is to maximise the reward gained over the course of play, while ensuring that unsafe arms---those for which $\nu^k > \alpha$---are not played too often.

We propose the following \emph{regret} formulation to model the above criteria. Let $\mu^*$ be the mean reward of the largest safe action, i.e, the largest $\mu^k$ over arms such that $\nu^k \le \alpha.$ Let $A_t$ be the arm pulled by the algorithm at time $t$. We study \begin{equation} \label{eqn:regret_def} \mathcal{R}_T := \sum_{t \le T} \max( \mu^* - \mu^{A_t}, \nu^{A_t} - \alpha). \end{equation}
Before describing the results, let us sketch a scenario of particular interest, which informs our formulation. 

\noindent \textbf{Clinical Trials.} Trial drugs have both positive (eg. curing a disease) and negative side-effects (headaches, nausea, etc) on a patient in a clinical trial, and it is as much in the interest of a patient to ensure that negative side effects are limited as it is to ensure that the drug is effective \citep[e.g.][]{genovese2013efficacy}. This scenario motivates the problem of choosing drug and dosage (arms) that have the maximum positive response while ensuring that the side-effects remain below some threshold $\alpha$. Since each patient responds differently, the observed response and the manifestation of side-effects for a specific patient can be modelled as random-variables, with the corresponding means representing population averages. Importantly, for such a scenario, safety must be accounted for in a per-round sense - it does no good to alternate between assigning ineffective placebos and effective but harmful doses. Instead we need to ensure that individuals are not exposed to undue risk while accruing benefits.

\noindent \textbf{How does our formulation account for this scenario?}
\begin{itemize}[wide, leftmargin=8pt,nosep]
\item \emph{Risk Per Round.} Regret ensures that unsafe arms are rarely played in a per-round (per-patient) sense rather than ensuring safety in an overall sense--for any $k \neq k^*,$ at least one of  $\mu^{k} - \mu^*$ or $\nu^k - \alpha$ must be positive, and so benefits in efficacy due to unsafe dosages are discounted. 
\item \emph{Small safety violations are penalized less (smoothness).} Small violations of negative side-effects is a permissible risk (elevated nausea level than desired), worth taking on for a few patients, in the hope of finding a drug/dosage that is effective for the population. Our penalty on safety violations is smooth. 
\item \emph{Control of Cumulative risk and Violations} Since choosing an infeasible arm in any round contributes a constant amount to the regret, a small $\mathcal{R}_T$ further ensures that the cumulative safety risk and the cumulative safety violations (i.e. times such that $\nu^{A_t} > \alpha$) are also small. 
\end{itemize}
We next describe our main technical contributions. 

{\bf Four Optimistic Strategies.} We explore \emph{doubly optimistic} index-based strategies for choosing arms. These maintain optimistic indices for both the reward and safety risk of each arm, and proceed by first developing a set of plausibly safe actions using the safety indices, and then choose the arm with the highest reward index to play, thus encouraging sufficient exploration. In standard bandits there are two broad classes of such index-based strategies - those based on {\it frequentist confidence bounds}, and those based on {\it Bayesian posteriors.} This suggests four natural variants in the safe bandit case, through two choices for each of the reward and safety indices. We explicitly study three of these - first when both indices are frequentist, second when the safety index is left frequentist but the reward index is replaced by {\it Thompson sampling}, and finally when both indices are based on Bayesian methods. While left explicitly unstudied, the case of frequentist reward and Bayesian safety indices follows naturally from our analysis.

{\bf Logarithmic Regret Bounds.} In all cases, we show that these strategies admit strong gap-dependent logarithmic regret rates. Further each of these also ensure that the number of times any unsafe arm is played at all (i.e., $\sum \indi\{\nu^{A_t} > \alpha\}$) is similarly logarithmically bounded. Finally we show a lower bound which demonstrates that our regret bounds are tight in the limit of large time horizons. The proofs adapt existing results of bandit theory to argue that for well designed safety indices, the optimal arm $k^*$ always remains valid, but any unsafe arms are quickly eliminated. Further, so long as $k^*$ remains valid, standard approaches show that inefficient arms cannot be played too often. {\it An interesting consequence is that the play of strictly dominated arms - those that are both unsafe and inefficient - is limited by the larger of the two gaps.}

{\bf Empirical Results.} We complement the above theoretical study with simulations. First, we practically illustrate that prior policy-based approaches to the safe and constrained bandits do not yield favourable play in our scenario. Next, we implement our proposals, and both illustrate that the methods indeed meet the theoretical guarantees, and further contextualise their relative merits in a practical sense. The broad observation regarding the latter is that Thompson sampling based methods tend to offer better performance in terms of means.

\subsection{Related Work}

Bandit problems are exceedingly well studied, and a plethora of methods with subtle differences have been established. We refer the reader to the recent book of \citet{lattimore_szepesvari_book} for a thorough introduction.

We first describe prior approaches to constrained bandit problems from a formulational point of view. The most important aspect of this is that prior formulations tend to constrain play in an aggregate sense. This raises issues when we need to ensure safety in a per-round sense, as is illustrated by a running example. We then contextualise our methodological proposals with respect to the prior work, and finally discuss pure exploration in the safe-bandit setting.

\textbf{Globally Constrained Formulations} The theory of bandits with global constraints was initiated by \citet{bandits_with_knapsacks_2013}, and extended by \citet{agrawal2014_BwCR}. Specialised to our context, these works constrain the total number of adverse effects whilst matching the performance of the optimal dynamic policy that is aware of all means. More concretely, suppose that the safety risk observed is a random variable $S_t$. \citet{bandits_with_knapsacks_2013} enforce the hard constraint that $\left(\sum S_t - \alpha T\right) \le 0,$ while \citet{agrawal2014_BwCR} relax this into a second regret $\mathcal{S}_T = \max\left(0, \sum S_t - \alpha T\right),$ and ensure that this is small. 

Such aggregate safety formulation is lacking from our perspective, as is illustrated by the following simple example of two arms with means \begin{equation}\label{eq:example} (\mu^1, \nu^1) = (\nicefrac{1}{2}, 0), \quad (\mu^2, \nu^2) = (1,1).\end{equation} Due to the global constraint, the optimal dynamic policy is to pull arm $2$ for $\alpha T$ rounds, and then switch to pulling arm $1$. A low regret algorithm must then also pull arm $2$ $\Omega(T)$ times. However, such play undesirably exposes a linear number of rounds to the very unsafe action $2$. Our formulation instead would penalise every play of arm $2$ by a cost of $(1-\alpha),$ and thus effective schema would only play arm $2$ sublinearly many times. It should be noted that since the constraint is applied in a per-round way, the optimal dynamic policy in our case is supported on a single arm.

In passing, we also mention the conservative bandit problem \citep{wu2016conservative}, which only considers rewards, and enforces a running aggregate constraint that for any round $t$, $\sum_{s \le t} \mu^{A_s} \ge (1-\alpha) t \mu^{k_0}$. While an interesting variation, we note that such a running constraint on safety-risk would have similar issues as the above in our situation.

\textbf{Per-round Constraints} The recent work of \citet{pacchiano2021stochastic} studies the safe bandit problem with two crucial differences from us. Firstly, the action space is lifted from single arms to policies (i.e. distributions) over arms, denoted $\pi_t$, and secondly, the hard per-round constraint $\langle \pi_t ,\nu \rangle \le \alpha$ is enforced. Of course, actual arms are selected by sampling from $\pi_t$. The regret studied is $\sum \langle \pi^* - \pi_t, \mu\rangle,$ where $\pi^*$ is the optimal static safe policy, i.e., the maximiser of $\langle \pi, \mu\rangle$ subject to $\langle \pi, \nu \rangle \le \alpha.$ Exploration is enabled by giving the scheme an arm $k_s$ known a priori to be safe, and by spending the slack $\alpha - \nu^{k_s}$ as room for exploration in $\pi_t$. 

While ostensibly constrained at each round, this formulation suffers from similar issues as the previously discussed globally constrained formulations since the optimal static policy is only safe in aggregate. Indeed, in the previous example (\ref{eq:example}), the optimal $\pi^*$ is $(1-\alpha, \alpha),$ and so a low regret algorithm must place large mass on the unsafe arm $2$ in most rounds, therefore exposing about $\Omega(T)$ rounds to it.

A similar approach, but crucially without the policy action space, was taken by \citet{amani_safe_linear, moradipari2021safe_thompson} for in the linear bandit setting. These papers also study hard round-wise safety constraints, and again utilise a known safe action, as well as the continuity of the action space to enable sufficient exploration. We note that the particulars of the signalling model adopted by \citet{amani_safe_linear} paper preclude extending their results to the multi-armed setting, and while the model of \citet{moradipari2021safe_thompson} does admit such extension, the scheme proposed fundamentally relies on having a continuous action space with a linear safety-risk, and cannot be extended to multi-armed settings without lifting to policy space.


\textbf{Methodological Approaches} The bulk of the previous papers are based on frequentist confidence bounds, with two variants. Similar to our Alg.~\ref{alg:DOCB}, \citet{agrawal2014_BwCR} use doubly optimistic methods that maintain optimistic upper bounds on the rewards and lower bounds on the risk, and play the policy that maximises reward upper bounds while being safe with respect to the risk lower bounds. In contrast, \citet{pacchiano2021stochastic, amani_safe_linear, wu2016conservative} all use optimistic-pessimistic methods, which instead maintain upper bounds on both the rewards and safety risk and play the actions with maximum reward upper bound whilst being safe with respect to the stringent risk upper bounds. \citet{moradipari2021safe_thompson} take a similar pessimistic approach, but replace the reward upper bounds with a Thompson sampling procedure that is similar in spirit to our Alg.~\ref{alg:STOP}, although this uses optimistic safety indices. We also further study a fully Bayesian approach in Alg.~\ref{alg:TBU}.

\textbf{Pure Exploration with Safety} \citet{katz2018feasible, katz2019top} design procedures for finding the best \emph{feasible} arm based on a combination of optimistic and pessimistic confidence bounds that is typical of pure exploration approaches. An interesting variant of this problem was studied in a recent preprint of \citet{wang2021best}, who associate a continuous `dosage' parameter with each arm, now interpreted as a single drug, with the understanding that both reward and risk grow monotonically with dosage. These should be compared to the dose-finding bandit problem \citet{aziz2021multi}, which seeks to identify a dose level out of $K$ options that minimises $|\nu^k - \alpha|,$ with the intuition being that higher doses are more effective, and so should be maximised, but without exceeding the safety threshold by much. The dose-finding approach relies strongly on this assumed monotonicity. This models the scenario of a single drug, but is inappropriate for the setting of multiple drugs that are trialled together, which is better represented as a constrained optimisation problem (as studied by the former papers). Our formulation takes precisely this view, but from the perspective of controlling regret rather than identification. Note that our smooth penalty for safety violation, $\max(0,\nu^k - \alpha),$ bears similarities to the absolute value loss $|\nu^k - \alpha|,$ where again a small violation of safety is not penalised strongly.

\section{Definitions and Setup}

An instance of the \emph{safe bandit problem} is defined by a risk level $\alpha \in [0,1],$ a natural $K \ge 2$, corresponding to a number of arms, and a corresponding vector of probability distributions, $(\mathbb{P}^k)_{k \in [1:K]},$ each entry of which is supported on $[0,1]^2.$ We will represent the corresponding random vector as two components $(R,S),$ which are termed the reward and safety-risk of a draw from $\mathbb{P}^k$. We further associate two vectors $\mu, \nu \in [0,1]^K,$ corresponding to the \emph{mean reward and safety-risk} of each arm, i.e \begin{align*}
    (\mu^k, \nu^k) := \mathbb{E}_{(R,S) \sim \mathbb{P}^k}[(R,S)].
\end{align*}  
$R$ and $S$ need not be independent - this has little effect on the subsequent study, since each is marginally bounded.

The scenario proceeds in rounds, denoted $t \in \mathbb{N}$. At each $t$, the learner (i.e.~an algorithm for the bandit problem) must choose an \emph{action} $A_t \in [1:K],$ corresponding to `pulling an arm.' Upon doing so, the learner receives samples $(R_t, S_{t}) \sim \mathbb{P}^{A_t}$ independently of the history. The learner's \emph{information set} at time $t$ is $\hist_{t-1} = \{(A_s, R_s, S_s): s < t\},$ and the action $A_t$ must be adapted to the filtration induced by these sets. The learner is unaware of any properties of the laws $\mathbb{P}^k$ beyond the fact that they are supported on $[0,1]^2$. 

The \emph{competitor}, representing the \emph{best safe arm} given the safety constraint and the mean vectors, is defined as \[ k^* = \argmax_{k \in [1:K]} \mu^k \textrm{ s.t. } \nu^k \le \alpha, \] and its mean reward and safety risk are denoted as $\mu^*, \nu^*$. We will use this convention throughout - for any symbol $\mathfrak{s}^k,$ we set $\mathfrak{s}^* = \mathfrak{s}^{k^*}.$ We can ensure that the problem is feasible by including a no-reward, no-risk arm of means $(0,0)$ - this might correspond to a placebo in a clinical trial. Without loss of generality, we will assume that $k^*$ is unique. We define the \emph{inefficiency gap} $\Delta^k$ and the \emph{safety gap} $\Gamma^k$ of playing an arm $k$ as \begin{align*}
    \Delta^k := 0 \vee (\mu^* - \mu^k),\quad     \Gamma^k := 0 \vee (\nu^k - \alpha),
\end{align*} where $a \vee b := \max(a,b),$ and we will also use $a \wedge b :=\min(a,b)$. Note that $\Delta^k \vee \Gamma^k > 0$ for $k \neq k^*$.

The performance of a learner for the safe bandit problem is measured by the (pseudo-) \emph{regret} of (\ref{eqn:regret_def}), which may also be written as \( \mathcal{R}_T := \sum_{1 \le t \le T} \Delta_{A_t} \vee \Gamma_{A_t}. \)

Further, with each arm $k$, we associate state variables $N_t^k$ denoting the number of times it has been played up to time $t$, and $R_t^k, S_t^k$ denoting the total rewards and safety risk incurred on such rounds. More formally, \begin{align*}
    N_t^k := \sum_{s < t} \indi\{ &A_t = k\},\\
    R_t^k := \sum_{s < t} \indi\{A_t = k\} R_t, \quad &\textit{\&}\quad    S_t^k := \sum_{s < t} \indi\{A_t = k\} S_t.
\end{align*}
Similarly, $N_t^*, R_t^*, S_t^*$ denote the corresponding variables for $k^*$. Notice that $\mathcal{R}_t = \sum_{k \neq k^*} (\Delta^k \vee \Gamma^k) N_t^k.$ We also use the notation $\muhat := R_t^k/N_t^k, \nuhat := S_t^k/N_t^k$.

Since controlling it is of natural interest, we define the number of times an unsafe arm is played as \[ \mathcal{U}_T := \sum_t \indi\{\nu^{A_t} > \alpha\}. \]

Finally, for $a,b \in [0,1],$ we use the notation \[d(a\|b) := a \log \frac{a}{b} + (1-a) \log \frac{1-a}{1-b} \] to denote the KL divergence between Bernoulli laws with means $a$ and $b.$ We will also need the notation \begin{align*}
    d_<(a\|b) := d(a\|b) \indi\{a < b\},\\ d_>(a\|b) := d(a\|b) \indi\{a > b\}.
\end{align*} 

\textbf{Remark} While the formulation focuses on a single safety-constraint, this may be extended. For example, we may posit a safety-risk vector $S\in [0,1]^d$, and demand that the corresponding (vector) means $\nu^k$ should lie in some known safe set $\mathcal{S}.$ Natural extensions of the methods below would control, e.g., $\sum \max(\mu^{A_t} - \mu^*, \mathrm{dist}(\nu^{A_t},\mathcal{S})).$ We focus on a single constraint for clarity and ease of exposition.

\section{Doubly Optimistic Confidence Bounds}\label{sec:DOCB}

The use of optimistic confidence bounds is well established in standard bandits \citep[e.g. Ch. 7-10][]{lattimore_szepesvari_book}. The idea is that pulling according to the maximum optimistic bound on the means encourages exploration, while efficiency follows because the confidence bounds exploit information to shrink towards the means, eventually giving evidence for the inefficiency of suboptimal arms.

\begin{wrapfigure}[15]{r}{0.5\textwidth}
\vspace{-2\baselineskip} \begin{minipage}{0.5\textwidth}
\begin{algorithm}[H]
   \caption{Doubly Optimistic Confidence Bounds}
   \label{alg:DOCB}
   
\begin{algorithmic}[1] 
   \STATE \textbf{Input}: $K,$ functions $U, L$.
   \STATE \textbf{Initialise}: $\hist_0 \gets \varnothing$
   \FOR{$t = 1, 2, \dots$}
   \IF{$t \le K$}
   \STATE $A_t \gets t$
   \ELSE
   \STATE $\forall k, L^k_t \gets L(t, \hist_{t-1}, k).$ 
   \STATE $\Perm_t \gets \{k : L^k_t \le \alpha\}.$ 
   \STATE $\forall k \in \Perm_t, U^k_t \gets U(t, \hist_{t-1}, k).$
   \STATE $A_t \gets \argmax_{k \in \Perm_t} U^k_t.$
   \ENDIF
   \STATE Pull $A_t,$ receive $(R_t, S_t) \sim \mathbb{P}^{A_t}$.
   \STATE Update $\hist_t \gets \hist_{t-1} \cup \{(A_t, R_t, S_t)\}$. 
   \ENDFOR
\end{algorithmic}
\end{algorithm} 
    \end{minipage}
\end{wrapfigure}

The idea behind doubly optimistic bounds is identical - we maintain lower bounds on safety-risk $L^k_t$ and upper bounds on rewards $U^k_t$ such that $L^k_t \le \nu^k$ and $U^k_t \ge \mu^k$ with high probability. We then construct a set of `permissible arms' $\Perm_t := \{k : L^k_t \le \alpha\}$ - these are all the arms that are plausibly feasible given the information we have up to time $t$. $A_t$ is selected to maximise $U^k_t$ amongst $k \in \Perm_t$. The optimism of $\Perm_t$ allows us to explore for high rewards, but the concentration of $L^k_t$ as $N_t^k$ grows serves to identify unsafe arms, which then cease to be pulled. The broad scheme is described in Algorithm \ref{alg:DOCB}.

This scheme can be analysed using a variation of the standard bandit analysis. To control the play of unsafe arms, we argue that $\nu^k - L^k_t$ is bounded as $\sqrt{\log(T)/N_t^k}$. Thus, if $\nu^k > \alpha,$ the arm $k$ should fall out of $\Perm_t$ after it has been played at most $O(\log(T)/(\Gamma^k)^2)$ times. Next we argue that the bounds are `consistent' (or optimistic) with high probabiliy, that is, most of the time $L^{*}_t \le \nu^* \iff k^* \in \Perm_t$ and $U_t^* \ge \mu^*$. Given this, in order to play arm $k$, $U_t^k$ must exceed $\mu^*,$ but $U^k_t - \mu^k$ shrinks as $O(\sqrt{\log T/N_t^k})$ bounding $N_T^k$ as $O(\log(T)/(\Delta^k)^2)$. In the process, strictly dominated $k$ - for which $\nu^k > \alpha$ and $\mu^k < \mu^*,$ are doubly penalised, and their play is limited by the larger gap. 

We will explicitly analyse the scheme by instantiating the method with bounds based on \KLUCB \citep{garivier2011kl}, which offer optimal mean-dependent regret control for standard bandits. Note that the study of confidence bounds for bandit methods is mature, and our results can be improved with other choices of such bounds, e.g.~, using variance sensitive bounds such as \textsc{Empirical-KL-UCB} \citep{cappe2013kullback} or \UCBV \citep{audibert2009exploration}.

The \KLUCB type bounds take the following form \begin{align*}
    \gamma_t &:= \log( t (\log(t))^3),\\
    U(t, \hist_{t-1}, k) &:= \max \{ q > \muhat : d(\muhat \|q) \le \gamma_t/N_t^k \},\\
    L(t, \hist_{t-1},k) &:= \min \{q < \nuhat : d( \nuhat \|q) \le \gamma_t /N_t^k\},
\end{align*} where $\gamma_t$ trades-off the width and consistency of $U,L$. These bounds are natural for Bernoulli random variables, and since these are the `least-concentrated' law on $[0,1],$ the fluctuation bounds extend to general random variables. Using these, we show the following result in \S\ref{appx:gap_dep_klucb_pf}.


\begin{myth}\label{thm:pure_klucb}
Algorithm \ref{alg:DOCB} instantiated with \KLUCB type bounds attains the following for any $T$ and any $\varepsilon > 0.$ \begin{align*} \mathbb{E}[\mathcal{R}_T] \le \sum_{k\neq k^*}  & \frac{(1 + \varepsilon) (\Delta^k \vee \Gamma^k) \log T }{d_<(\mu^k\|\mu^*) \vee d_>(\nu^k \|\alpha) } + \xi_k, \end{align*} where $\xi_k =   O(\log\log T + \varepsilon^{-2})$.
Further, the number of times an unsafe arm is played is bounded as \[ \mathbb{E}[\mathcal{U}_T] \le  \sum_{k : \Gamma^k > 0} \left( \frac{(1+\varepsilon) \log T}{d_<(\mu^k\|\mu^*) \vee d_>(\nu^k\|\alpha)}\right)  + \xi_k. \]
\end{myth}
The $O$ in the above hides instance-dependent constants, the most pertinent of which is a dependence on $(\Delta^k \vee \Gamma^k)^{-3}$ with the $\varepsilon^{-2}$ term. To ameliorate this, we also give a gap-independent analysis of the scheme in \S\ref{appx:klucb_gap_indep_pf}. \begin{myth}\label{thm:pure_klucb_gap_indep}
    Algorithm \ref{alg:DOCB} instantiated with \KLUCB attains \[\mathbb{E}[\mathcal{R}_T] \le \sqrt{28 K T \log T} + 6 K \log\log T + 32  .\]
\end{myth}
The above statement extends to \KLUCB for standard bandits upon sending $\alpha \to 1$, which, surprisingly, appears to have been unobserved, at least explicitly. 


\section{Bayesian Methods}

Thompson Sampling (\ThompNS) is the first proposed method for bandit problems \citep{thompson1933likelihood}, and encourages exploration by using randomisation. The idea is to choose an benign prior, and play arms according to their posterior probability being optimal. The posteriors remain flat for insufficiently explored arms, giving a non-trivial chance of pulling them. An advantage of \Thomp lies in the fact that it exploits a posterior that may be much better adapted to the underlying law $\mathbb{P}^k$ than confidence bounds that rely on a few simple statistics. Indeed, it has been empirically observed that \Thomp offers improved regret versus comparable \UCB methods in multi-armed bandits \citep{chapelle2011empirical}.

This section explores the use of Bayesian methods for safe bandits. We start by replacing the \KLUCB based selection of arms to play in Algorithm \ref{alg:DOCB}, but retaining the construction of $\Perm_t$. We then study a Bayesian method of selecting $\Perm_t.$

In the subsequent, we restrict analysis to the case of Bernoulli bandits, i.e., where the laws $\mathbb{P}^k$ are such that marginally $R \sim \mathrm{Bern}(\mu^{A_t})$ and $S \sim \mathrm{Bern}(\nu^{A_t}).$ We note that since the resulting bounds depend on only the means of the rewards and safety-risk, these bounds extend to generic laws supported on $[0,1]^2$ - indeed, as observed by \cite{agrawal2012analysis_TS_Bern_equals_general}, one can exploit an algorithm for Bernoulli bandits for generic laws by passing to the algorithm two samples $\widetilde{R}_t \sim \mathrm{Bern}(R_t), \widetilde{S}_t \sim \mathrm{Bern}(S_t)$. The corresponding $\widetilde{R}, \widetilde{S}$ are then Bernoulli with the same means, and any guarantee that only depends on the means for the Bernoulli case extends to the underlying bandit problem. Of course, such a procedure may blow up variances, and thus be profligate in the case of highly concentrated instances.

Note: the methods described below admit essentially the same guarantees as the bounds of Theorems \ref{thm:pure_klucb} and \ref{thm:pure_klucb_gap_indep}. For the sake of brevity, we suppress the explicit bounds on $\mathbb{E}[\mathcal{U}_T]$ and the gap-independent bounds in the following.

\subsection{Thompson Sampling with Optimistic Safety Indices}\label{sec:STOP}

For Bernoulli bandits, it is natural to use the $\mathrm{Beta}$ family for priors, due to favourable conjugacy. The standard form of \Thomp instantiates each arm with the uninformative prior $\mathrm{Beta}(1,1) = \mathrm{Unif}[0,1].$ The corresponding posterior at time $t$ is $\mathrm{Beta}(R^k_t + 1, N^k_t - R^k_t + 1)$.

\begin{wrapfigure}[16]{r}{0.51\textwidth}
\vspace{-2\baselineskip} \begin{minipage}{0.51\textwidth}
\begin{algorithm}[H]
   \setlength{\textfloatsep}{0pt}
   \caption{Thompson Sampling With Optimistic Safety Indices (\textsc{topsi}) for Bernoulli Bandits}
   \label{alg:STOP}
\begin{algorithmic}[1]
   \STATE \textbf{Input}: $K,$ function $L$.
   \STATE \textbf{Initialise}: $\hist_0 \gets \varnothing.$
   \FOR{$t = 1, 2, \dots$}
   \IF{$t \le N$}
   \STATE $A_t \gets t$
   \ELSE
   \STATE $\forall k, L^k_t \gets L(t, \hist_{t-1}, k).$ 
   \STATE $\Perm_t \gets \{k : L^k_t \le \alpha\}.$ 
   \STATE $\forall k \in \Perm_t,$ sample $\rho_t^k \sim \mathrm{Beta}(R_t^k +1, N_t^k - R_t^k + 1)$
   \STATE $A_t \gets \argmax_{k \in \Perm_t} \rho^k_t.$
   \ENDIF
   \STATE Pull $A_t,$ receive $(R_t, S_t) \sim \mathbb{P}^{A_t}$.
   \STATE Update $\hist_t \gets \hist_{t-1} \cup \{(A_t, R_t, S_t)\}$. 
   \ENDFOR
\end{algorithmic}
\end{algorithm}
    \end{minipage}
\end{wrapfigure}

Algorithm \ref{alg:STOP} describes the proposed strategy - we retain the optimistic lower bound from \cref{alg:DOCB}, but replace the arm selection given $\Perm_t$ to a \Thomp strategy: random scores $\rho_t^k$ are drawn from the posterior for each arm in $\Perm_t,$ the arm with the largest $\rho^k_t$ is pulled.


The analysis of such a method is simple, \emph{given an analysis of \Thomp} for standard bandits. Indeed, we can control the play of infeasible arms as we did for \cref{alg:DOCB}. Further, as long as we can ensure $k^* \in \Perm_t$ with high probability, we can invoke the decomposition \begin{align*}
    \mathbb{E}[N_T^k] \le  \sum_{t} \mathbb{P}(k^* \not\in\Perm_t) + \mathbb{P}(  A_t = k, k^* \in \Perm_t).
\end{align*}
The first term is handled using the consistency of the lower bound $L^*_t$. The second term is essentially the term analysed for standard bandits, and we can use any analysis of \Thomp to control it. We concretely use the approach of \citet{agrawal2013further} in \S\ref{appx:TS+UCB} to show the following result. 

\begin{myth}\label{thm:stop_gap_dep}
    For Bernoulli Bandits, Algorithm \ref{alg:STOP} instantiated with a \KLUCB type confidence bound attains the following regret bound for any $T$ and any $\varepsilon > 0$ \begin{align*} \mathbb{E}[\mathcal{R}_T] \le  \sum_{k\neq k^*}  & \frac{(1 + \varepsilon) (\Delta^k \vee \Gamma^k) \log T }{d_<(\mu^k\|\mu^*) \vee d_>(\nu^k \|\alpha) } + \xi_k,\end{align*} where $\xi_k = O(\log \log T +  \varepsilon^{-2}\log(\nicefrac1\varepsilon))$
\end{myth}

\subsection{Thompson Sampling with \BayesUCB} \label{sec:TUB}

While Algorithm \ref{alg:STOP} admits a tight analysis, it still uses the potentially loose frequentist bound to decide $\Perm_t,$ and it is possible that using the posteriors on the safety-risks to do this may improve the behaviour. 

It is tempting to appeal to the basic structure of Thompson sampling, and associate a posterior with the safety risk of $P_{t,\nu}^k = \mathrm{Beta}(S^k_t +1, N^k_t - S^k_t + 1),$ sample safety scores $\sigma_t^k \sim P_{t,\nu}^k,$ and let $\Perm_t = \{k : \sigma_t^k \le \alpha\}.$ However, this attempt is misguided, essentially because we need to compare the scores to a fixed level $\alpha,$ rather than amongst each other. Indeed, if it is the case that $\nu^* = \alpha,$ then there is a constant chance that $\sigma_t^* > \alpha,$ even if the empirical mean $\widehat{\nu}^*_t$ is faithful. This would mean a constant chance of playing a suboptimal arm, and so linear regret. A similar issue has been observed with trying to analyse \Thomp using the analysis developed for \textsc{UCB}-type schema \citep{kaufmann2012thompson}, but the issue is now at the level of the scheme rather than an analysis. Indeed, we show via simulations that when $\nu^* = \alpha,$ such a scheme suffers linear expected regret (\S\ref{appx:exp_ts_with_no_slack}). 

So, this idea needs a fix. One natural attempt is to introduce a slack, say some $\beta_t^k,$ such that $\Perm_t = \{ \sigma_t^k \le \alpha + \beta_t^k\}$. This $\beta_t^k$ should likely decay as $N_t^k$ rises, but be large enough to ensure that $k^* \in \Perm_t$ - this is similar to the analytical approach taken by \citet{kaufmann2012thompson}. However, in designing such a $\beta_t^k,$ we are functionally designing a confidence bound, somewhat defeating the purpose.

\begin{wrapfigure}[18]{r}{0.5\textwidth}
\vspace{-2\baselineskip} \begin{minipage}{0.5\textwidth}
\begin{algorithm}[H]
   \caption{Thompson Sampling with \BayesUCB (\textsc{tsbu}) for Bernoulli Bandits}
   \label{alg:TBU}
\begin{algorithmic}[1]
   \STATE \textbf{Input}: $K,$ schedule $\delta_t^k$. 
   \STATE \textbf{Initialise}: $\hist_0 \gets \varnothing.$
   \FOR{$t = 1, 2, \dots$}
   \STATE $\forall k$ 
   \IF{$S_t^k = 0$} \STATE $L_t^k \gets 0$
   \ELSE \STATE $L^k_t \gets Q(\mathrm{Beta}(S_t^k, N_t^k - S_t^k + 1), \delta_t^k).$ 
   \ENDIF
   \STATE $\Perm_t \gets \{k : L^k_t \le \alpha\}.$ 
   \STATE $\forall k \in \Perm_t,$ sample $\rho_t^k \sim \mathrm{Beta}(R_t^k +1, N_t^k - R_t^k + 1)$
   \STATE $A_t \gets \argmax_{k \in \Perm_t} \rho^k_t.$
   \STATE Pull $A_t,$ receive $(R_t, S_t) \sim \mathbb{P}^{A_t}$.
   \STATE Update $\hist_t \gets \hist_{t-1} \cup \{(A_t, R_t, S_t)\}$. 
   \ENDFOR
\end{algorithmic}
\end{algorithm}
    \end{minipage}
\end{wrapfigure}

We take a different tack, and instead use a \emph{Bayesian confidence bound}, essentially exploiting the \BayesUCB method of \citet{kaufmann2012bayesUCB}. The idea is to choose a $\delta_t^k$th quantile of the posterior $P_{t, \nu}^k$ as a score, where $\delta_t^k$ is a schedule that decays with $t$. This is able to exploit the potentially improved adaptivity of the posterior, but due to $\delta_t^k$ being small, would continue to produce an optimistic score, and so have a high chance of $k^* \in \Perm_t$ at any time. Additionally, due to the concentration of the $\mathrm{Beta}$-law for large $N_t^k,$ the score of unsafe arms would converge towards $\nu^k,$ and thus preclude their play beyond a point. Altogether, the method seems tailor-made for our situation of filtering arms at a given level. The scheme is described in Algorithm \ref{alg:TBU}, where $Q(P,\delta)$ denotes the $\delta$th quantile of the law $P$. We introduce a slight bias in the same for technical convenience.

The main design parameter is $\delta_t^k$ which trades off the consistency and tightness. In our argument, we use a conservative choice of $\delta_t^k = (\sqrt{8 N_t^k} t\log^3 t)^{-1},$ which leads to a simplified proof, but introduces the inefficiency of $\nicefrac{2}{3}$ in the bounds below. We find that in simulations, the uniform choice $\delta_t^k = 1/(t+1)$ is better (\S\ref{appx:bayesUCB}), and perhaps an improved analysis can establish better bounds for such a schedule. The following summarises our analysis in \S\ref{appx:ts+bayes_ucb}.

\begin{myth}\label{thm:TS+BayesUCB}
    For Bernoulli bandits, Algorithm \ref{alg:TBU}, instantiated with $\delta_t^k = (\sqrt{8 N_t^k}t\log^3 t)^{-1}$ attains the following regret bound for any $\varepsilon > 0 $ and any $T:$ \begin{align*}  \mathbb{E}[\mathcal{R}_T] \le  \sum_{k\neq k^*}  & \frac{(1 + \varepsilon) (\Delta^k \vee \Gamma^k) \log T }{d_<(\mu^k\|\mu^*) \vee \nicefrac{2}{3} \cdot d_>(\nu^k \|\alpha) } + \xi_k,\end{align*} where $\xi_k = O(\log \log T +  \varepsilon^{-2}\log(\nicefrac1\varepsilon))$
\end{myth}

\section{Lower Bound}

We conclude our theoretical study with a lower bound for algorithms that admit sub-polynomial regrets against all bounded distributions. This is based on the technique of \citet{garivier2019lower_bound}, who use the chain rule of KL divergence and the data processing inequality to show the following relation, which extends to our case without change: \begin{mylem}\label{lemma:lower_bound_tech} For any safe bandit algorithm, and any two safe bandit instances $\{\bbP^k\}, \{\widetilde{\bbP}^k\},$ and any $T$, \[ \sum  \mathbb{E}[ N^k_T] D(\bbP^k\|\widetilde\bbP^k) \ge d( \mathbb{E}[N^k_T/T] \| \widetilde{\mathbb{E}}[N^k_T/T]).\]  \end{mylem} This lemma enables a standard approach - pick $\widetilde\bbP$ so that $\widetilde{\mathbb{E}}^k[(R,S)] =  (\mu^k \vee \mu^* + \varepsilon, \nu^k \wedge \alpha)$, and leave the other $\bbP^k$s unchanged. For any bandit algorithm with sub-polynomial mean regret, the right hand side grows as $\log(T),$ while the left hand side side reduces to $\smash{\mathbb{E}[N^k_T] D(\bbP^k\| \widetilde{\bbP}^k)}$. Of course, the optimal choice of $\smash{\widetilde\bbP}$ depends subtly on the details of $\bbP.$ We study a simple concrete case to illustrate that our prior analyses are tight. \begin{myprop}\label{thm:lower_bound}
    Any algorithm that ensures that, uniformly over all instances of safe Bernoulli bandit problems with independent rewards and safety-risks, the mean number of plays of any suboptimal arm is bounded as $O(T^x)$ for every $x \in (0,1)$ must satisfy \[ \varliminf_{T \nearrow \infty} \frac{\mathbb{E}[N^k_T]}{\log T} \ge \frac{1}{d_<(\mu^k\|\mu^*) + d_>(\nu^k\|\alpha)}.\]
\end{myprop}
Since mean regret can be expressed in terms of $\mathbb{E}[N_T^k],$ this also lower bounds regret. Note the sum in the denominator, rather than a max as in our upper bounds. This means that for strictly dominated arms (i.e. $k : \Delta^k > 0, \Gamma^k > 0$), our scheme may be loose by up to a factor of two. This arises since our scheme does not utilise the dependence structure of $(R,S),$ and represents an opportunity for future work.

\section{Simulations}\label{sec:sims}

We provide practical contextualisation for the schema described in the theoretical section using simulation studies over small safe bandit environments. Of course, to concretely study the schema we describe, we need to instantiate them with appropriate confidence bounds. We will do so using the \KLUCB and \BayesUCB based indices which we analysed previously. Finally, we use \Thomp instantiated with the Beta priors as described in the text. Of course, a variety of other methods can be implemented in these schema, but we believe that these methods serve well to illustrate both the theory and a first order practical design. All implementation details are left to \S\ref{appx:sims}.

We will begin by empirically illustrating that the prior policy based methods are indeed ineffective in our scenario, and play unsafe arms far too often. We then illustrate the performance of the methods on a realistic problem instance. Finally, we will and investigate the dependence of regret of the three methods on the gaps to the optimal arm.



\subsection{Empirical Demonstration of the Ineffectiveness of Prior Formulations}\label{sec:exp_policy}

As discussed previously, the globally constrained \citep{bandits_with_knapsacks_2013, agrawal2014_BwCR} and policy level \citep{pacchiano2021stochastic} formulations are unsatisfactory in the context of safety-constraints, as illustrated by example (\ref{eq:example}). Nevertheless, a priori it may be possible that the schema designed for these objectives may be effective in our scenario, especially if there exist optimal policies supported on a single arm. We implement the the doubly optimistic policy method (\textsc{BwCR}) of \citet{agrawal2014_BwCR}, and the optimistic-pessimistic method (\textsc{Pess}) of \citet{pacchiano2021stochastic} to demonstrate that this is untrue.

We explore two illustrative cases, both of which are for Bernoulli bandits with independent means and safety-risks. The data reported is across $100$ trials of horizon $50000$. Since these policy methods are based on confidence bounds, we also compare them to Alg.~\ref{alg:DOCB}. In all cases we instantiate these schema with \KLUCBNS-based confidence bounds.

\textbf{1. Multiple optimal policies.} We consider four arms with \begin{align}
    \mu = (0,0.4, 0.5, 0.6),
    \nu = (0, 0.4, 0.5, 0.6), \alpha = 0.5. \notag
    \end{align} The $(0,0)$ arm is included as a known safe arm, which is required for \textsc{Pess} to enable sufficient exploration. Notice that in this case there are two optimal static policies - one that is entirely supported on arm $3$, while another that is uniformly supported on arms $2$ and $4$. However, which one of these two policies these schema converge to is essentially random, and we thus see linear growth of $\mathbb{E}[\mathcal{U}_t]$ in Fig.~\ref{fig:compare_to_policy}.

\textbf{2. Single arm optimal policy.} We jack up the rewards of arm $3$ to $0.6,$ but leave the other means unchanged. Now the optimal policy is singly supported on arm $3$ and has a significant gap of $0.1$. Despite the fact that such a case is the most promising for policy-based methods in terms of efficacy in our formulation, Fig.~\ref{fig:compare_to_policy} again shows that they do rather poorly - for instance, while our implementation play the unsafe arms about 550 times, these methods play it at least 8000 times. This occurs because the policy-based methods are designed for the much richer policy space---a simplex---and so must explore a lot more than methods designed for single arm play. We note that in this case, while \textsc{BwCR} plays unsafe arms more often, it suffers less regret than \textsc{Pess}, since the unsafe arm incurs a smaller loss.

\begin{figure}[H]
    \centering
    \includegraphics[width = 0.49\linewidth]{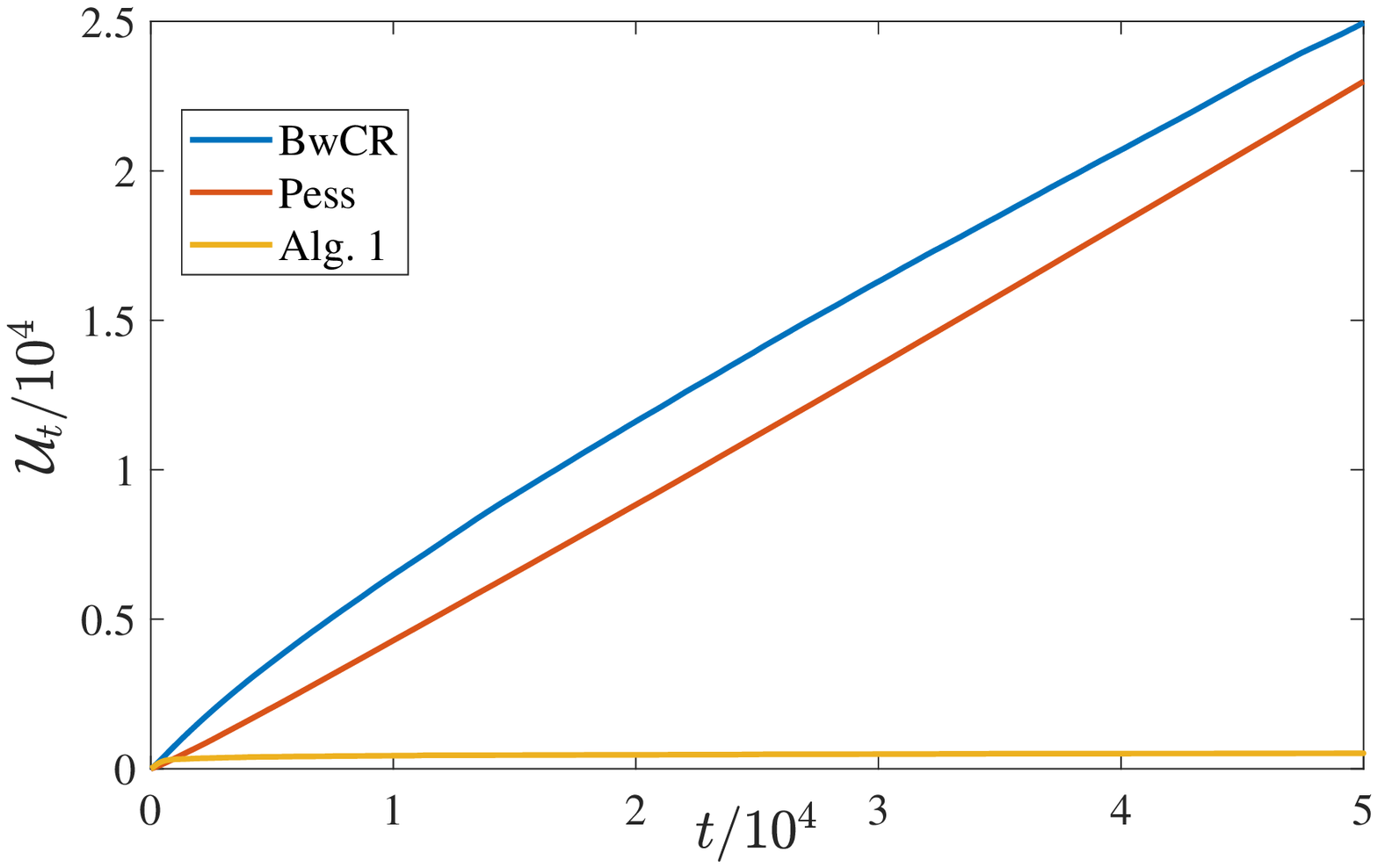}~
    \includegraphics[width = 0.49\linewidth]{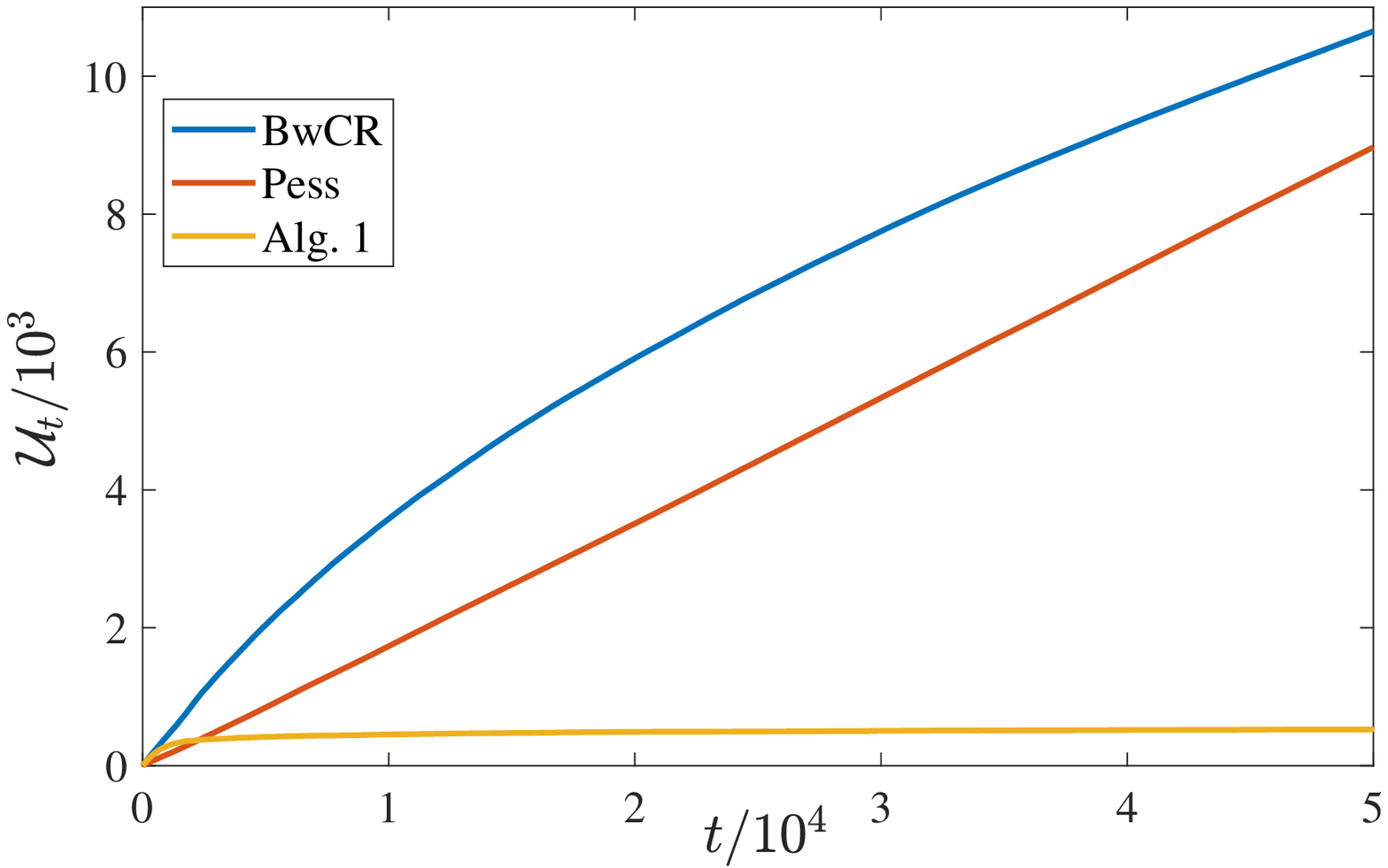}
    \caption{Empirical means of $\mathcal{U}_t$ versus $t$ averaged over $100$ trials over $t \in [1:50000].$ Left is case 1 (with multiple optimal policies), right case 2 (with a single optimal policy supported on one arm).} \vspace{-\baselineskip}
    \label{fig:compare_to_policy}
\end{figure}

\subsection{Characterisation of the Proposed Schema}\label{sec:exp_probing}

We implement the three methods to establish a practical contextualisation of their performance, and to verify the theoretical claims. For the sake of realism, we use the data of \citet{genovese2013efficacy}, who report efficacy and infection rates from a phase 2 randomised trial for various dosages of a drug to treat rheumatoid arthritis. The dosages studied were $(0, 25, 75, 150, 300)$ mg, and the observations were \begin{align*}
    \mu &= (0.360, 0.340, 0.469, 0.465, 0.537),\\
    \nu &= (0.160, 0.259, 0.184, 0.209, 0.293).
\end{align*}
This data is challenging for any safety level - no matter the choice, we have to deal with either a potential safety gap of order $10^{-2}$, or an efficacy gap of $10^{-3}$, both of which contribute a large regret. We study the safety level $0.21,$ under which arm $3$ is optimal, while arms $2,5$ are unsafe. We chose this to allow large enough safety gaps that the behaviour of $\mathcal{U}_T$ is easy to establish with runs of length about $50K$ - if we took $\alpha$ smaller, say $0.2,$ then we would expect to need runs of length $100K$ simply to reach a point at which arm $4$ is played fewer than about a third of the time. This consideration also illustrates why the regret $\mathcal{R}_T$ is a much more reasonable notion of study than $\mathcal{U}_T,$ which can grow very large due to tiny, practically undetectable safety gaps. Plots for a run with $\alpha = 0.19$ are included in \S\ref{appx:probing}.

\begin{figure}[tb]
    \centering
    \includegraphics[width = 0.49\linewidth]{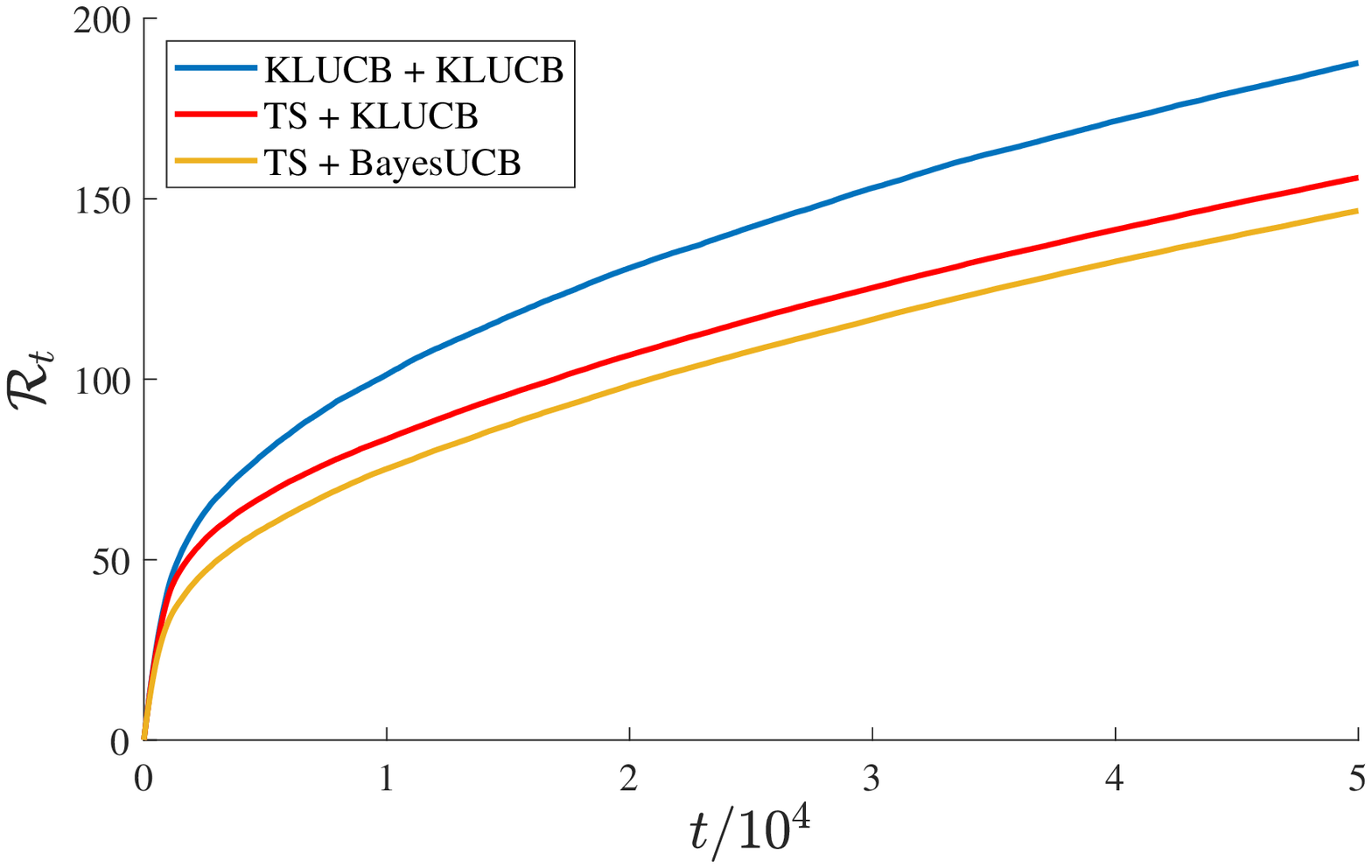}~
    \includegraphics[width = 0.49\linewidth]{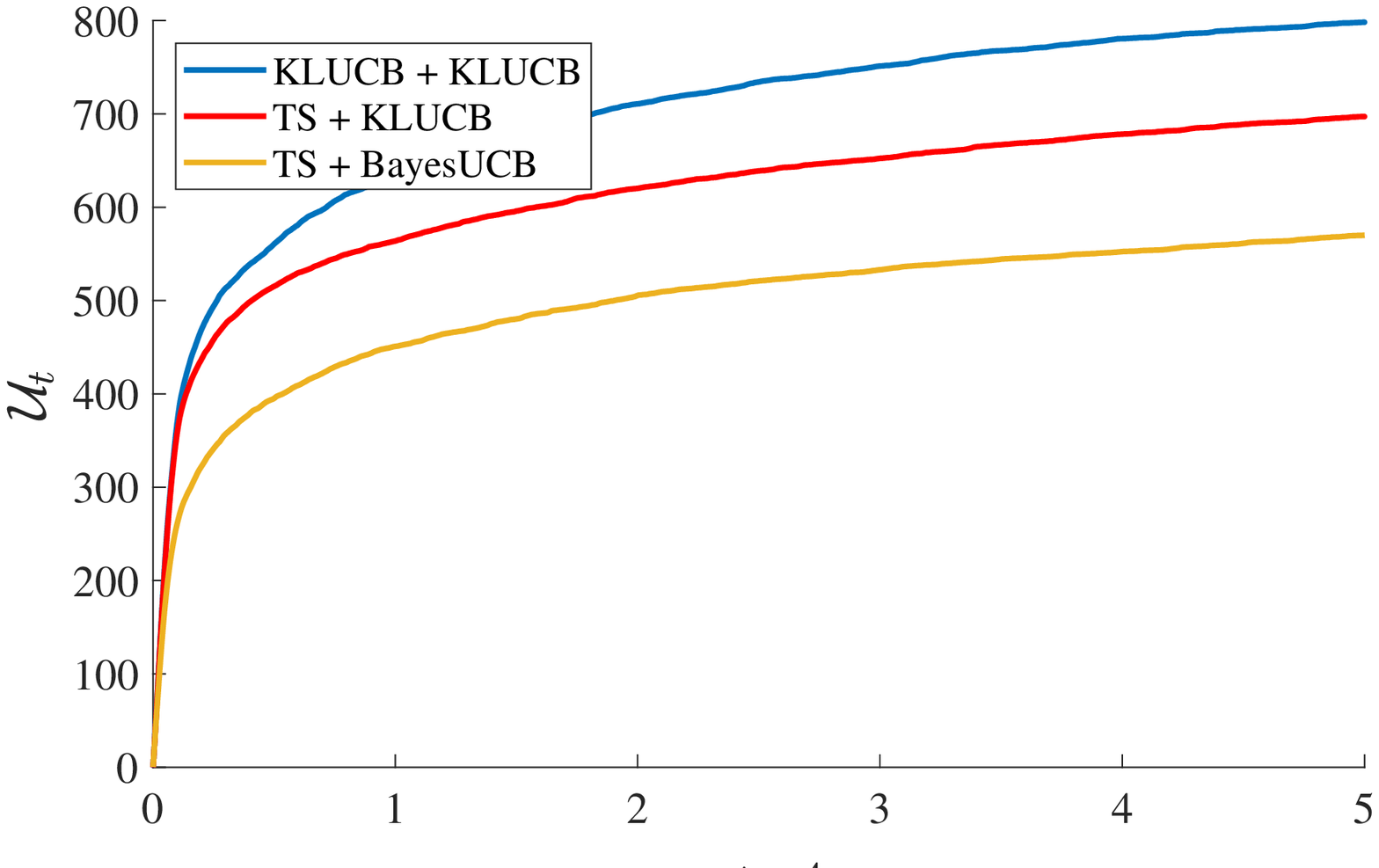}
    \caption{Empirical means over 500 trials of $\mathcal{R}_t$ (left) and $\mathcal{U}_t$ (right) for the drug trial data with $\alpha = 0.21$.} \vspace{-\baselineskip}
    \label{fig:trial_data_0_21}
\end{figure}

\textbf{Observations of Performance} From Fig.\ref{fig:trial_data_0_21}, we first note that both $\mathcal{R}_t$ and $\mathcal{U}_t$ are well controlled and well within the theoretical bounds for the methods we have analysed.\footnote{The main term of the regret bound is $137 \log t,$ and the unsafe-arm bound is $81\log t,$ both $>750$ for $t > 10^4$.} The general trend observed is that Alg.~\ref{alg:STOP} based methods that use a \ThompNS-based index outperform confidence bound indices of Alg.~\ref{alg:DOCB}, which is consistent with \citet{chapelle2011empirical}. Finally, we observe that Alg.~\ref{alg:TBU}, as represented by \ThompNS+\BayesUCB outperforms all other methods. These observations held regardless of the means we have run the methods on. One caveat, however, is that the underlying Bernoulli laws used are well aligned to the priors for Bayesian methods, which may improve their performance. 

\subsubsection{Dependence on Gaps} We investigate the dependence of the regret on the gaps $\Delta^k, \Gamma^k$, in particular illustrating that, as predicted by the theorems, this decays inversely with the larger of the two, and is insensitive to the smaller of the two.

\textbf{Inverse Dependence on Gaps} First, we will demonstrate that the regret varies with $(\Delta^k \wedge \Gamma^k)$ inversely. To this end, we study the the cases \begin{align*}
    \mu_i &= (0.5, 0.5 - i/25, 0.5 + i/25),\\
    \nu_i &= (0.5, 0.5 - i/25, 0.5 + i/25),
\end{align*}
for $\alpha = 0.5$ over $i$ in $[1:10]$ over $100$ trials across a horizon of $T = 2\times 10^4.$ Fig.~\ref{fig:gap_explore} reports the regret $\mathcal{R}_{T}$ versus $i/25$ over this data, and exhibits a clear inverse dependence on $i$.

\textbf{Lack of Dependence on Smaller Gaps} Secondly, we will illustrate that the dependence on the gaps is driven by the \emph{larger} of $\Delta^k$ and $\Gamma^k$, but not on $(\Delta^k \wedge \Gamma^k)$. For this we study the data \begin{align*}
    \mu_i = (0.5, 0.5 - i/25, 0.5 + i/250),\\
    \nu_i = (0.5, 0.5 + i/250, 0.5 + i/25),
\end{align*}
again with $\alpha = 0.5$ for $100$ trials over a horizon of $T = 2\times 10^4$. Observe that $\Delta^k \vee \Gamma^k$ is the same as the previous case, but $\Delta^k \wedge \Gamma^k$ is reduced by a factor of $10$ for each suboptimal arm. The principal observation from the second part of Fig.~\ref{fig:gap_explore} is that the plot remains similar to the previous case of `large' minimum gaps, bearing out this independence from the smaller of the two gaps. 

\begin{figure}[H]
    \centering
    \includegraphics[width = 0.49\linewidth]{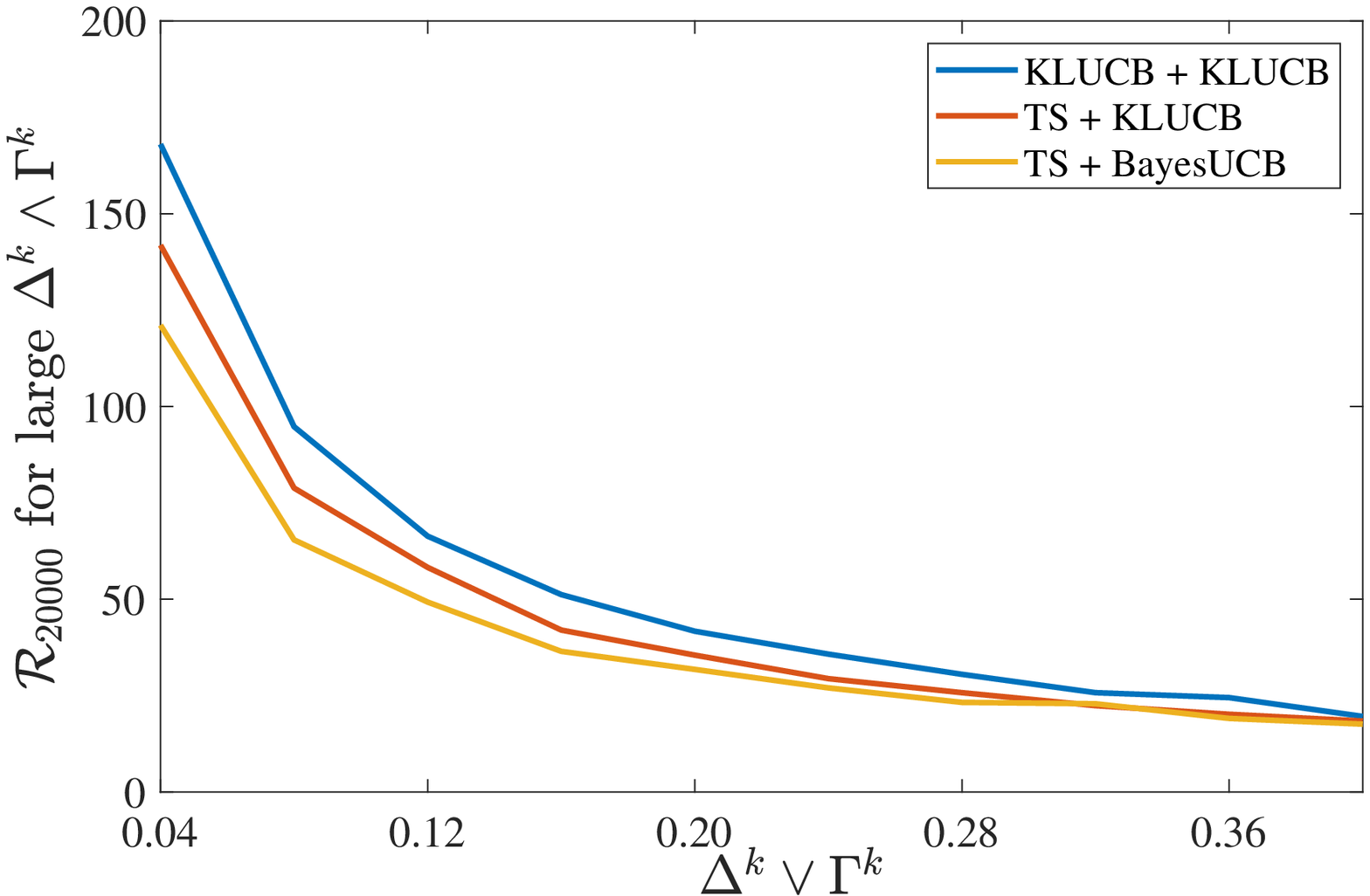}~
    \includegraphics[width = 0.49\linewidth]{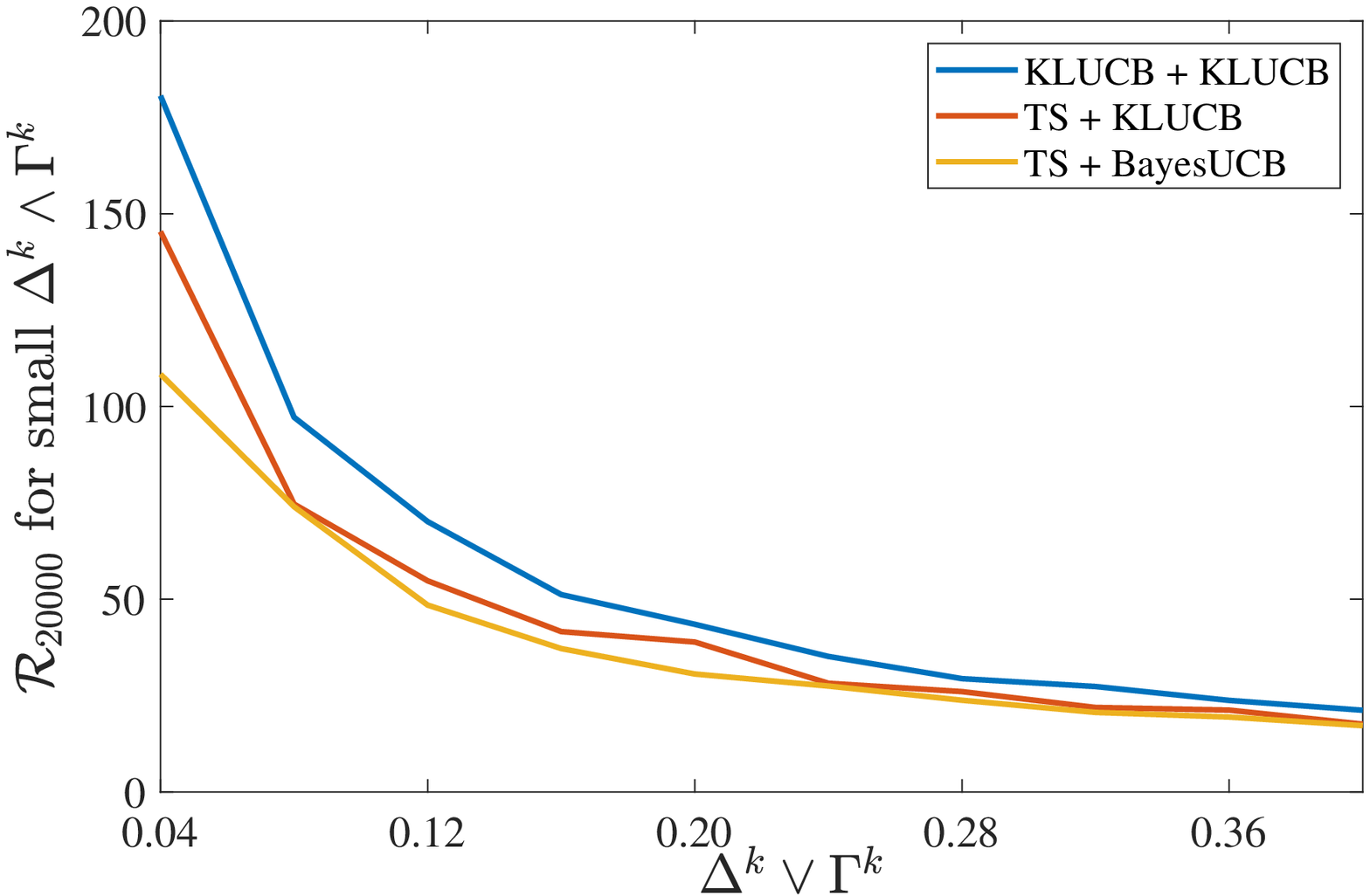}
    \caption{Behaviour of Regret at $T = 20000$ with respect to the maximum gap. Medians over 100 runs are reported. Left is the case of large minimum gaps, while right is the case of small minimum gaps.}
    \label{fig:gap_explore} \vspace{-\baselineskip}
\end{figure}

\bibliography{safe_bandit}

\begin{thebibliography}{23}
\providecommand{\natexlab}[1]{#1}
\providecommand{\url}[1]{\texttt{#1}}
\expandafter\ifx\csname urlstyle\endcsname\relax
  \providecommand{\doi}[1]{doi: #1}\else
  \providecommand{\doi}{doi: \begingroup \urlstyle{rm}\Url}\fi

\bibitem[Agrawal \& Devanur(2014)Agrawal and Devanur]{agrawal2014_BwCR}
Agrawal, S. and Devanur, N.~R.
\newblock Bandits with concave rewards and convex knapsacks.
\newblock In \emph{Proceedings of the fifteenth ACM conference on Economics and
  computation}, pp.\  989--1006, 2014.

\bibitem[Agrawal \& Goyal(2012)Agrawal and
  Goyal]{agrawal2012analysis_TS_Bern_equals_general}
Agrawal, S. and Goyal, N.
\newblock Analysis of {T}hompson sampling for the multi-armed bandit problem.
\newblock In \emph{Conference on learning theory}, pp.\  39--1. JMLR Workshop
  and Conference Proceedings, 2012.

\bibitem[Agrawal \& Goyal(2013)Agrawal and Goyal]{agrawal2013further}
Agrawal, S. and Goyal, N.
\newblock Further optimal regret bounds for {T}hompson sampling.
\newblock In \emph{Artificial intelligence and statistics}, pp.\  99--107.
  PMLR, 2013.

\bibitem[Amani et~al.(2019)Amani, Alizadeh, and
  Thrampoulidis]{amani_safe_linear}
Amani, S., Alizadeh, M., and Thrampoulidis, C.
\newblock Linear stochastic bandits under safety constraints.
\newblock In \emph{Advances in Neural Information Processing Systems},
  volume~32, 2019.

\bibitem[Audibert et~al.(2009)Audibert, Munos, and
  Szepesv{\'a}ri]{audibert2009exploration}
Audibert, J.-Y., Munos, R., and Szepesv{\'a}ri, C.
\newblock Exploration--exploitation tradeoff using variance estimates in
  multi-armed bandits.
\newblock \emph{Theoretical Computer Science}, 410\penalty0 (19):\penalty0
  1876--1902, 2009.

\bibitem[Aziz et~al.(2021)Aziz, Kaufmann, and Riviere]{aziz2021multi}
Aziz, M., Kaufmann, E., and Riviere, M.-K.
\newblock On multi-armed bandit designs for dose-finding clinical trials.
\newblock \emph{Journal of Machine Learning Research}, 22:\penalty0 1--38,
  2021.

\bibitem[Badanidiyuru et~al.(2013)Badanidiyuru, Kleinberg, and
  Slivkins]{bandits_with_knapsacks_2013}
Badanidiyuru, A., Kleinberg, R., and Slivkins, A.
\newblock Bandits with knapsacks.
\newblock In \emph{2013 IEEE 54th Annual Symposium on Foundations of Computer
  Science}, pp.\  207--216. IEEE, 2013.

\bibitem[Capp{\'e} et~al.(2013)Capp{\'e}, Garivier, Maillard, Munos, and
  Stoltz]{cappe2013kullback}
Capp{\'e}, O., Garivier, A., Maillard, O.-A., Munos, R., and Stoltz, G.
\newblock Kullback-leibler upper confidence bounds for optimal sequential
  allocation.
\newblock \emph{The Annals of Statistics}, pp.\  1516--1541, 2013.

\bibitem[Chapelle \& Li(2011)Chapelle and Li]{chapelle2011empirical}
Chapelle, O. and Li, L.
\newblock An empirical evaluation of {T}hompson sampling.
\newblock \emph{Advances in neural information processing systems},
  24:\penalty0 2249--2257, 2011.

\bibitem[Garivier \& Capp{\'e}(2011)Garivier and Capp{\'e}]{garivier2011kl}
Garivier, A. and Capp{\'e}, O.
\newblock The {KL-UCB} algorithm for bounded stochastic bandits and beyond.
\newblock In \emph{Proceedings of the 24th annual conference on learning
  theory}, pp.\  359--376. JMLR Workshop and Conference Proceedings, 2011.

\bibitem[Garivier et~al.(2019)Garivier, M{\'e}nard, and
  Stoltz]{garivier2019lower_bound}
Garivier, A., M{\'e}nard, P., and Stoltz, G.
\newblock Explore first, exploit next: The true shape of regret in bandit
  problems.
\newblock \emph{Mathematics of Operations Research}, 44\penalty0 (2):\penalty0
  377--399, 2019.

\bibitem[Genovese et~al.(2013)Genovese, Durez, Richards, Supronik, Dokoupilova,
  Mazurov, Aelion, Lee, Codding, Kellner, et~al.]{genovese2013efficacy}
Genovese, M.~C., Durez, P., Richards, H.~B., Supronik, J., Dokoupilova, E.,
  Mazurov, V., Aelion, J.~A., Lee, S.-H., Codding, C.~E., Kellner, H., et~al.
\newblock Efficacy and safety of secukinumab in patients with rheumatoid
  arthritis: a phase ii, dose-finding, double-blind, randomised, placebo
  controlled study.
\newblock \emph{Annals of the rheumatic diseases}, 72\penalty0 (6):\penalty0
  863--869, 2013.

\bibitem[Je{\v{r}}{\'a}bek(2004)]{jevrabek2004dual}
Je{\v{r}}{\'a}bek, E.
\newblock Dual weak pigeonhole principle, boolean complexity, and
  derandomization.
\newblock \emph{Annals of Pure and Applied Logic}, 129\penalty0 (1-3):\penalty0
  1--37, 2004.

\bibitem[Katz-Samuels \& Scott(2018)Katz-Samuels and Scott]{katz2018feasible}
Katz-Samuels, J. and Scott, C.
\newblock Feasible arm identification.
\newblock In \emph{International Conference on Machine Learning}, pp.\
  2535--2543. PMLR, 2018.

\bibitem[Katz-Samuels \& Scott(2019)Katz-Samuels and Scott]{katz2019top}
Katz-Samuels, J. and Scott, C.
\newblock Top feasible arm identification.
\newblock In \emph{The 22nd International Conference on Artificial Intelligence
  and Statistics}, pp.\  1593--1601. PMLR, 2019.

\bibitem[Kaufmann et~al.(2012{\natexlab{a}})Kaufmann, Capp{\'e}, and
  Garivier]{kaufmann2012bayesUCB}
Kaufmann, E., Capp{\'e}, O., and Garivier, A.
\newblock On bayesian upper confidence bounds for bandit problems.
\newblock In \emph{Artificial intelligence and statistics}, pp.\  592--600.
  PMLR, 2012{\natexlab{a}}.

\bibitem[Kaufmann et~al.(2012{\natexlab{b}})Kaufmann, Korda, and
  Munos]{kaufmann2012thompson}
Kaufmann, E., Korda, N., and Munos, R.
\newblock {T}hompson sampling: An asymptotically optimal finite-time analysis.
\newblock In \emph{International conference on algorithmic learning theory},
  pp.\  199--213. Springer, 2012{\natexlab{b}}.

\bibitem[Lattimore \& Szepesv{\'a}ri(2020)Lattimore and
  Szepesv{\'a}ri]{lattimore_szepesvari_book}
Lattimore, T. and Szepesv{\'a}ri, C.
\newblock \emph{Bandit algorithms}.
\newblock Cambridge University Press, 2020.

\bibitem[Moradipari et~al.(2021)Moradipari, Amani, Alizadeh, and
  Thrampoulidis]{moradipari2021safe_thompson}
Moradipari, A., Amani, S., Alizadeh, M., and Thrampoulidis, C.
\newblock Safe linear {T}hompson sampling with side information.
\newblock \emph{IEEE Transactions on Signal Processing}, 2021.

\bibitem[Pacchiano et~al.(2021)Pacchiano, Ghavamzadeh, Bartlett, and
  Jiang]{pacchiano2021stochastic}
Pacchiano, A., Ghavamzadeh, M., Bartlett, P., and Jiang, H.
\newblock Stochastic bandits with linear constraints.
\newblock In \emph{International Conference on Artificial Intelligence and
  Statistics}, pp.\  2827--2835. PMLR, 2021.

\bibitem[Thompson(1933)]{thompson1933likelihood}
Thompson, W.~R.
\newblock On the likelihood that one unknown probability exceeds another in
  view of the evidence of two samples.
\newblock \emph{Biometrika}, 25\penalty0 (3/4):\penalty0 285--294, 1933.

\bibitem[Wang et~al.(2021)Wang, Wagenmaker, and Jamieson]{wang2021best}
Wang, Z., Wagenmaker, A., and Jamieson, K.
\newblock Best arm identification with safety constraints.
\newblock \emph{arXiv preprint arXiv:2111.12151}, 2021.

\bibitem[Wu et~al.(2016)Wu, Shariff, Lattimore, and
  Szepesv{\'a}ri]{wu2016conservative}
Wu, Y., Shariff, R., Lattimore, T., and Szepesv{\'a}ri, C.
\newblock Conservative bandits.
\newblock In \emph{International Conference on Machine Learning}, pp.\
  1254--1262. PMLR, 2016.

\end{thebibliography}
\bibliographystyle{icml2022}

\newpage
\appendix
\onecolumn

\section{Notation and Broad Proof Strategy}\label{appx:strat}

We begin with some notation, and then describe the general proof strategy. 

We will abuse notation and let $\hist_{t-1}$ also stand for the sigma algebra induced by the history, with $\hist_0$ denoting the trivial sigma algebra. Naturally, $\{\hist_{t}\}$ forms a filtration - observe that in the \Thomp cases, the laws of $\rho_t^k$ are measurable with respect to $\hist_{t-1}$. The Bayesian methods also utilise extraneous randomness, as represented by the various $\rho_t^k$s. 
An important observation regarding all the methods we design is that the permissible set $\Perm_t$ is a \emph{predictable} process, i.e., is determined given $\hist_{t-1}$. Indeed, all the methods use an index based on the history to decide $\Perm_t,$ and so it is a deterministic function of the variables $\{(A_s, R_s, S_s) : s< t\}.$ Of course, this is not as such required, but it represents a convenience that we will employ in our proofs. For the sake of brevity, we will denote the conditional laws $\mathbb{P}(\cdot | \hist_{t-1})$ as $\bbP_{t-1}$.

\textbf{Proof strategy} The basic decomposition of regret is in terms of $N_T^k$ - indeed, due to the additive definition, \[ \mathbb{E}[\mathcal{R}_T] = \sum_{k \neq k^*} \mathbb{E}[N_T^k] (\Delta^k \vee \Gamma^k).\] 
Therefore, the main arguments all control $\mathbb{E}[N_T^k]$ for all suboptimal arms $k$. Of course, subsidiary claims about $\mathbb{E}[\mathcal{U}_T]$ also follow from these. 

The arguments separately control $\mathbb{E}[N_T^k]$ for infeasible and inefficient arms. For arms which are both inefficient and infeasible, the tighter of the control offered by these two arguments can be taken, and this yields the form of the expressions in the main text. 

\emph{Infeasible arms} All of our schemes use a safety index $L_t^k$ to populate the permissible set $\Perm_t$. We exploit the properties of this index to control the play of infeasible arms. Indeed, we can decompose \begin{align*} \mathbb{E}[N_T^k] &= \sum \mathbb{P}(A_t = k) \le \sum_{t = 1}^T \mathbb{P}(L_t^k \le \alpha).
\end{align*}

The design of the two indices - that via \KLUCB and via \BayesUCB both ensure that the chance of playing an infeasible arm more than $O(\log(T)/d(\nu^k\|\alpha))$ times is exponentially small. For \KLUCBNS, this is a simple consequence of Chernoff's bound. For \BayesUCBNS, the argument reduces to that for \KLUCB using a connection between the tails of Beta distributions and Binomials.

\emph{Inefficient arms} Following the standard method for confidence bound based index policies, controlling the play of inefficient arms requires some known good index to compare the reward indices to. Naturally, we want to use the index of $k^*,$ but doing so requires that $k^*$ itself is permitted, since otherwise the algorithm never takes its reward index into consideration when choosing an arm. This represents the main deviation from standard proofs.

Let us take the case of \KLUCBNS. The idea is to decompose \begin{align*}
    \mathbb{E}[N_t^K] &= \sum \mathbb{P}(A_t = k) \\
                      &= \sum \mathbb{P}(A_t = k, k^* \not\in \Perm_t) + \mathbb{P}(A_t = k, k^* \in \Perm_t)\\
                      &\le \sum_t \mathbb{P}(k^* \not\in \Perm_t) + \sum_t \mathbb{P}(A_t = k, k^* \in \Perm_t)
\end{align*}  
The first course of action then is to ensure that the first term is small, which exploits the consistency of $L_t^*$. 

This enables us to proceed pretty much as usual. For \KLUCBNS, we decompose the second term as \begin{align*}
    \sum_t \mathbb{P}(A_t = k, k^* \in \Perm_t) &\le \sum_t \mathbb{P}( U_t^* \le \mu^*, k^* \in \Perm_t) + \mathbb{P}(U_t^k < \mu^*, U_t^* \ge \mu^*, k^* \in \Perm_t, A_t = k) \\ 
    &\le \sum_t \mathbb{P}(U_t^* < \mu^*) + \mathbb{P}(U_t^k \ge \mu^*, A_t = k).
\end{align*}
Of course, the final expression is the usual quantity controlled in regret proofs, and this argument can be repeated without change. For the sake of being self-contained, we will sketch these proofs in the subsequent as well. For \KLUCBNS, this is essentially the argument of \citet{garivier2011kl}, while for the \BayesUCB bound, this is the argument of \citet{kaufmann2012bayesUCB} (which itself is very similar to \citet{garivier2011kl}). For the efficiency of \ThompNS, we will use the argument of \cite{agrawal2013further}.

\emph{Remark on showing consistency of $L_t^*$} We observe that, by design, our choices of $L_t^k$ are such that consistency proofs for $U_t^*$ translate directly into those for $L_t^*$ - this is due to the symmetry of the relevant functionals under the maps $(S, \nu^k, \alpha) \mapsto (1-S, 1-\nu^k, 1-\alpha),$ upon doing which $1- L_t^k$ is a $U_t^k$-type upper bound for $1 - \nu^k$. Similarly, the argument for controlling $\sum \mathbb{P}(L_t^k \le \alpha)$ for infeasible arms is basically the same as that for controlling $\sum \mathbb{P}(U_t^k \ge \mu^*)$ for the standard bandit version of the appropriate method. 

That said, we note a deviation from the proof of this consistency for the case of \BayesUCBNS. Since controlling standard regret in a Bayesian setting requires one to compare two \emph{random} indices, \citet{kaufmann2012bayesUCB} use direct comparison of their index $U_{t, \BayesUCBNS}^*$ to $\mu^*$ only enough to argue that $N_t^*$ is at least logarithmically large. With this in hand, they can argue that $U_{t, \BayesUCBNS}^*$ is at least $\mu^* - O(\sqrt{1/\log(T)})$ with high probability and argue that this is unlikely to be exceeded by suboptimal arms. However, to ensure that our (random) safety index $L_t^*$ is consistent, we must compare it to a fixed value $\alpha$, and so this second argument utilising a weakened consistency does not carry over. We handle this by loosening the quantiles $\delta_t^k$ enough so that the first argument itself is sufficient to provide consistency. This represents a gap, which may be possibly be resolved by a stronger analysis.

\emph{Remark on dependence} Note that the sketch above above does not use the potential dependence between the signals $(R,S).$ It is possible that this can be exploited, and this exploitation may gain in importance as we increase the number of safety constraints. We leave this as a direction for further work.

\section{Proof for Doubly Optimistic Confidence Bounds}\label{appx:kl-ucb-proof}

The following lemma essentially follows from the main result of the \KLUCB analysis due to \citet{garivier2011kl}, and forms the key statement to demonstrate our results. We note that this is stated slightly more generically than in their paper, essentially to let us use the same result to show both gap dependent and gap independent bounds. We came across this trick in the work of \citet{agrawal2013further}.

\begin{mylem}[Adaptation of \citet{garivier2011kl}]\label{lem:klucb_key_bound}
    Let $k$ be a suboptimal arm. Then Algorithm \ref{alg:DOCB}, instantiated with the \KLUCB type confidence bounds attains the following guarantees for all $k$.
    \begin{itemize}
        \item If $\Delta^k > 0,$ then for any $x \in (\mu^k, \mu^*),$ \begin{equation} \label{eq:klucb_reward_bound} \mathbb{E}[N_T^k] \le \frac{ \log T + 3 \log\log T}{d(x\|\mu^*)} + 6 \log\log T + \frac{2}{ 1 \wedge  d(x\|\mu^k)} + 24. \end{equation} 
        \item If $\Gamma^k > 0,$ then for any $y \in (\alpha, \nu^k),$ \begin{equation} \label{eq:klucb_safety_bound} \mathbb{E}[N_T^k] \le \frac{ \log T + 3 \log\log T}{d(y\|\alpha)} + \frac{2}{ 1 \wedge  d(y\|\nu^k)}. \end{equation}
    \end{itemize}
\end{mylem}

We will first show the proofs of the two results using the above lemma, and leave proving it until the end.

\subsection{Proof of Theorem \ref{thm:pure_klucb}}\label{appx:gap_dep_klucb_pf}

\begin{proof}

Fix an arm $k$. If $\Delta^k > 0,$ then choose $x \in (\mu^k, \mu^*)$ such that $d(x\|\mu^*) = \frac{d(\mu^k\|\mu^*) }{1 + \varepsilon}$ - this exists since $d(x\|\mu^*)$ is continuous and monotonically decreases from $d(\mu^k\|\mu^*)$ to $0$ as $x$ varies in $(\mu^k, \mu).$ We need to argue that the third term in the bound of (\ref{eq:klucb_reward_bound}) is bounded as $O(\varepsilon^{-2})$. This follows since for small $\varepsilon,$ $x = \mu^k + O(\varepsilon)$. 

Indeed, let us abbreviate $d = d(\mu^k\|\mu^*),$ and observe that the the derivative $d' := \left.\partial_z d(z\|\mu^*) \right|_{z = \mu^k}$ is non-zero, and so $x - \mu^k = \varepsilon \frac{d}{|d'|} + O(\varepsilon^2).$ But then notice that since $d(z \|\mu^k)$ is minimised at $z = \mu^k,$ $d(x\|\mu^k) = \nicefrac12(\widetilde{d}''\varepsilon^2 (d/d')^2) + O(\varepsilon^3),$ where $ \widetilde{d}'':= \left.\partial^2_{zz} d(z\|\mu^k)\right|_{z = \mu^k}$.   We conclude that \[ \frac{2}{d(x\|\mu^k) \wedge 1} = O\left( \frac{d'^2}{d'' d^2 \varepsilon^2}\right), \] which of course is a scaling of $\varepsilon^{-2}$ by a problem dependent constant. 

Next, if $\Gamma^k > 0,$ we proceed similarly to the above, and choose $y \in (\alpha, \nu^k)$ such that $d(y\|\alpha) = d(\nu^k\|\alpha)/(1+\varepsilon).$ By an entirely identical calculation as above, the final term of (\ref{eq:klucb_safety_bound}) is bounded as $O \left( \frac{f'^2}{f''} \frac{1}{d^2(\nu^k \|\alpha) \varepsilon^2} \right),$ where $f' =\left. \partial_z d(z\|\alpha)\right|_{z = \nu^k},$ and $\widetilde{f}'' = \left. \partial^2_{zz} d^2(z\|\nu^k) \right|_{z = \nu^k}.$

Using both of these bounds, we conclude that \begin{align*} \mathbb{E}[N_T^k] &\le \frac{1}{\indi\{\mu^k < \mu^*\}} \left\{ \frac{(1+\varepsilon) \log T}{d(\mu^k\|\mu^*)} + \frac{(1+\varepsilon) 3 \log\log T}{d(\mu^k\|\mu^*)} + 6 \log\log T+ 24+  O\left( \frac{(d'^2/\widetilde{d}'')}{d^2(\mu^k\|\mu^*) \varepsilon^2}\right)\right\}, \\ 
\mathbb{E}[N_T^k] &\le \frac{1}{\indi\{\nu^k > \alpha\}} \left\{\frac{(1+\varepsilon) \log T}{d(\nu^k\|\alpha)} + \frac{(1+\varepsilon) 3 \log\log T}{d(\nu^k\|\alpha)} + O\left( \frac{(f'^2/\widetilde{f}'')}{d^2(\nu^k\|\alpha) \varepsilon^2}\right)\right\}, \end{align*}
where we set $1/\indi\{ p \} = \infty$ when the proposition $p$ is untrue. Of course, recalling that $\indi\{\mu^k < \mu^*\}d(\mu^k\|\mu^*) = d_<(\mu^k\|\mu^*)$ and similarly $d_>(\nu^k\|\nu^*),$ we may choose the tighter of the above bounds to get the result \[\mathbb{E}[N_t^k] \le \frac{(1+\varepsilon) \log T}{d_<(\mu^k\|\mu^*) \vee d_>(\nu^k\|\nu^*)} + O\left( \frac{\log\log T}{d_<(\mu^k\|\mu^*) \vee d_>(\nu^k\|\nu^*)} + \frac{1}{(d_<(\mu^k\|\mu^*) \vee d_>(\nu^k\|\nu^*))^2 \varepsilon^2}\right).  \]

The claimed bounds now follow trivially - to control $\mathbb{E}[\mathcal{R}_T],$ simply multiply by the per-round regret of playing arm $k$, $\Delta^k \vee \Gamma^k,$ and sum. To control $\mathbb{E}[\mathcal{U}_T],$ simply add up the above over the unsafe arms.

\end{proof} 

Note that as the gaps $\Delta^k$ and $\Gamma^k$ decay, the last term scales as $1/(\Delta^k \wedge \Gamma^k)^4$, which only yields a $T^{3/4}$ gap-independent bound.

\subsection{Proof of Theorem \ref{thm:pure_klucb_gap_indep}}\label{appx:klucb_gap_indep_pf}

As is standard, the gap-independent regret bounds follow on observing that arms for which the gap is too small cannot actually incur large regret over $T$ rounds. To this end, let $\mathbf{M} > 0$ be a parameter to be chosen, and express regret as \begin{equation} \label{eqn:regret_in_terms_of_gaps} \mathbb{E}[\mathcal{R}_T] \le  \sum_{k : \Delta^k > \Gamma^k \vee \mathbf{M} } \mathbb{E}[N_T^k] \Delta^k + \sum_{k : \Gamma^k > \Delta^k \vee \mathbf{M} } \mathbb{E}[N_T^k] \Gamma^k + \mathbf{M} \sum_{k : (\Delta^k \vee \Gamma^k) \le \mathbf{M}} \mathbb{E}[N_T^k].\end{equation} The last term is of course bounded by $\mathbf{M}T,$ and so we will end up taking $\mathbf{M}$ of order $\sqrt{K \log T/T}$ to control regret. It remains to show that $\mathbb{E}[N_T^k]$ is not too large for arms with large gaps. To this end, we first develop bounds dependent explicitly on the gaps using (\ref{eq:klucb_reward_bound}) and (\ref{eq:klucb_safety_bound}).

\begin{mylem}
    For any arm $k$ with $\Delta^k > 0,$ \[ \mathbb{E}[N_T^k] \le \frac{2\log T + 6\log\log T + 4}{(\Delta^k)^2} + 6 \log\log T + 24 . \]
    Similarly, for any arm $k$ with $\Gamma^k > 0,$ \[ \mathbb{E}[N_T^k] \le \frac{2\log T + 6\log\log T + 4}{(\Gamma^k)^2}.\]
\end{mylem}
\begin{proof}
    First, take a $k$ with $\Delta^k > 0,$ and in the bound (\ref{eq:klucb_reward_bound}), set $x = (\mu^k + \mu^*)/2 =: \bar\mu^k.$ By Pinsker's inequality, $d(\bar\mu^k\|\mu^*) \ge 2(\mu^* - \bar\mu^k)^2 = (\Delta^k)^2/2,$ and $d(\bar\mu^k\|\mu^k) \ge 2(\bar\mu^k - \mu^k)^2 = (\Delta^k)^2/2.$ Plugging these into the bound yields the claim upon observing that $(\Delta^k)^2/2 \le 1.$
    
    For arms with $\Gamma^k > 0,$ we can develop a similar control resulting from $(\ref{eq:klucb_safety_bound})$ by setting $y = (\alpha + \nu^k)/2$. 
\end{proof}

We are now in a position to show the claim. 

\begin{proof}[Proof of Theorem \ref{thm:pure_klucb_gap_indep}] 

The first term in (\ref{eqn:regret_in_terms_of_gaps}) can be bounded as \[ \sum_{\Delta^k > \Gamma^k \vee \mathbf{M}} \frac{2\log T + 6\log\log T + 2}{\Delta^k} + (6 \log\log T + 24) \Delta^k \le K_\Delta \left( \frac{2\log T + 6\log\log T  + 4}{\mathbf{M}} + 6 \log\log T + 24\right),  \] where $K_\Delta = |\{ k : \Delta^k > \Gamma^k\}|$.

Similarly, the second term in (\ref{eqn:regret_in_terms_of_gaps}) can be bounded as \[ \sum_{\Gamma^k > \Delta^k \vee \mathbf{M}} \frac{2\log T + 6\log\log T + 4}{\Gamma^k} \le K_\Gamma \frac{2\log T + 6\log\log T + 4}{\mathbf{M}}, \] where $K_\Gamma = |\{k : \Gamma^k > \Delta^k\}|.$

Finally, observing that $K_\Gamma + K_\Delta \le K,$ we conclude that \[ \mathbb{E}[\mathcal{R}_T] \le \frac{K}{\mathbf{M}} (2\log T + 6\log\log T + 4) +  (6 \log\log T +24)\sum (\Delta^k \vee \Gamma^k) +  T \mathbf{M} . \]

The claim follows on choosing $\mathbf{M} = \sqrt{K (2\log T + 6\log\log T + 4)/T},$ and observing that $ 2 \log T \ge 4$ for $T \ge 8,$ and $2\log\log T \le \log T$ for all $T$. 
\end{proof}

\subsection{Proof of Lemma \ref{lem:klucb_key_bound}  }

\begin{proof} We make the argument separately for infeasible and inefficient arms. The former is easier, so let us begin with it.

\textit{Infeasible arms} 

We follow the decomposition from \S\ref{appx:strat}. Recall that $L_t^k = \min\{q \le \nuhat : d(\nuhat \|q) \le \gamma_t/N_t^k\}$. Since $d(\nuhat\|x)$ is a continuous decreasing function on $[0, \nuhat],$ if $L_t^k \le \alpha$ then it must either hold that $\nuhat \le \alpha,$ or $d(\nuhat\|\alpha) \le d(\nuhat\|L_t^k) = \gamma_t/N_t^k.$ Either way, we have that $d_>(\nuhat \|\alpha) \le \gamma_t/N_t^k$.

Now, let $\hat\nu^k(s)$ denote the value of $\nuhat$ after the $s$th time we play the arm $k$. We observe that \begin{align*}
    \sum_t \indi\{A_t = k\} &\le \sum_{t= 1}^T \indi\{A_t = k, d_>(\nuhat\|\alpha) \le \gamma_t/N_t^k\}\\
                            &= \sum_{t= 1}^T \sum_{s = 1}^t \indi\{A_t = k, s d_>(\nuhat \|\alpha) \le \gamma_t, N_t^k = s\}\\
                            &\le \sum_{t= 1}^T \sum_{s = 1}^t \indi\{A_t = k, N_t^k = s\} \cdot \indi\{ s d_>(\hat\nu^k(s)\|\alpha) \le \gamma_T\} \\
                            &= \sum_{s = 1}^T \indi\{s d_>(\hat\nu^k(s)\|\alpha) \le \gamma_T\} \cdot  \sum_{t = s}^T \indi\{A_t = k, N_t^k = s\} \\
                            &= \sum_{s = 1}^T \indi\{s d_>(\hat\nu^k(s)\|\alpha) \le \gamma_T\},
\end{align*}
where we have used that $\gamma_t$ increases with $T$, and for any value $s$, there is at most one time step on which $N_t^k$ is exactly $s$ and we play the action $k$.

Now, we observe that for any $y \in (\alpha, \nu^k),$ the event $\{ d_>(\hat\nu^k(s)\|\alpha) \le d(y\|\alpha)\} = \{ \hat\nu^k(s) \le y\}.$ Indeed, $d_>(u\|\alpha)$ is exactly equal to $0$ for $u \le \alpha,$ and monotonically increasing for $u > \alpha$. But, recalling Chernoff's bound (which applies since the random variables are bounded in $[0,1]$), $P( \hat\nu^k(s) \le y ) \le \exp(-s d(y\|\nu^k)).$ This sets up the following calculation.

Let $y \in (\alpha, \nu^k),$ and define $S(y) := \lfloor \gamma_T / d(y\|\alpha) \rfloor,$ so that for all $s > S(y),$ $\gamma_T/s < d(y\|\alpha)$. Then \begin{align}\label{eq:klucb_safety_guts}
    \mathbb{E}[N_T^k] &= \sum_{t = 1}^T \mathbb{P}(A_t = k) \notag \\
                      &\le \sum_{s = 1}^T \mathbb{P}(s d_>(\hat\nu^k(s)\|\alpha) \le \gamma_t) \notag\\
                      &\le S(y) + \sum_{s = S(y) + 1}^T \mathbb{P}( d_>(\hat\nu^k(s) \|\alpha) \le d(y\|\alpha)) \notag\\
                      &\le S(y) + \sum_{s = S(y) + 1}^T e^{-s d(y\|\nu^k)}\notag \\
                      &\le S(y) + \frac{e^{-(S(y) + 1) d(y\|\nu^k)}}{1- e^{-d(y\|\nu^k)}}\notag\\
                      &\le S(y) + \frac{2}{1 \wedge d(y\|\nu^k)},
\end{align}    
where the last term uses that $(S(y) + 1)d(y\|\nu^k) \ge 0$, and $\frac{1}{1-e^{-u}} \le 2/(1\wedge u).$ But $S(y) \le \frac{\gamma_T}{d(y\|\alpha)} = \frac{\log T + 3 \log\log T}{d(y\|\alpha)}.$

\emph{Inefficient arms} Again, we follow the decomposition from \S\ref{appx:strat}, namely \[\mathbb{E}[N_T^k] \le \sum_t \mathbb{P}(k^* \not\in \Perm_t) + \mathbb{P}(U_t^* < \mu^*) + \mathbb{P}(U_t^k \ge \mu^*).  \] Observe that $\mathbb{P}(k^* \not\in \Perm_t) = \mathbb{P}( L_t^* > \alpha).$

As noted in \S\ref{appx:strat}, the final term is controlled in exactly the same way as the inefficiency control. Indeed, $U_t^k = \max\{q \ge \muhat: d(\muhat\|q) \le \gamma_t/N_t^k\}$. Since $d(\muhat\|x)$ increases in the range $[\muhat, 1],$ if $U_t^k \ge \mu^*,$ then either $\muhat > \mu^*,$ or $d(\muhat\|\mu^*) \le \gamma_t/N_t^k$.  Developing the subsequent bound in exactly the same way, we find that \[ \sum_t \indi\{A_t = k, U_t^k \ge \mu^*\} \le \sum_{s = 1}^T \indi\{s d_<(\hat\mu^k(s)\|\mu^*) \le \gamma_T\},\] and again, for any $x \in (\mu^k, \mu^*),$ $P( d_<(\hat\mu^k(s)\|\mu^*) \le d(x\|\mu^*) ) = P(\hat\mu^k(s) \le x) \le \exp( - d(x\|\mu^k)).$ The resulting sum then gives the bound \[\sum \mathbb{P}(U_t^k \ge \mu^*, A_t = k) \le S(x) + \frac{2}{1 \wedge d(x\|\mu^k)}, \] where $S(x) \le \frac{\gamma_T}{d(x\|\mu^*)}.$

It remains to control $\sum \mathbb{P}(L_t^* > \alpha) + \mathbb{P}(U_t^* < \mu^*)$. To control the second term, we first exploit the monotonicity of $d(\hat\mu_t^*\|q)$ on $[\hat\mu_t^*,1]$ to note that \[\{U_t^* > \mu^*\} = \{\max\{ q > \hat\mu_t^* : d(\hat\mu_t^*\|q) \le \gamma_t/N_t^* \} < \mu^* \} = \{ \hat\mu_t^* < \mu^*, d(\hat\mu_t^*\|\mu^*) > \gamma_t/N_t^*\}.  \]

The final event is the subject of \citep[Theorem 10,][]{garivier2011kl}, who show that for any $z>0$, and any $k$ \begin{equation}\label{eq:klucb_consistency}
    \mathbb{P}(N_t^k d(\hat\mu_t^k\|\mu^k) > z) \le e (z \log(t) + 1) e^{-z} 
\end{equation}
The statement extends, of course, to the empirical mean of any subsampling of any i.i.d.~process in $[0,1]$. The gist of the argument is to partition the space according to the size of $N_t^k$. If $N_t^k$ is non-trivially large at some fixed time $t$, then it is exponentially unlikely for $Nd(\hat\mu^k_t\|\mu^k)$ to exceed $z$, essentially because the cumulant generating function is bounded by that of a Bernoulli, and $d$ is the Fenchel dual of this function for the Bernoulli. It is then just a question of stitching together these bounds over a well-chosen grid of values that $N_t^k$ may take (concretely, a geometrically increasing grid is used, and we end up with a $\log t$ due to this grid), and accounting for the poor behaviour for small $N_t^k$ (whence the premultiplying $e$). The argument presented in the supplement to the follow up work by \citet{cappe2013kullback} is somewhat cleaner than the original, and might be preferred.

Applying (\ref{eq:klucb_consistency}) to $\hat\mu_t^*$ and $z = \gamma_t$, we find that \[\mathbb{P}(U_t^* < \mu^*) \le e(\gamma_t \log (t) + 1)e^{-\gamma_t}, \] and so \begin{align}\label{eq:klucb_u_t_consistency}
    \sum_{t = 3}^T \mathbb{P}(U_t^* < \mu^*) &\le \sum_{t = 3}^T \frac{e( \log^2 t + 3\log t \cdot \log\log t + 1}{t \log ^3(t)} \notag \\
                                             &\le e ( \log\log T + 4).
\end{align}

Control on $\sum \mathbb{P}(L_t^* > \alpha)$ follows identically. Exploiting monotonicity twice, \[ \{L_t^* > \alpha\} = \{\hat\nu_t^* > \alpha, d(\hat\nu^*_t\|\alpha) > \gamma_t/N_t^* \} \subset \{\hat\nu_t^* > \nu^*, d(\hat\nu_t^*\|\nu^*) > \gamma_t/N_t^*\},\]
and thus, applying (\ref{eq:klucb_consistency}) to $\hat\nu_t^*$ with $z = \gamma_t,$ \begin{equation} \label{eq:klucb_l_t_consistency} \sum_{t = 3}^T \mathbb{P}(k^* \not\in \Perm_t) = \sum_{t = 3}^T \mathbb{P}(L_t^* > \alpha) \le e\log\log T + 4e .\end{equation}

Putting these together, we have \[\mathbb{E}[N_t^k] \le \frac{\log T + 3\log\log T}{d(x\|\mu^*)} + 6\log\log T + 24 + \frac{2}{1 \wedge d(x\|\mu^k)},  \] where we have used $2e < 6, 8e + 2 < 24$.
\end{proof}

\section{Proofs for Thompson Sampling with Optimistic Safety Indices}\label{appx:TS+UCB}

The first observation is that since the safety index $L_t^k$ remains unchanged, we may directly use the proofs of Lemma \ref{lem:klucb_key_bound} to observe that the bounds $(\ref{eq:klucb_safety_bound})$ and (\ref{eq:klucb_l_t_consistency}) continue to hold, that is, \begin{align*}
    \mathbb{E}[N_T^k] &\le \inf_y \frac{1}{\indi\{ \alpha < y < \nu^k\} } \left( \frac{\log T + 3 \log\log T}{d(y\|\alpha)} + \frac{2}{1 \wedge d(y\|\nu^k)} \right),\\
    \sum_{t = 3}^T &\mathbb{P}(k^* \not\in \Perm_t) \le e\log\log T + 4e.
\end{align*}

The focus of the study then is to ensure that the \Thomp analysis extends to control the play of inefficient arms. This pretty much exploits the analysis of \Thomp due to \citet{agrawal2013further}, although alternate analyses such as that of \citet{kaufmann2012thompson} can equivalently be used. 

The main bound is summarised in the following \begin{mylem}[Adaptation of \citet{agrawal2013further}]\label{lem:ts_efficiency_bound}
    There exists a universal constant $C$ such that if $\Delta^k > 0,$ then for any $u, v$ such that $\mu^k < u < v < \mu^*$, \begin{equation}\label{eq:TS_efficiency}\sum_{t = 1}^T \mathbb{P}(A_t = k, k^* \in \Perm_t) \le \frac{\log T}{d(u\|v)} + \frac{3}{1 \wedge d(u\|\mu^k)} +  \frac{C}{(\mu^* - v)^2} \left(1  + \log \frac{1}{\mu^* -v} + \log \left( \frac{1}{1 - e^{-d(v\|\mu^*)}} \wedge  T(\mu^* - v) \right) \right)  \end{equation}
\end{mylem}

Let us first demonstrate the result from the main text using the above Lemma. 

\begin{proof}[Proof of Theorem \ref{thm:stop_gap_dep}]
    We first argue the theorem.
    
    For infeasible arms, instantiate (\ref{eq:klucb_safety_bound}) with a $y$ such that $d(y\|\alpha) = d(\nu^k\|\alpha)/(1+\varepsilon)$. Since as previously argued, the resulting $d(y\|\nu^k)$ is $\Theta(\varepsilon^2).$
    
    For inefficient arms, consider the decomposition \[ \mathbb{E}[N_T^k] = \sum_{t = 1}^T \mathbb{P}(A_t = k) \le  \sum_{t = 1}^T \mathbb{P} (k^* \not\in \Perm_t) + \sum_{t = 1}^T \mathbb{P}(k^* \in \Perm_t, A_t = k). \]
    
    The first term is bounded as $3\log\log T.$ For the second term, we instantiate the bound (\ref{eq:TS_efficiency}) with a $u$ and a $v$ chosen so that \begin{enumerate}
        \item $d(u\|\mu^*) = d(\mu^k\|\mu^*)/\sqrt{1+\varepsilon}$
        \item $d(u\|v) = d(u\|\mu^*)/\sqrt{1+\varepsilon} = d(\mu^k\|\mu^*)/(1+\varepsilon),$
    \end{enumerate}
    both of which exist by continuity. 
    
    Showing the bound then requires control on $u - \mu^k$ and $\mu^* -v$ (using the upper bound $d(a\|b) \ge 2(a-b)^2$). To this end, as in the proof of Theorem \ref{thm:pure_klucb}, observe that $u = \mu^k + \Theta(\sqrt{1 + \varepsilon} -1) = \mu^k + \Theta(\varepsilon).$ Similarly, $v = \mu^* - \Theta(\varepsilon).$ Therefore,  $d(u\|\mu^k), d(v\|\mu^*) = \Theta(\varepsilon^{-2}).$ Finally, since this $\varepsilon^{-2}$ term does not grow with $T$, $(d(v\|\mu^*))^{-1} \wedge T = O(\varepsilon^{-2})$.  
    
    We may now conclude the argument exactly as in the proof of Theorem \ref{thm:pure_klucb}
\end{proof}

Similarly to the case for Algorithm \ref{alg:DOCB}, this scheme also admits a gap-independent bound. \begin{myprop}\label{prop:STOP_gap_indep}
    Algorithm \ref{alg:STOP}, instantiated with \KLUCB type lower confidence bounds, attains the gap independent regret bound \[ \mathbb{E}[\mathcal{R}_T] \le O(\sqrt{KT \log T} + K \log \log T).\]
    \begin{proof}
        For infeasible arms, instantiate (\ref{eq:klucb_safety_bound}) with $y = (\alpha + \nu^k)/2$ to conclude that \[ \mathbb{E}[N_T^k] \le O\left(\frac{\log T}{(\Gamma^k)^2}\right)\]
            
        For inefficient arms, instantiate (\ref{eq:TS_efficiency}) with $u = \mu^k + \Delta^k/3,$ and $v = \mu^k + 2\Delta^k/3.$ Then $\mu^* - v = v - u = u - \mu^k = \Delta^k/3,$ and by observing that $d(v\|\mu^*)^{-1} \wedge T\Delta^k/3 \le T\Delta^k/3,$ we have the upper bound \[ \mathbb{E}[N_T^k] \le O(\log\log T) +  O\left( \frac{\log T}{(\Delta^k)^2} + \frac{1 + \log(1/\Delta^k) + \log(T\Delta^k)}{(\Delta^k)^2} \right) = O\left(\log\log T + \frac{\log T}{(\Delta^k)^2} \right).\]
        
        Taking the tighter of these bounds, and partitioning according to the size of $\Delta^k \vee \Gamma^k$, we have the bound \[ \mathbb{E}[\mathcal{R}_T] \le \inf_{\mathbf{M} > 0} T\mathbf{M} + O \left( \frac{K \log T}{\mathbf{M}}\right)+ O(K \log\log T), \] giving the  claim upon optimisation. 
    \end{proof}
\end{myprop}

It remains to show the key Lemma. Again, we note that the key ideas are due to \citet{agrawal2013further}.

\begin{proof}[Proof of Lemma \ref{lem:ts_efficiency_bound}]
    Fix a $k$. The values $u$ and $v$ essentially represent indices that we can compare the random scores $\rho_t^*$ and $\rho_t^k$ to. To this end, we define the `good' events \begin{align*} \good{\mu,k}t := \{ \muhat \le u \}, \\ \good{\rho,k}t:=\{ \rho_t^k \le v\}. \end{align*}
    Notice that $\good{\mu,k}t$ lies in $\hist_{t-1}$.
    
    Now, we start with the decomposition \begin{align}\label{ts:eff_decomposition}
        \mathbb{P}(A_t = k, k^* \in \Perm_t) &= \mathbb{P}(A_t = k, k^* \in \Perm_t, \good{\mu,k}t, \good{\rho,k}t) + \mathbb{P}(A_t = k, k^* \in \Perm_t, \good{\mu,k}t, (\good{\rho,k}t)^c) + \mathbb{P}(A_t = k, k^* \in \Perm_t, (\good{\mu,k}t)^c) \notag \\
        &\le \mathbb{P}(A_t = k, k^* \in \Perm_t, \good{\mu,k}t, \good{\rho,k}t) + \mathbb{P}(A_t = k, \good{\mu,k}t, (\good{\rho,k}t)^c) + \mathbb{P}(A_t = k, (\good{\mu,k}t)^c).
        \end{align}
        
        Now, the last of these terms in (\ref{ts:eff_decomposition}) is easily controlled - indeed, $\mathbb{P}(A_t = k, (\good{\mu,k}t)^c) = \mathbb{P}(A_t = k, \muhat > u)$ is exponentially small if $N_t^k$ is large. In fact, mirroring the approach of the proof of Lemma \ref{lem:klucb_key_bound}, we find that \begin{align*}
            \sum_{t \le T} \indi\{A_t = k, \muhat > u\} &= \sum_{t\le T} \sum_{s \le t} \indi\{A_t = k, N_t^k = s, \muhat > u\} \\
                                                &= \sum_s \indi\{ \hat\mu^k(s) > u\} \sum_{t \ge s} \indi\{A_t = k, N_t^k = s\} \\
                                                &\le \sum_{s\le T} \indi\{ \hat\mu^k(s) > u\},
        \end{align*}
        where we set $\hat\mu^k(s)$ to be the value of $\muhat $ at the first $t$ such that $N_t^k = s$. But then, by Chernoff's bound, $P(\hat\mu^k(s) > u) \le \exp(-s d(u\|\mu^k),$ giving the bound \begin{equation}\label{eq:ts_third_term}
            \sum \mathbb{P}(A_t = k, (\good{\mu,k}t)^c) \le \frac{2}{1 \wedge d(u\|\mu^k)}.
        \end{equation}
        
        The second term of (\ref{ts:eff_decomposition}) too is similar to control, upon observing that the posterior $\mathrm{Beta}$ law is very well concentrated around $\muhat$ with variance scale $1/N_t^k.$ More concretely, \citet{agrawal2013further} exploit the following observation: if $F(x; \mathrm{Beta}(a,b))$ is the CDF of a $\mathrm{Beta}(a,b)$ random variable, and $G(k; \mathrm{Bin}(n,p))$ is the CDF of a Binomial random variable, then for natural $n \ge k,$ \[ 1- F(x; \mathrm{Beta}(k+1, n - k+1)) = G(k; \mathrm{Bin}(n+1,x) ). \] This relation most easily follows from the fact that the $\mathrm{Beta}(k+1, n-k+1)$ is the law of the $k+1$th order statistic of $n+1$ samples from the uniform distribution, and the chance of this exceeding $x$ is simply the chance that the $k$ smaller ones are at most $x$, and the rest are at least $x$, which of course is expressed by the Binomial distribution. But then we conclude that for any $N_0$ \begin{align*}
            \mathbb{P}( \rho_t^k > v |N_t^k > N_0, \hat\mu_t^k \le u) \le e^{-N_0 d(v\|u)}.
        \end{align*}

        Choosing $N_0 = \log(T)/d(v\|u),$ we then get the bound \begin{align} \label{eq:ts_second_term} \sum_{t = 1}^T \mathbb{P}(A_t = k, \rho_t^k > v, \muhat \le u) 
        &\le \sum_{t = 1}^T \mathbb{P}(A_t = k, N_T^k \le N_0) + \sum_{t = 1}^T \mathbb{P}(N_T^k > N_0, \rho_t^k > v, \muhat \le u) \notag \\
        &\le N_0 + T{e^{-N_0 d(v\|u)}} \notag \\
        &\le \frac{\log T}{d(v\|u)} + 1 \le \frac{\log T}{d(v\|u)} + \frac{1}{1 \wedge d(u\|\mu^k)}.\end{align}

        This leaves the first term of (\ref{ts:eff_decomposition}), which is the hardest to control, and ultimately relies upon hard analysis of Binomial tails. The idea is roughly to use $v$ as a lower index for $\rho^*_t$. Indeed,  let \[ \mathbf{P}_t := \mathbb{P}( \rho_t^* >  v | \hist_{t-1}) = \mathbb{P}_{t-1}(\rho_t^* > v).\] Then observe that 
        
        \begin{align*} \mathbb{P}_{t-1}(A_t = k, \good{\mu,k}t, \good{\rho,k}t , k^* \in \Perm_t) &= \indi\{\good{\mu,k}t, k^* \in \Perm_t\}\mathbb{P}_{t-1}( A_t = k, \rho_t^k < v ) \\
        &\le \indi\{\good{\mu,k}t, k^* \in \Perm_t\} \mathbb{P}_{t-1}(\forall k \in \Perm_t, \rho_t^k < v)\\ 
        &= \indi\{\good{\mu,k}t, k^* \in \Perm_t\}(1-\mathbf{P}_t) \mathbb{P}_{t-1}(\forall k \neq k^* \in \Perm_t, \rho_t^k < v) \\
        &= \frac{1 - \mathbf{P}_t}{\mathbf{P}_t} \indi\{\good{\mu,k}t, k^* \in \Perm_t\} \mathbb{P}_{t-1}(\rho_t^* > v, \forall k \neq k^* \in \Perm_t, \rho_t^k < v) \\
        &\le \frac{1 - \mathbf{P}_t}{\mathbf{P}_t} \mathbb{P}_{t-1}(A_t  = k^*),\end{align*} where we have used the fact that $\good{\mu,k}t \in \hist_{t-1}$ and $\Perm_t$ is predictable. The idea is to now exploit the fact that $\mathbf{P}_t$ is exponentially close to $1$ as $N_t^*$ increases, and by expressing this chance in terms of the size of $N_t^*$ and analysing the same, \citet{agrawal2013further} show in their Lemma 2 that \begin{equation*}
            \sum_{t = 1}^T\mathbb{E}[ (1-\mathbf{P}_t) \mathbb{P}_{t-1}(A_t = k^*)/\mathbf{P}_t] \le \frac{24}{\Delta_v^2} + C' \sum_{s \ge 8/\Delta_v}^{T - 1} e^{-\Delta_v^2 s/2} +  \frac{1}{e^{\Delta_v^2 s/4} - 1} + \frac{e^{-sd(v\|\mu^*)}}{(s+1)\Delta_v^2}, 
        \end{equation*}
        where $\Delta_v := (\mu^* - v)$ and $C'$ is a constant. Notice that each of the terms in the sum are monotonically decreasing. Therefore, we may derive upper bounds by comparison to an integral, which yields for the first and second terms that \begin{align*}
            \sum_{s = \lceil 8/\Delta_v \rceil}^{T-1} e^{-\Delta_v^2 s/2} \le \int_{0}^\infty e^{-\Delta_v^2 s/2} \mathrm{d}s = \frac{2}{\Delta_v^2},
        \end{align*}
        and \begin{align*}
            \sum_{s = \lceil 8/\Delta_v \rceil}^{T-1} \frac{1}{e^{\Delta_v^2 s/4} - 1} &\le \int_{7/\Delta_v}^{T} \frac{1}{e^{\Delta_v^2 s/4} - 1} \mathrm{d}s \\ &= \frac{4}{\Delta_v^2} \int_{(\nicefrac74\Delta_v}^{\Delta_v^2 T/4} \frac{1}{e^{u} - 1} \mathrm{d}u \\
            &\le \frac{4}{\Delta_v^2} \log \frac{1}{1-e^{-\nicefrac74\Delta_v}}\\
            &\le \frac{4}{\Delta_v^2} \log \frac{2}{1 \wedge \nicefrac74\Delta_v} \le \frac{4}{\Delta_v^2}\left( \log \frac{1}{\Delta_v} + O(1)\right),
        \end{align*}
        where we have used the previously established fact that $\frac{1}{1-e^{-x}} \le \frac{2}{x \wedge 1}$.
        
        For the final term, we may bound this in two ways - firstly simply observing that $e^{-sd} \le 1,$ we get the bound $\sum_{8/\Delta_v}^{T-1} \frac{1}{s+1} \le \log (T\Delta/8).$ In addition, we derive a $T$-independent bound as follows, wherein we abbreviate $d_v = d(v\|\mu^*).$
        
        \begin{align*}
            \sum_{s = \lceil 8/\Delta_v \rceil}^{T-1} \frac{e^{-sd_v}}{(s+1)\Delta_v^2} &= e^{d_v}\sum_{s = \lceil 8/\Delta_v}^{T-1} \frac{e^{-(s+1)d_v}}{s+1} \\
            &= e^{d_v} \sum_{s = \lceil 8/\Delta_v\rceil}^{T-1} \int_{u = d_v}^\infty e^{-(s+1)u} \mathrm{d}u \\
            &\le e^{d_v} \int_{u = {d_v}}^\infty \sum_{ s= 1}^{\infty} e^{-(s+1)u} \mathrm{d}u \\
            &= e^{d_v} \int_{u = d_v}^\infty \frac{e^{-u} }{e^u - 1} \mathrm{d}u\\
            &\le  \log \frac{1}{1-e^{-d_v}} \le \log \frac{2}{d_v} + O(1).
        \end{align*}
        
        Taking the smaller of these two bounds, the final term is controlled by $4\Delta_v^{-2}[\log( T\Delta_v \wedge d_v^{-1}) + O(1)]$, and we have
        \begin{equation}\label{ts:first_term}
            \sum_{t = 1}^T \mathbb{P}(A_t = k, \good{\mu,k}t, \good{\rho, k}t, k^* \in \Perm_t) \le \frac{C}{\Delta_v^2}\left( 1 + \log \frac{1}{\Delta_v} + \log \left( \Delta_v T \wedge -d(v\|\mu^*)^{-1}\right) \right).
        \end{equation}
        
        The claimed bound is then realised by adding up (\ref{eq:ts_third_term}, \ref{eq:ts_second_term}, \ref{ts:first_term}).
\end{proof}

\section{Proofs for Thompson Sampling with \BayesUCB}\label{appx:ts+bayes_ucb}

Since the procedure for selecting arms given $\Perm_t$ is left unchanged from the previous case, we only need to demonstrate that $\Perm_t$ is good, that is, that the lower bound index $L_t^k$ performs well. Indeed, this is essentially exploiting the fact that the argument of the previous section only uses the fact that $\Perm_t$ is a predictable process, and then specifics of the Thompson scores $\rho_t^k$s, and so the second term of the decomposition  \[\sum_{t} \mathbb{P}(A_t = k) \le \sum_t \mathbb{P}(k^* \not\in \Perm_t) + \sum_t \mathbb{P}(k^* \in \Perm_t, A_t = k)\] can be pursued identically to control the play of inefficient arms on rounds such that $k^* \in\Perm_t$, again giving (\ref{eq:TS_efficiency}). 

We show the following bound, following the methods of \citet{kaufmann2012bayesUCB} as described in \S\ref{appx:strat}. \begin{mylem}
    In the setting of Theorem \ref{thm:TS+BayesUCB}, the following hold. \begin{itemize}\label{lem:bayesucb}
        \item If $\Gamma^k > 0,$ then for any $x \in (\alpha, \nu^k),$ \begin{equation}\label{eq:bayesucb_safety_control}
            \mathbb{E}[N_T^k] \le \frac{\nicefrac{3}{2} \log T + 3 \log\log T + \nicefrac{3}{2}\log 2}{d(x\|\alpha)} + \frac{2}{1\wedge d(x\|\nu^k)}
        \end{equation}
        \item The mean number of times the optimal arm is treated as impermissible is bounded as \[ \sum_{t = 3}^T \mathbb{P}(k^* \not\in\Perm_t) \le e \log\log T + 4e.\]
    \end{itemize}
\end{mylem}

The claimed bound is quickly forthcoming upon combining the appropriate pieces of the proofs of Theorems \ref{thm:pure_klucb} and \ref{thm:stop_gap_dep}. 

\begin{proof}[Proof of Theorem \ref{thm:TS+BayesUCB}]
    For inefficient arms, combining the second part of Lemma \ref{lem:bayesucb} and (\ref{eq:TS_efficiency}), we conclude that if $\Delta^k > 0,$ then \[ \mathbb{E}[N_T^k] \le \frac{\log T}{d(u\|v)} + \frac{3}{1 \wedge d(u\|\mu^k)} + \frac{C}{(\mu^* - v)^2} (1 + (d(v\|\mu^*)^{-1} \wedge \log T) ) + e \log\log T + 4e. \]
    
    Similarly, for infeasible arms, by using (\ref{eq:bayesucb_safety_control}), we have the control \[\mathbb{E}[N_t^k] \le \frac{\log T + 3\log\log T + 2\log 2}{\nicefrac23 d(y\|\alpha)} + \frac{2}{1\wedge d(y\|\nu^k)}. \]
    
    Now choosing $u,v,y$ as in the proof of Theorem \ref{thm:stop_gap_dep} and proceeding along the same lines gives the claim.
\end{proof}

The same approach also shows the following gap-independent result. The proof is identical, and so omitted. \begin{myprop}
    Algorithm \ref{alg:TBU} instantiated with \BayesUCB with $\delta_t^k = 1/\sqrt{8N_t^k} t\log^3 t$ also satisfies the bound \[ \mathbb{E}[\mathcal{R}_T] = O(\sqrt{KT \log T} + K \log\log T).\]
\end{myprop}

We conclude by showing the main Lemma.

\begin{proof}[Proof of Lemma \ref{lem:bayesucb}]
    The argument relies on the following estimate, which essentially serves as a reduction to the analysis of \KLUCBNS. This result is a variation of Lemma 1 of \citet{kaufmann2012bayesUCB}.
    \begin{mylem}\label{lem:bayesucb_quantile_bounds}
        Define the quantities \begin{align*}  \underline\varphi_t^k &:= \indi\{S_t^k > 0\} \min\left\{q \le \frac{S_t^k}{N_t^k} : N_t^k d\left(\frac{S_t^k}{N_t^k} \middle\|q\right) \le \log((2 t \log^2t )^{3/2})\right\} \\  \overline\varphi_t^k &:= \indi\{S_t^k > 0\} \min \left\{q \le  \frac{S_t^k}{N_t^k} : N_t^k d\left(\frac{S_t^k}{N_t^k} \middle \| q\right) \le \log(t \log^3(t) ) \right\} .\end{align*}
        Then for all $t,$ \[ \underline\varphi_t^k \le L_t^k \le \overline\varphi_t^k.\]
        \begin{proof}
            Firstly, since $L_t^k = 0$ whenever $S_t^k = 0,$ this case is trivial. So assume $S_t^k \ge 1$.
            
            The idea behind the bounds is to exploit the relationship between the CDFs of Beta and Binomial random variables to reduce the quantile estimation to that of a Binomial, and then use Chernoff's bound for the Binomial to control where the quantile can be. Indeed, let $Z \sim \mathrm{Beta}(S_t^k, N_t^k - S_t^k + 1)$. Then we know that \[ \mathbb{P}(Z \le q) = \mathbb{P}(\mathrm{Bin}(N_t^k, q) \ge S_t^k). \]
            Further, by Chernoff's upper bound, and by estimating the $s$th term in the Binomial series using Stirling's approximation, we may show the following result (where the lower bound holds generally, and the upper bound holds for any $s \ge nq$).  \[ \frac{1}{\sqrt{8n}} \exp( - n d( (s/n)\|q) ) \le \mathbb{P}(\mathrm{Bin}(n,q) \ge s) \le \exp(-nd((s/n)\|q)).\]
            
            Now, recall that $L_t^k$ is the $\delta_t^k$th quantile of the law of $Z$, so that $P(Z \le L_t^k) = \delta_t^k$. 
            
            \emph{Lower bound} Suppose $q \le S_t^k/N_t^k$ is such that \[ \exp(- N_t^kd( S_t^k/N_t^k)\|q) ) \le \delta_t^k.\] Then it follows that $q \le L_t^k$. Therefore, \begin{align*}
                L_t^k &\ge \max \left\{ q \le \frac{S_t^k}{N_t^k} : N_t^k d\left(\frac{S_t^k}{N_t^k} \| q\right) \ge \log(1/\delta_t^k) \right\}\\ &= \min\left\{q \le \frac{S_t^k}{N_t^k} : N_t^k  d\left(\frac{S_t^k}{N_t^k}\| q\right) \le \log(1/\delta_t^k) \right\},
            \end{align*}  where the final equality is due to the continuity of $d(a\|\cdot).$ 
            
            Now observe that \[ \log(1/\delta_t^k) \le (2(t+1))^{3/2} \log^3 t.\] Therefore, replacing $\log(1/\delta_t^k)$ by the larger $\log( 2(t+1)^{3/2}\log^3 t$ in the lower bound can only decrease it.
            
            \emph{Upper bound} Suppose that $q \le S_t^k/N_t^k$ is such that the lower bound on the Binomial tail exceeds $\delta_t^k$. Then $L_t^k$ must be smaller than this $q$, and so \[ L_t^k \le \min\left\{ q\le \frac{S_t^k}{N_t^k}: N_t^k d\left(\frac{S_t^k }{N_t^k}\| q\right) \le \log\left( \frac{1}{\sqrt{8N_t^k}\delta_t^k}\right) \right\}.\]
            
            But, by definition, \[\frac{1}{\sqrt{8N_t^k} \delta_t^k} = t\log^3 t.  \]

        \end{proof}
    \end{mylem}

    Observe that the bounds $\overline\varphi$ and $\underline\varphi$ exactly take the form of the \KLUCB bounds, but with a different value for $\gamma_T$. Thus, the same proofs may be repeated. 
    
    Indeed, to show (\ref{eq:bayesucb_safety_control}), we observe that for an arm with a safety gap, $\{L_t^k \le \alpha\} \subset \{ \underline\varphi_t^k \le \alpha,\}$ and we may then follow the proof of Lemma \ref{lem:klucb_key_bound} to control this identically to there - the only change is that $\log(\gamma_T)$ in $S(y)$ is replaced by $\log ( (2 t \log^2 t)^{3/2}).$
    
    Further, the upper bound is exactly the bound of \KLUCBNS, and therefore  without alteration we may immediately conclude that \begin{align*}
        \sum_{t \ge 3} \mathbb{P}(L_t^* > \alpha) &\le \sum_{t \ge 3} \mathbb{P}(\overline\varphi^*_t > \alpha) \le e\log\log t + 4e. \qedhere
    \end{align*}
    
    We note that the last property in the proof of Lemma \ref{lem:bayesucb_quantile_bounds} is exactly the reason for selecting $\delta_t$ of the form that we did, which is essentially the $1/\gamma_t$ from \KLUCBNS, but scaled down to ensure that the \BayesUCB bound is at least as optimistic as that of \KLUCBNS. In principle, then, this gives an avenue for a tighter analysis by choosing a more refined notion of $\delta_t$ by exploiting stronger bounds for the Binomial tails. 
    
    For instance, it is known \citep[Prop A.4, A.2][]{jevrabek2004dual} that there exists a constant $C$ such that for $s \ge nq +\sqrt{nq(1-q)}, $ \[ \frac1C \frac{qn - qs}{s - qn } \sqrt{\frac{n}{s(n-s)}} e^{-n d(s/n\|q)} \le \mathbb{P}(\mathrm{Bin}(n,q) \ge s) \le C \frac{qn - qs}{s - qn } \sqrt{\frac{n}{s(n-s)}} e^{-n d(s/n\|q)},  \] while for $s \le nq +  \sqrt{nq (1-q)},$ it is bounded below by another constant $C'$. This suggests using $\delta_t \sim \min\left( C',  \frac{1}{t \log^3 t} \cdot \sqrt{\frac{N_t^k}{S_t^k (N_t^k - S_t^k)}}\right),$ although it is unclear how to handle the $(qn - qs)/( s-qn)$ term properly. Assuming this is indeed handled, though, this should result in an improvement to $\overline \varphi$ of replacing the $t^{3/2}$ by something $O(t),$ while the lower bound should remain unchanged. Of course, this does not quite explain the success of $\delta_t = 1/t$ in the experiments, and it is possible that this approach simply serves to make \BayesUCB look more like \KLUCBNS, which defeats the purpose somewhat.
\end{proof}

\section{Lower Bound}

We begin by showing the key Lemma. 

\begin{proof}[Proof of Lemma \ref{lemma:lower_bound_tech}]
    Fix a (possibly randomised) algorithm. Let $\{ \bbP^k \}$ and $\{\widetilde{\bbP}^k\}$ be two safe bandit instances, and recall that  $\hist_{t} := \{(A_s,R_s, S_s): s \le t\}$ denotes the history of play. We will use $\mathbb{P}$ to represent laws in the first instance and $\widetilde{\bbP}$ for laws in the second. Similary, $\mathbb{E}$ and $\widetilde{\mathbb{E}}$ denote expectations under the two laws.
    
    Let $Z$ be any function of measurable with respect to $\sigma(\hist_{T})$ that is bounded in $[0,1]$. Then observe that from $\hist_{T}$, we can generate a random bit by first computing $Z(\hist_{T+1}),$ and then sampling $B \sim \mathrm{Bern}(Z).$ Clearly, the mean of $B$ is the same as that of $Z$. But then, by the data processing inequality,\[ D( \mathbb{P}_{\hist_{T}} \| \widetilde{\mathbb{P}}_{\hist_{T}} ) \ge D(\bbP_B \|\widetilde\bbP_B) = d( \mathbb{E}[Z] \| \widetilde{\mathbb{E}}[Z]).\] 
    
    Next, due to the chain rule of KL divergence, for any $t \ge 1$, \begin{align*} D(\bbP_{\hist_{t}} \| \widetilde\bbP_{\hist_{t}}) &= D(\bbP_{\hist_{t-1}} \| \widetilde\bbP_{\hist_{t-1}})  \\ &\qquad   + \mathbb{E}[ D(\bbP_{A_t| \hist_{t-1}}\|\widetilde\bbP_{A_t |\hist_{t-1}} | \hist_{t-1} ) ] \\&\qquad\qquad + \mathbb{E}[ D(\bbP_{(R_t, S_t)|A_t,  \hist_{t-1}}\|\widetilde\bbP_{(R_t, S_t)|A_t,  \hist_{t-1}} |A_t,\hist_{t-1} ) ]. \end{align*}

Now, the second term in the RHS is $0$ since the learner must be causal, and thus the law of $A_t$ is determined by $\hist_{t-1}.$ Further, the feedback $(R_t, S_t)$ is independent of the history given $A_t$, and is distributed according to $\bbP^{A_t}$ and $\widetilde\bbP^{A_t}$ under the two instances. We thus have the recurrence \[ D(\bbP_{\hist_{t}} \| \widetilde\bbP_{\hist_{t}}) - D(\bbP_{\hist_{t-1}} \| \widetilde\bbP_{\hist_{t-1}}) = \sum_k \mathbb{P}(A_t = k) D(\bbP^k \|\widetilde\bbP^k). \]

Summing this up, and observing that $\hist_0$ is trivial, and then recalling $\sum_t \mathbb{P}(A_t = k) = \mathbb{E}[N_T^k] ,$  it follows that \[D(\bbP_{\hist_{T}}\|\widetilde{\bbP}_{\hist_{T}}) =  \sum_k \mathbb{E}[N_T^k] D(\bbP^k\|\widetilde\bbP^k). \] The conclusion now follows on taking $Z = N_T^k/T$, which trivially lies in $[0,1]$. \end{proof}

\begin{proof}[Proof of \cref{thm:lower_bound}]
    As mentioned in the main text, choose $\widetilde{\bbP}^j = \bbP^j$ for $j \neq k,$ and instead let $\widetilde\bbP^k$ be any law on $\{0,1\}^2$ of means $(\mu^k \vee \mu^* + \varepsilon, \nu^k \wedge \alpha).$ Notice that in the $\widetilde\bbP$-instance, arm $k$ is optimal.
    
    Since the algorithm ensures that suboptimal arms are not played more than $C_x T^x$ times, $\mathbb{E}[N_T^k/T] \le C_xT^{-(1-x)}, $ and $\widetilde{\mathbb{E}}[N_T^k/T] \ge 1 - C_xT^{-(1-x)}$ for any $x \in (0,1)$. Therefore, \begin{align*} d(\mathbb{E}[N_T^k/T]\|\widetilde{\mathbb{E}}[N_T^k/T] ) &\ge \left( 1- \frac{\mathbb{E}[N_T^k]}{T}\right) \log \frac{1}{1- \widetilde{\mathbb{E}}[N_T^k/T]}   - \log 2\\
    &\ge (1 - o(1)) (1-x) \log \frac{T}{C_x} - \log 2 = (1-o(1))(1-x)\log T. \end{align*} 
    
    Next, since we are working with independent means and safety rewards, taking $\widetilde{\bbP}^k$ to also have the independent rewards, we get $D(\bbP^k\|\widetilde{\bbP}^k) = d_<(\mu^k\|\mu^* + \varepsilon) + d_>(\nu^k\|\alpha).$
    
    We conclude that for any $x,\varepsilon \in (0,1),$ \[\frac{\mathbb{E}[N_T^k]}{\log T} \ge \frac{(1-x)(1-o(1)) }{d_<(\mu^k\|\mu^* + \varepsilon) + d_>(\nu^k\|\alpha)},  \] whence the claim follows on taking $\varliminf_{T \nearrow \infty},$ and then taking limits as $x \to0, \varepsilon \to 0,$ and exploiting the continuity of $d_<(a\|b)$.    
\end{proof}

\section{Simulation Details and Supplementary Plots}\label{appx:sims}

\textbf{Implementation Details} All methods are implemented on MATLAB. Throughout we use independent Bernoulli bits for both $R$ and $S$. The particular details of the methods used are described below.

\emph{Policy approaches} It is a straightforward observation that for a single constraint and objective, the solution to linear program $\max_{\pi \in \Delta} \langle \pi, a\rangle \textrm{ s.t. } \langle \pi, b\rangle \le c$ is supported on at most two coordinates. Further, the optimal policy on two given coordinates itself is simple to compute - clearly at least one needs to be safe according to the relevant safety index at the particular time, else this is not a permitted policy. If both are safe as per the index, then the policy can concentrate on the one with larger reward index. If one is safe and the other not, then the policy concentrates on the safe one if it has a larger reward index. Otherwise, we assign the slack between the safety level and the safety index of the safe coordinate as the mass of the policy on the coordinate with the unsafe index. This enables a simple - and fast - method to select the round-wise policies for both \textsc{BwCR} and \textsc{Pess} - we simply evaluate the value of the optimal policy on each pair of arms, and choose the one with the largest reward.

\emph{Details of Confidence Bound Computation}

In effect we use two types of confidence bounds - \KLUCBNS-based, and \BayesUCBNS-based. \begin{itemize}
    \item \KLUCBNS-based bounds are all evaluated with $\gamma_t = 1/t$ (i.e., without the extra $1/\log^3 t$ factor in the main text). This is aligned with the practical recommendations of \citet{garivier2011kl}. 
    
    The upper indices $U^k_t$ on $\mu^k$ are computed simply by computing a lower bound for $1-\mu^k$, and then subtracting this from one. The soundness of this procedure is a trivial exercise.
    
    Finally, the KL inversion is performed via a binary search. Specifically, we carry this out for $\max(4, \log_2(t))$ rounds, thus ensuring that any error in the estimate is of order $1/t,$ which ensures that extra regret due to numerical precision is at most $\log T.$
    
    \item \BayesUCB-based bounds are all evaluated with $\delta_t^k = 1/(t+1)$. Again, this is in line with the recommendations of \citet{kaufmann2012bayesUCB}. We note that this is a larger quantile than studied in the main text, and a regret bound with this $\delta_t^k$ is currently unavailable. Nevertheless, the empirical performance is sound, as seen in \S\ref{sec:sims}.
    
    The quantile estimation is performed by using the library \texttt{betainv} function provided by the Statistics Toolbox of MATLAB. This uses Newton's method to solve the equation defining a quantile of a Beta distribution.
    
    \item For \ThompNS, we sample from the appropriate Beta posteriors by using the library \texttt{betarnd} function provided by MATLAB. 
\end{itemize}

\subsection{Supplement to \S\ref{sec:exp_policy}}

We provide plots that detail the regrets achieved by each algorithm in the two cases studied. The main observations remain unchanged - the regrets of policy based methods grow linearly in the first case, and while they appear sublinear inthe second, they are at a much larger scale than our implementation. We note that in both cases the more unsafe \textsc{BwCR} performs better on the regret criterion. This should be evident on the data of case two, for which playing the unsafe arm only contributes $0.6- 0.5 = 0.1$ to the regret, while the suboptimal arm has a gap of $0.6 - 0.4 = 0.2$. However, the data of case 1 suggests that this is also true more broadly, and may be an effect of the optimism principle. That said, this is a moot point in this case since hte growth rate is very much linear.

\begin{figure}[h]
    \centering
    \includegraphics[width = 0.49\textwidth]{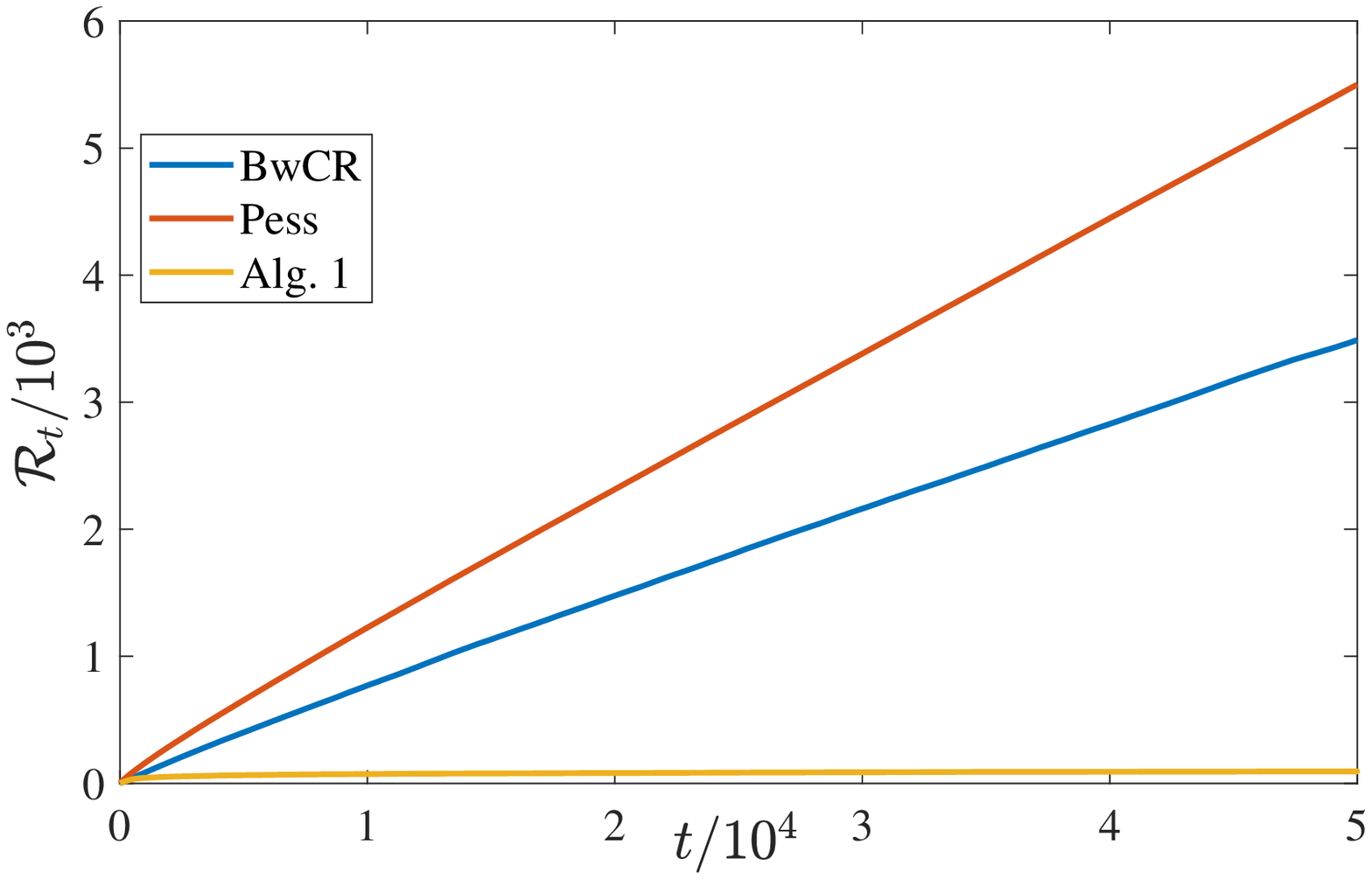}
    \includegraphics[width = 0.49\textwidth]{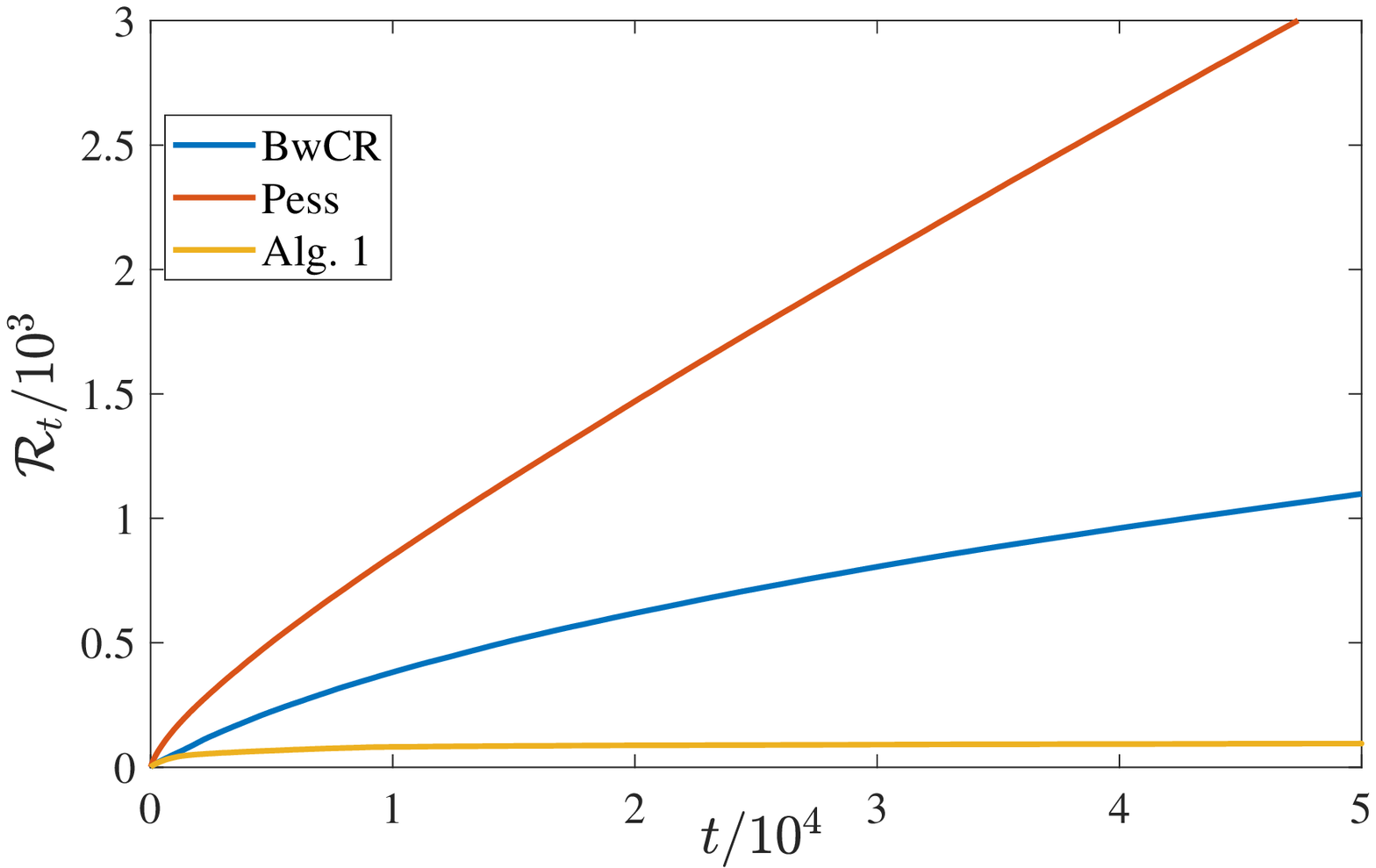}
    \caption{Regrets for the situations of \S\ref{sec:exp_policy} - left is the first case with two optimal policies, right is the second case with a single optimal policy supported on a single arm}
    \label{fig:policy_regret}
\end{figure}

\subsection{Supplement to \S\ref{sec:exp_probing}}\label{appx:probing}

We first provide Box plots in Figure \ref{fig:trial_alpha_0_21_box_plots} of the spread of regret and safety for the situation studied in the main text with $\alpha = 0.21$. Note that the Regret of the \Thomp based methods shows somewhat larger fluctuations, although the maximum of the data is similar. For the net safety violation, the fluctuations are similarly sized, and the Bayesian methods retain an advantage. 

\begin{figure}[htb]
    \centering
    \includegraphics[width = 0.49\textwidth]{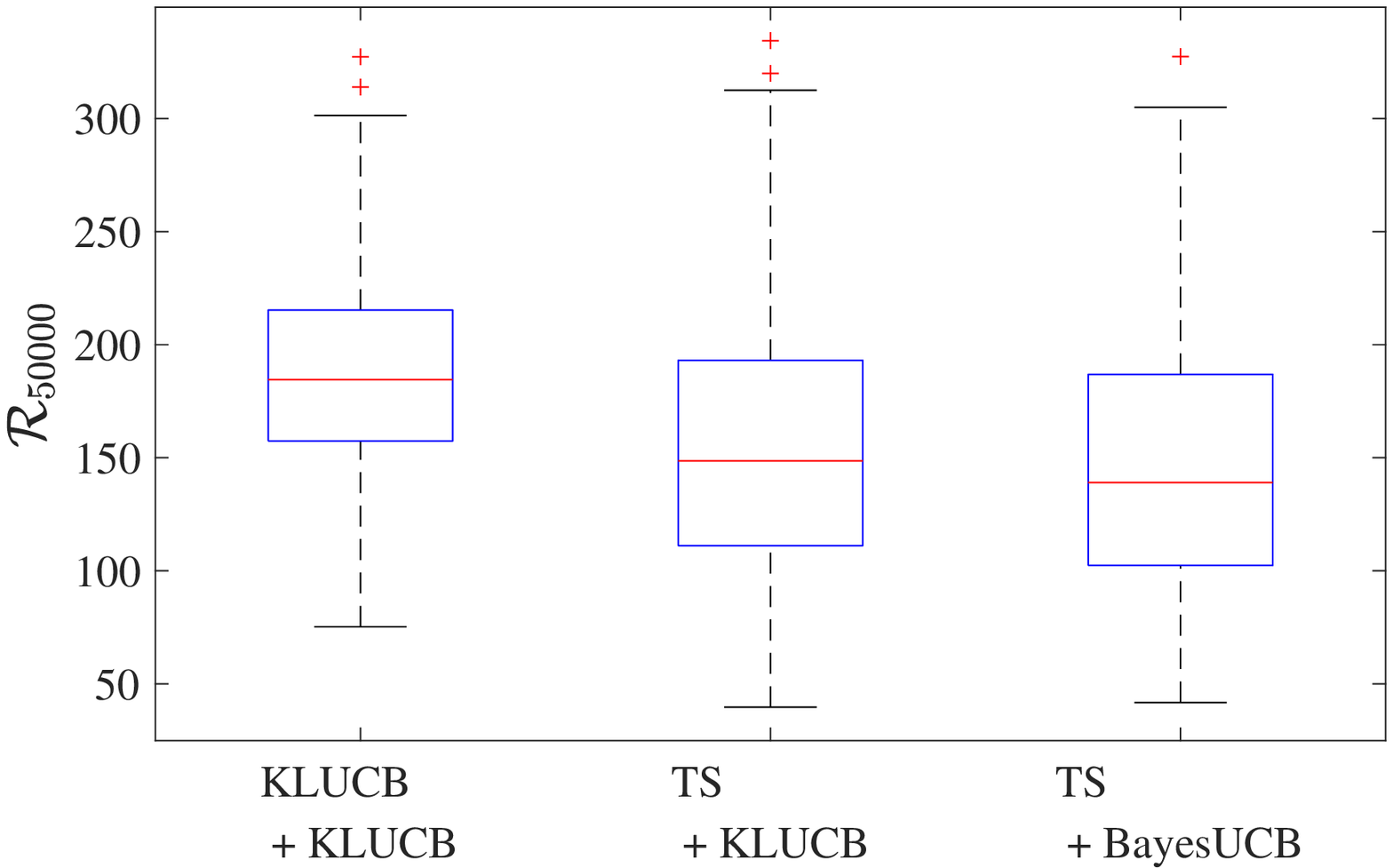}~\includegraphics[width = 0.49\textwidth]{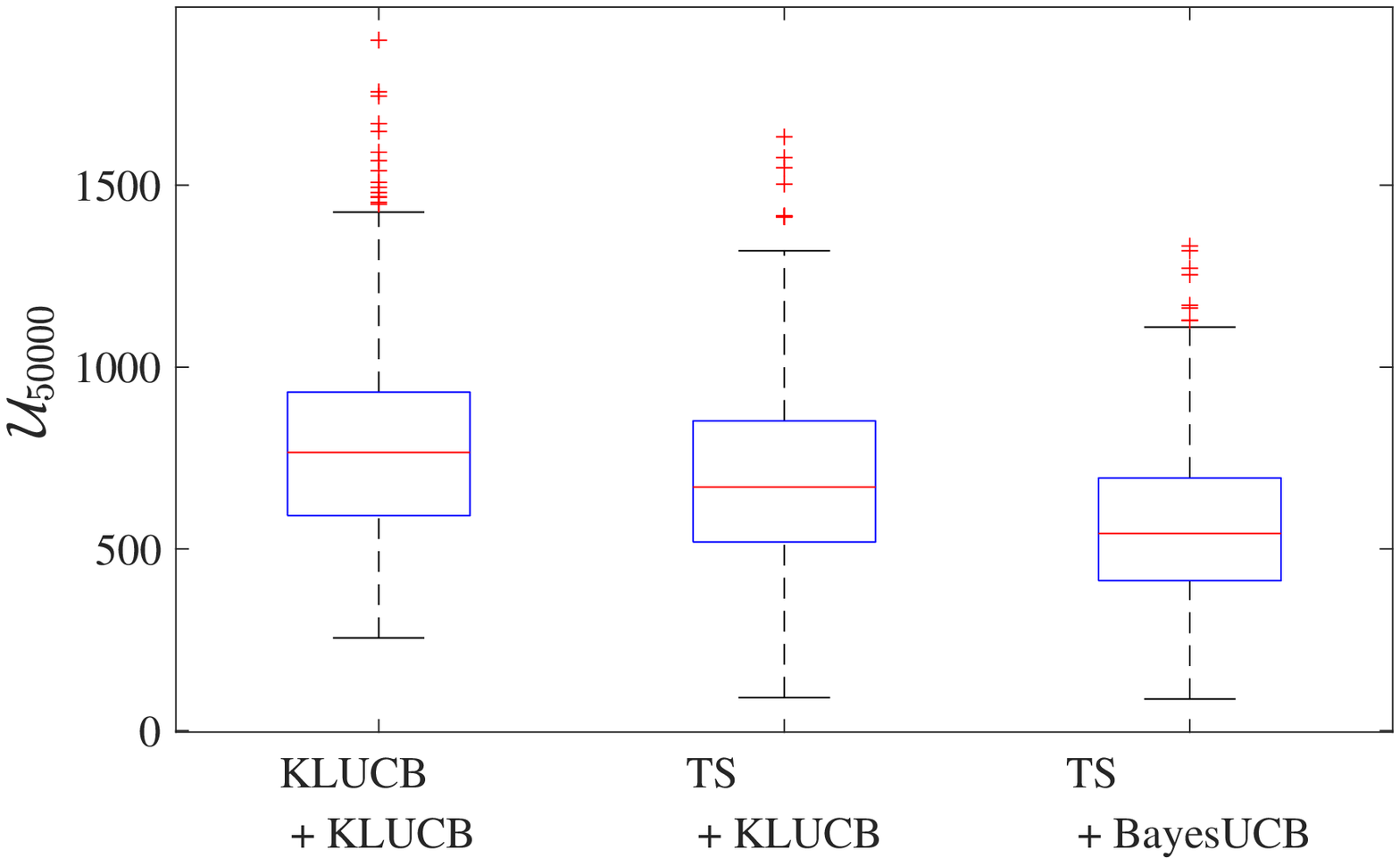}
    \caption{Box plots of Regret and Total Safety Violation at time $T = 50K$ over $500$ runs for the Trial Drugs data with $\alpha = 0.21$}
    \label{fig:trial_alpha_0_21_box_plots}
\end{figure}

Next, we provide plots for the same scenario, but with  $\alpha = 0.19.$ Note that this induces the difference that the arm $4$ is now unsafe by a significant amount, which increases its gap $\Delta^4 \vee \Gamma^4$ to about $0.02$ from $0.004$. However, since $0.004$ is about the same size as $\sqrt{K/T} = 0.01,$ this arm was not contributing much to the regret in the previous case. Further, the safety violation of the least unsafe arm is not only about $0.02$ instead of the previous $0.05$. Correspondingly, we expect to see an increase in the play of unsafe arms, as well as a slight increase in regret due to the scale up from $0.04$ to $0.019$ in the play of this arm. Both of these observations are clearly borne out in Figure \ref{fig:alpha_0_19}, which presents data over 100 trials.

We note that these observations are again consistent with the theoretical bounds. The main term of the regret bound is roughly $40 \log t$, while that of the safety violation bound is roughly $1500\log t,$ and $\log(10^4) \approx 4 \cdot 2.3 \approx 10$. 

\begin{figure}[H]
    \centering
    \includegraphics[width = 0.49\textwidth]{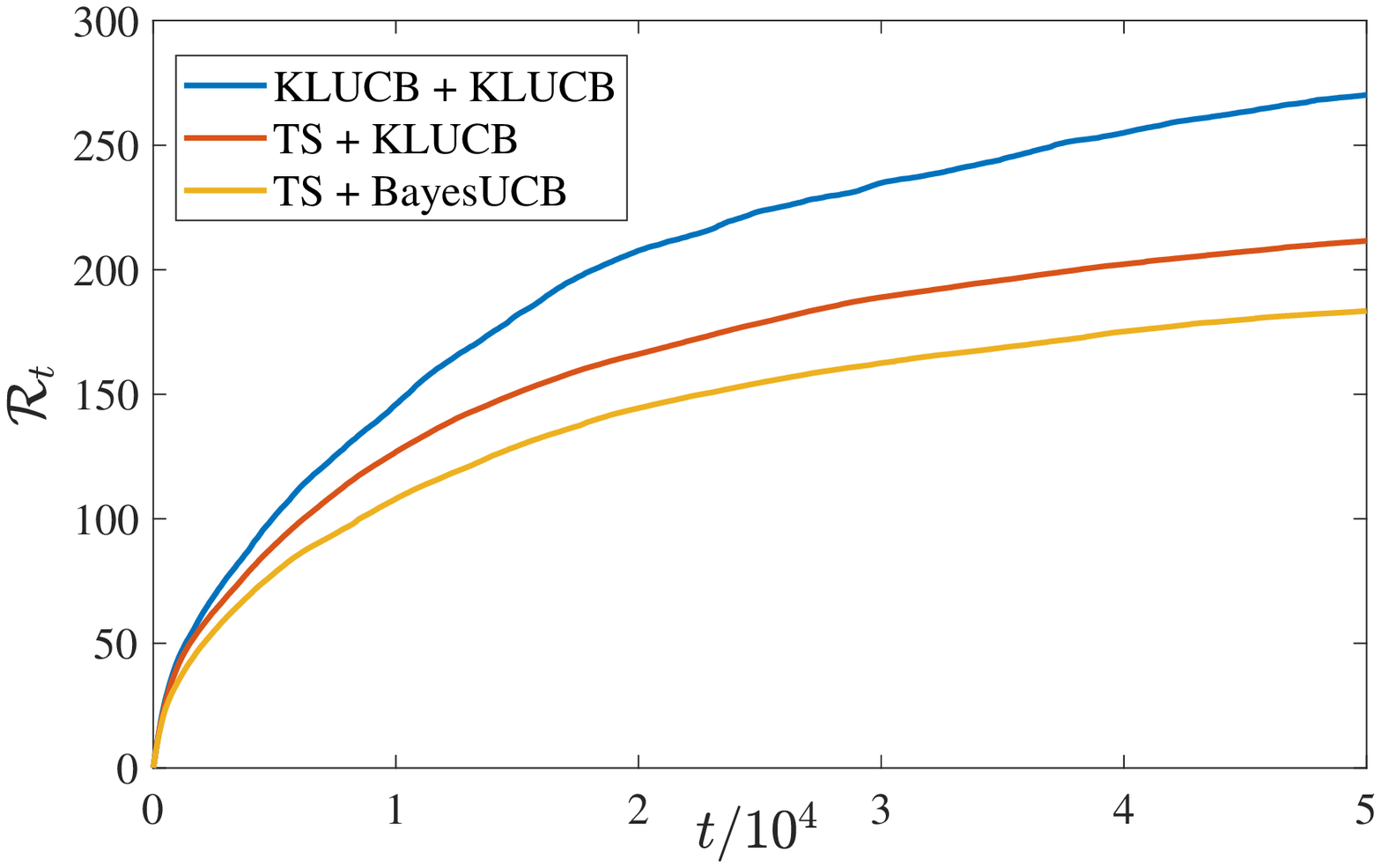} \includegraphics[width = 0.49\textwidth]{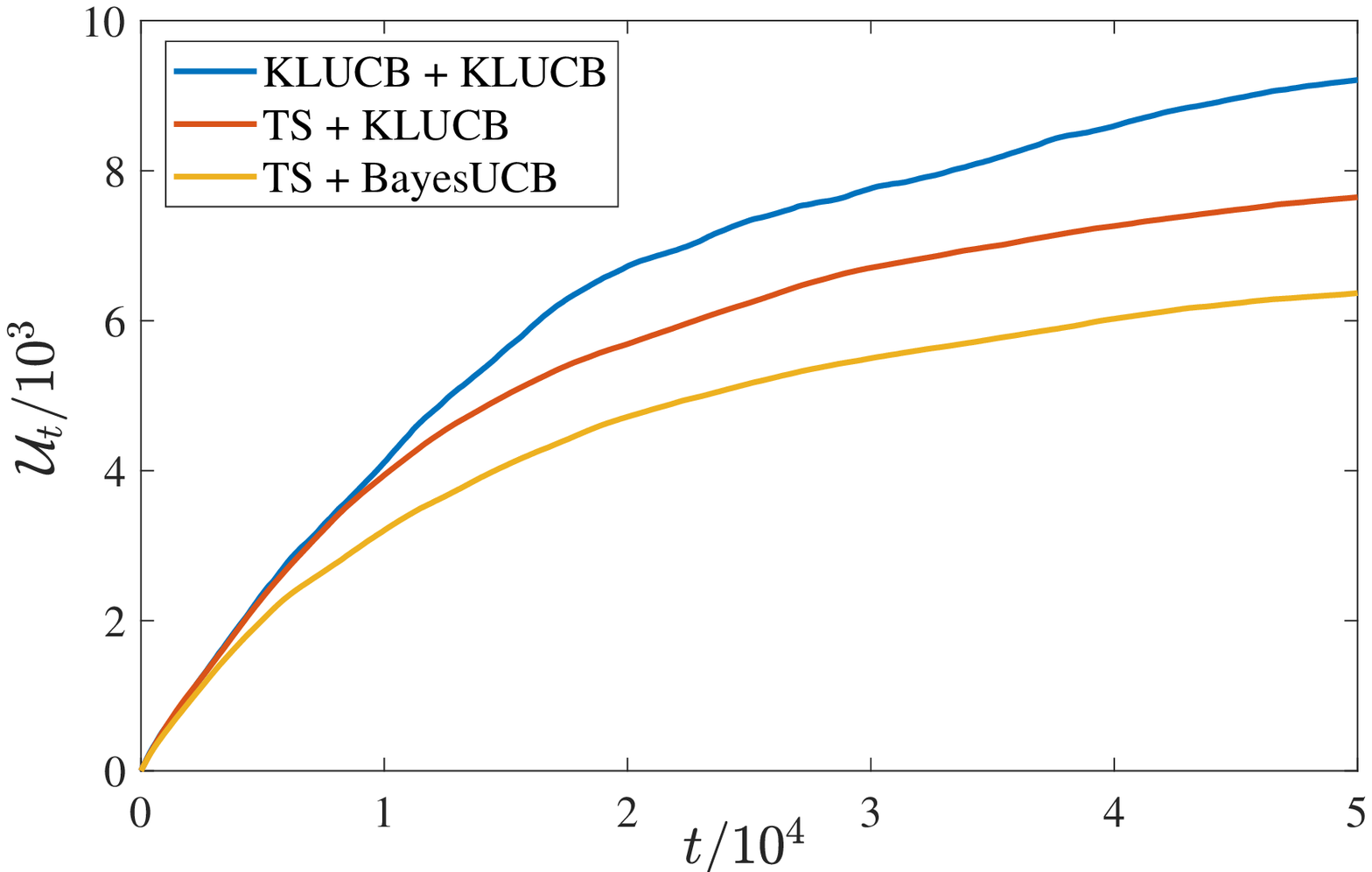}
    \includegraphics[width = 0.49\textwidth]{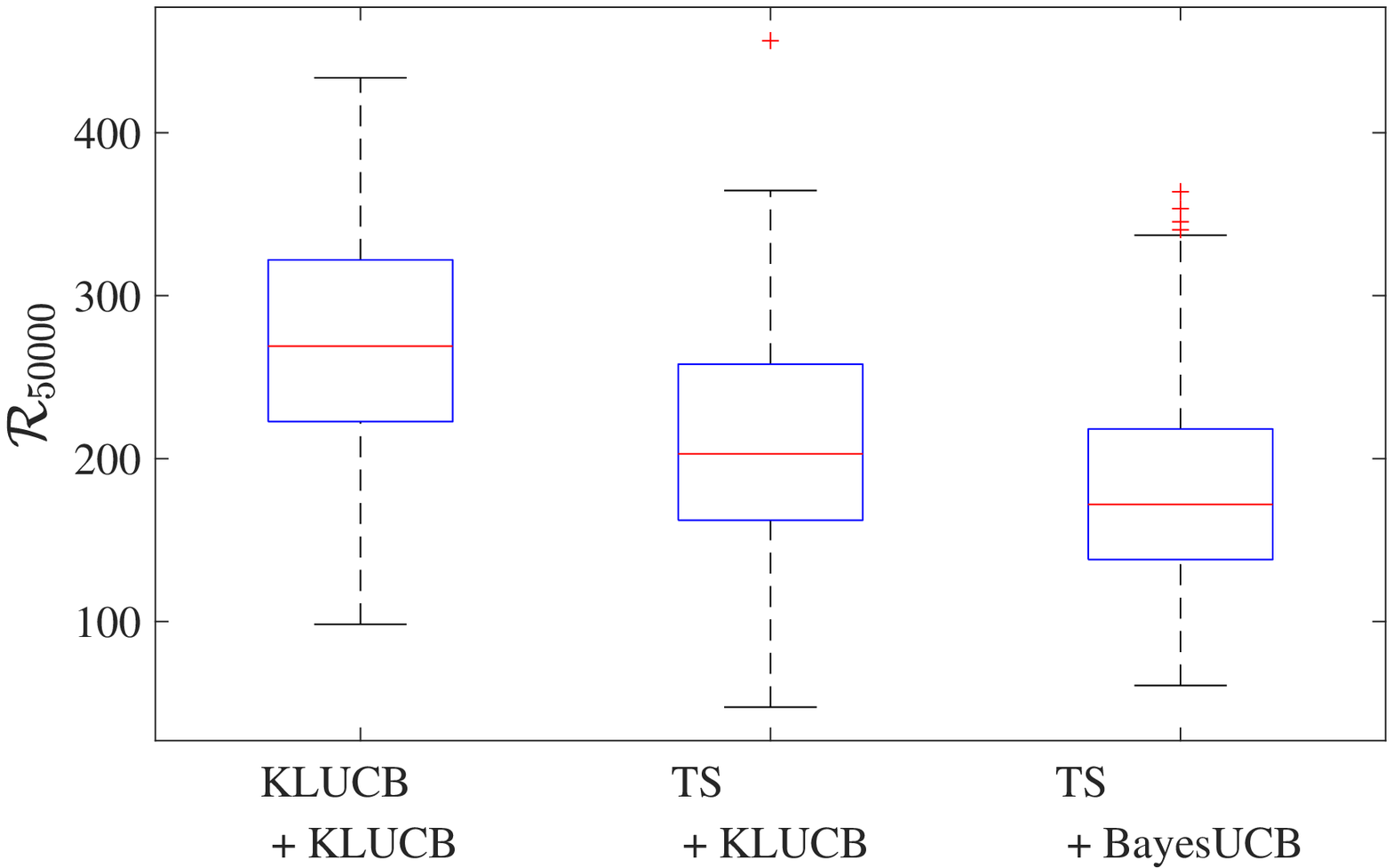}
    \includegraphics[width = 0.49\textwidth]{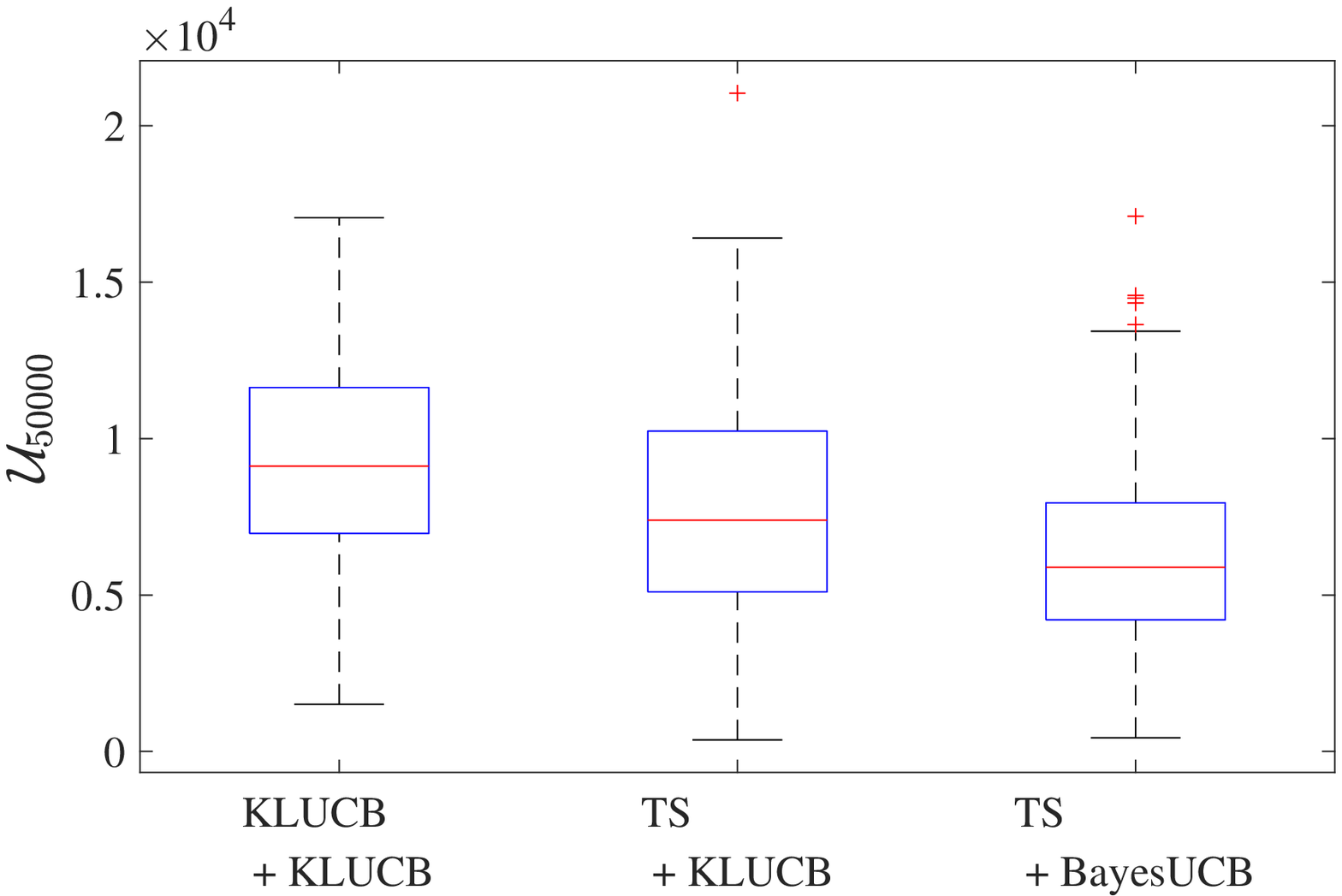}
    \caption{Top: Mean regret (left) and safety (right) violation as a function of $t$, averaged over 100 trials, for the Trial Drug data with $\alpha = 0.19$ Bottom: Box plot of the same at $T = 50K$. }
    \label{fig:alpha_0_19}
\end{figure}

\subsection{Comparing theoretically analysed \BayesUCB quantiles with the practically implemented ones}\label{appx:bayesUCB}

As observed in the main text, the simulation of \S\ref{sec:sims} all present Algorithm \ref{alg:TBU} run with the quantile schedule $\delta_t^k = 1/t$ - in actuality, we use the slightly more reasonable schedule of $\min(\alpha/2, t^{-1}),$ simply to ensure that for small $t$, all arms are declared as feasible. While this choice is consistent with the recommendation of \citet{kaufmann2012bayesUCB}, it differs from the schedule analysed theoretically in \S\ref{sec:TUB}, which instead suggsets $\delta_t^k = (\sqrt{8N_t^k} t\log^3 t)^{-1}).$ We present the behaviour of such a schedule below, although we modify it slightly to $\min(\alpha/2, (\sqrt{8 N_t^k} t)^{-1})$ - here we drop the $\log$ term as recommended by \citet{garivier2011kl}, and introduce the minimum to again ensure that for small $t$ all arms are declared to be feasible. The resulting behaviour is compared with the previously studied $1/t$ schedule in Figure \ref{fig:probing_bayesucb} on the simple data \( \mu = \nu = (0.4,0.5,0.6),  \alpha = 0.5.\) 

Observe that the theoretical schedule displays the favourable logarithmic growth, and so is consistent with Theorem \ref{thm:TS+BayesUCB}. Further, while it certainly suffers degradation relative to the $1/t$ schedule, this is limited. The reason for this degradation is largely because the theoretical lower indices $L_t^k$ are more optimistic, and allow the unsafe arm to be played for a larger number of times, as borne out in the plot of total safety violations.

\begin{figure}[htb]
    \centering
    \includegraphics[width = 0.49 \textwidth]{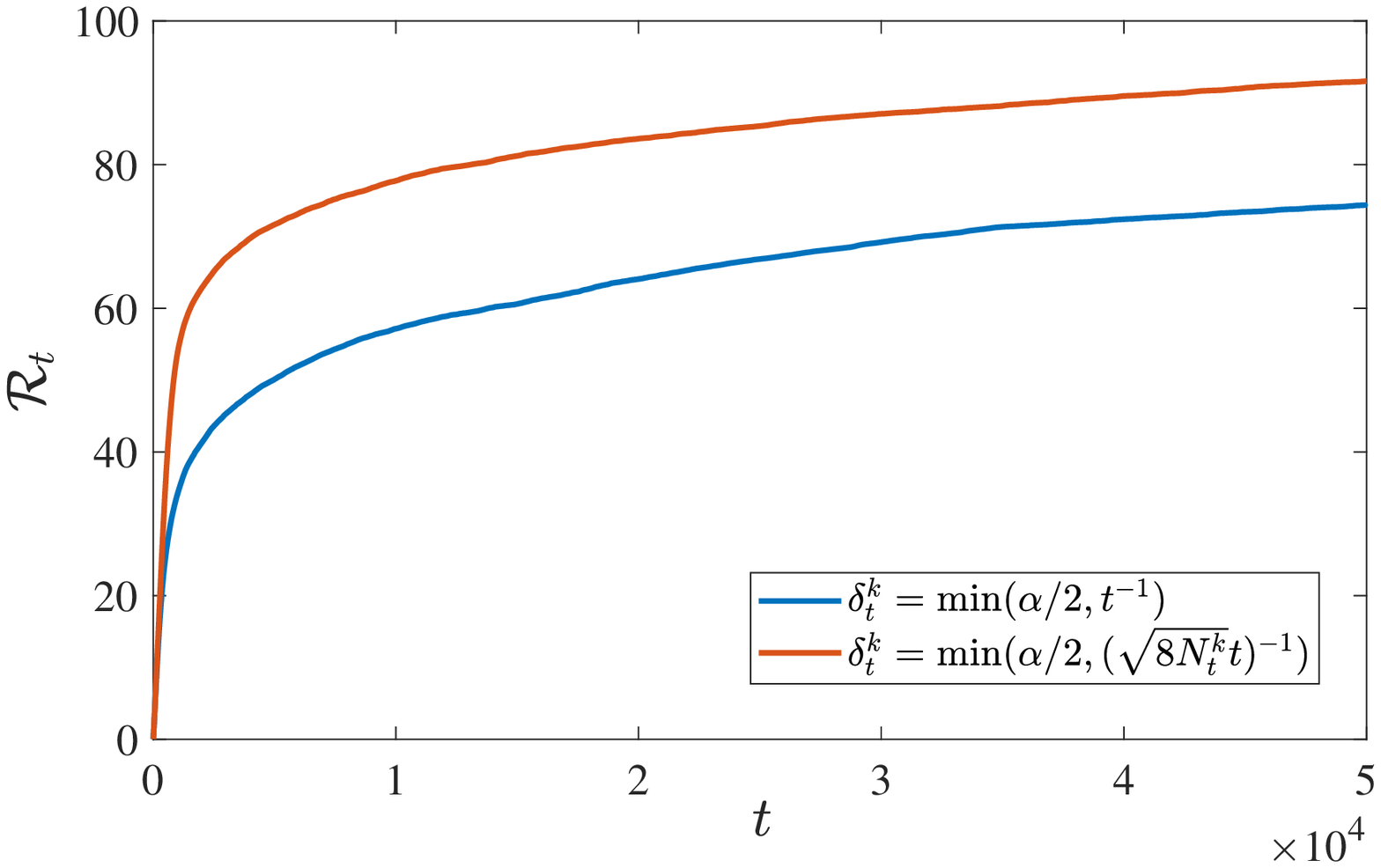}
    \includegraphics[width = 0.49\textwidth]{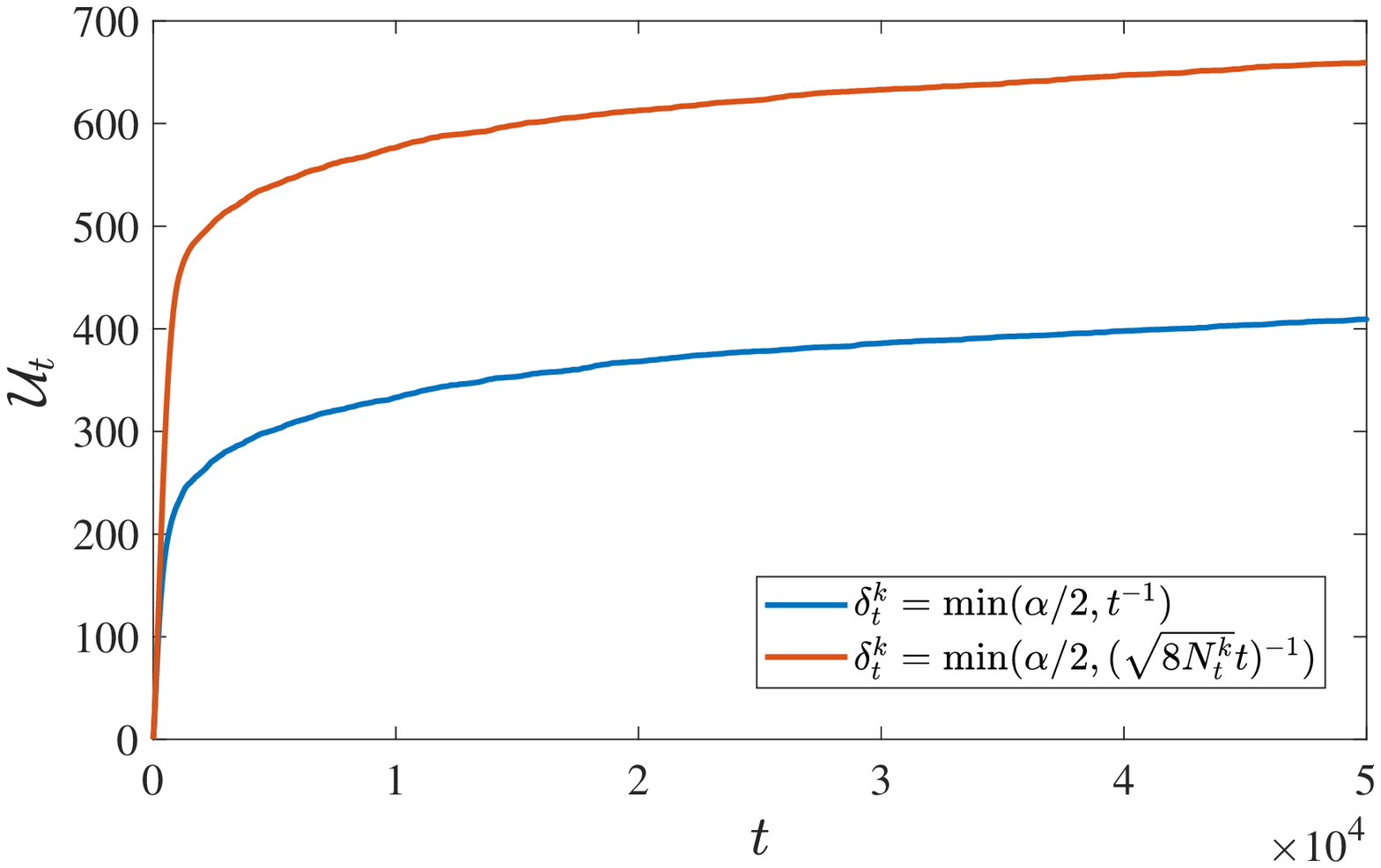}
    \caption{Regret (left) and total safety violations (right) of the theoretical and the $1/t$ schedule for Algorithm \ref{alg:TBU}. Averages over 500 trials are presented.}
    \label{fig:probing_bayesucb}
\end{figure}

Additionally, Figure \ref{fig:probing_bayesucb_boxplots} presents boxplots of the regret and net safety at $T = 50000$ for the two schedules. An interesting observation is that the schedule $1/t$ exhibits greater variability, with some (rare) but massive outliers that are not present for the theoretical schedule. Investigating this more closely requires determining high-probability bounds on these methods, which is a subject for future work.

\begin{figure}[H]
    \centering
    \includegraphics[width = 0.49\textwidth]{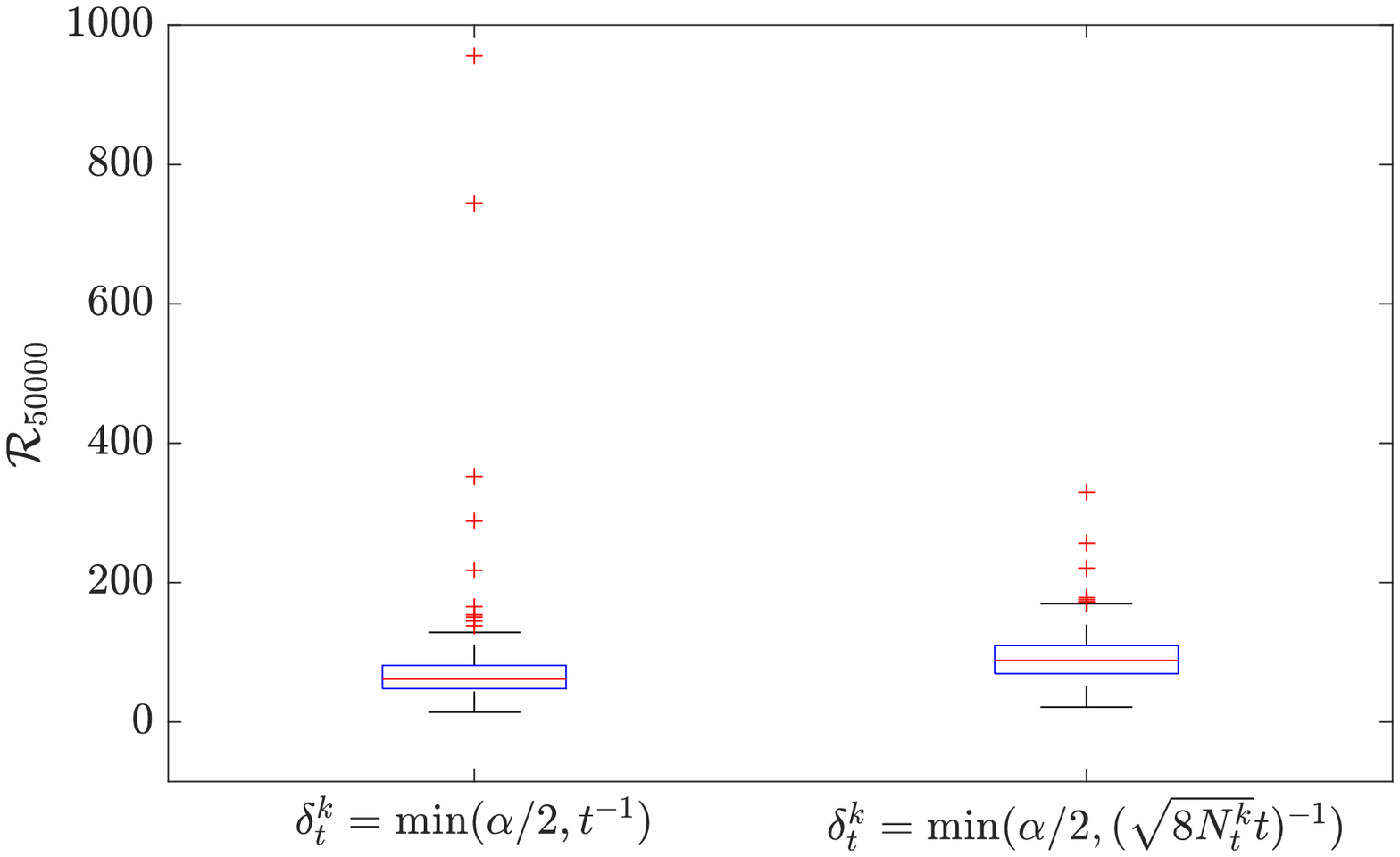}
    \includegraphics[width = 0.49\textwidth]{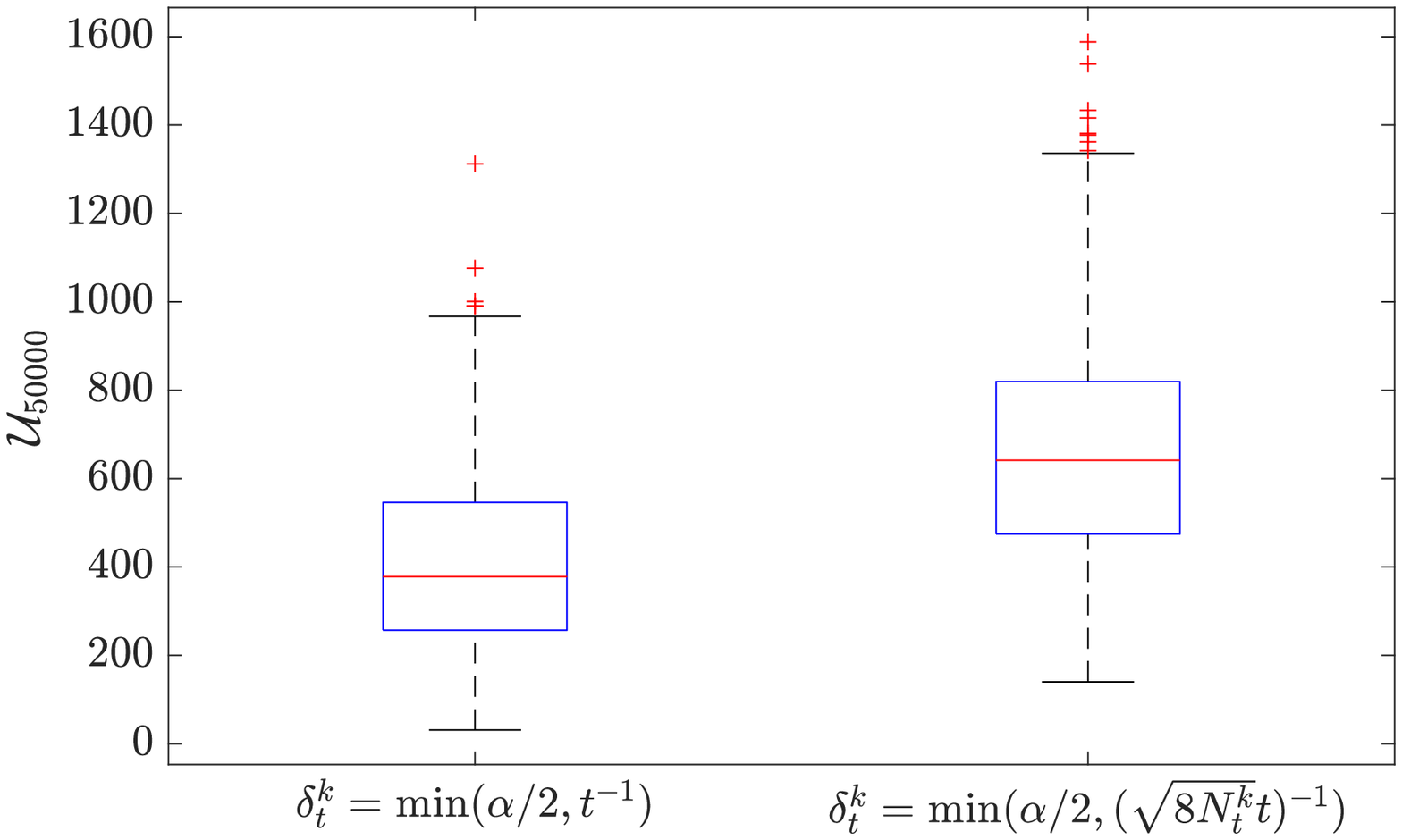}
    \caption{Boxplots across 500 trials of the regret (left) and safety violations (right) for the two schedules at $T = 50000$. One outlier for regret for the $1/t$ schedule at $\approx 2100$ has been omitted for the sake of clarity.}
    \label{fig:probing_bayesucb_boxplots}
\end{figure}

\subsection{The behaviour of a Na\"{i}ve Thompson Sampling Based Safety Index}\label{appx:exp_ts_with_no_slack}

As discussed in \S\ref{sec:TUB}, a na\"{i}ve way of constructing a safety index by just sampling $\theta_t^k \sim \mathrm{Beta}(S_t^k+1, N_t^k -S_t^k + 1)$ should be ineffective when the safety score $\nu^*$ is close to $\alpha.$ We first investigate this effect.

Concretely, the scheme is the same as Alg.\ref{alg:TBU}, except that instead of the \BayesUCB index, we construct a safety index by sampling as above, and then populate $\Perm_t = \{ k : \theta_t^k \le \alpha\}.$ We run this scheme with the data \begin{align*}
    \mu = (0.3, 0.5, 0.7),\\
    \nu = (0.3, 0.5, 0.7),
\end{align*}

and \emph{vary} $\alpha$ as $0.5 + i/50$ for $i \in [0:9]$. This corresponds to an increasing safety slack, while for $i \in [0:5],$ the safety gap of the unsafe arm $3$ remains large, but decaying. Note that this ostensibly should increase the large $t$ regret for a scheme with optimal dependence.

Figure \ref{fig:naive_ts} plots the resulting mean regrets over a horizon of length 10K for four of the 10 cases (chosen evenly to not clutter the plot too much). The data is averaged over 200 trials. Observe that for $i = 0,$ wherein the gap $(\alpha - \nu^*)$ is $0$, the regret grows linearly, while the dependence becomes sublinear as $i$ increases, and further improves, even though it should grow like $1/i$. 

\begin{figure}[htb]
    \centering
    \includegraphics[width = 0.49\textwidth]{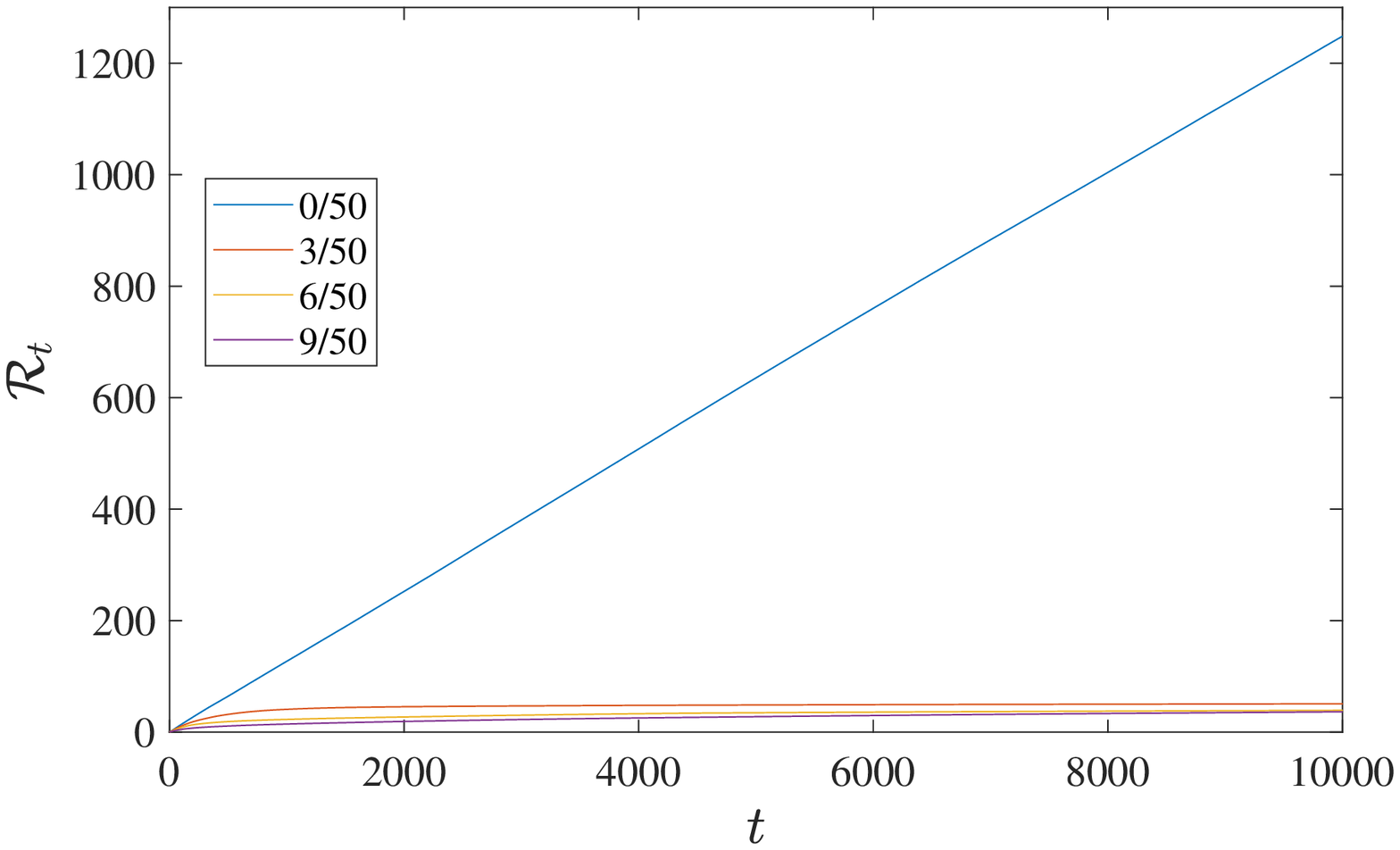}
    \includegraphics[width = 0.49\textwidth]{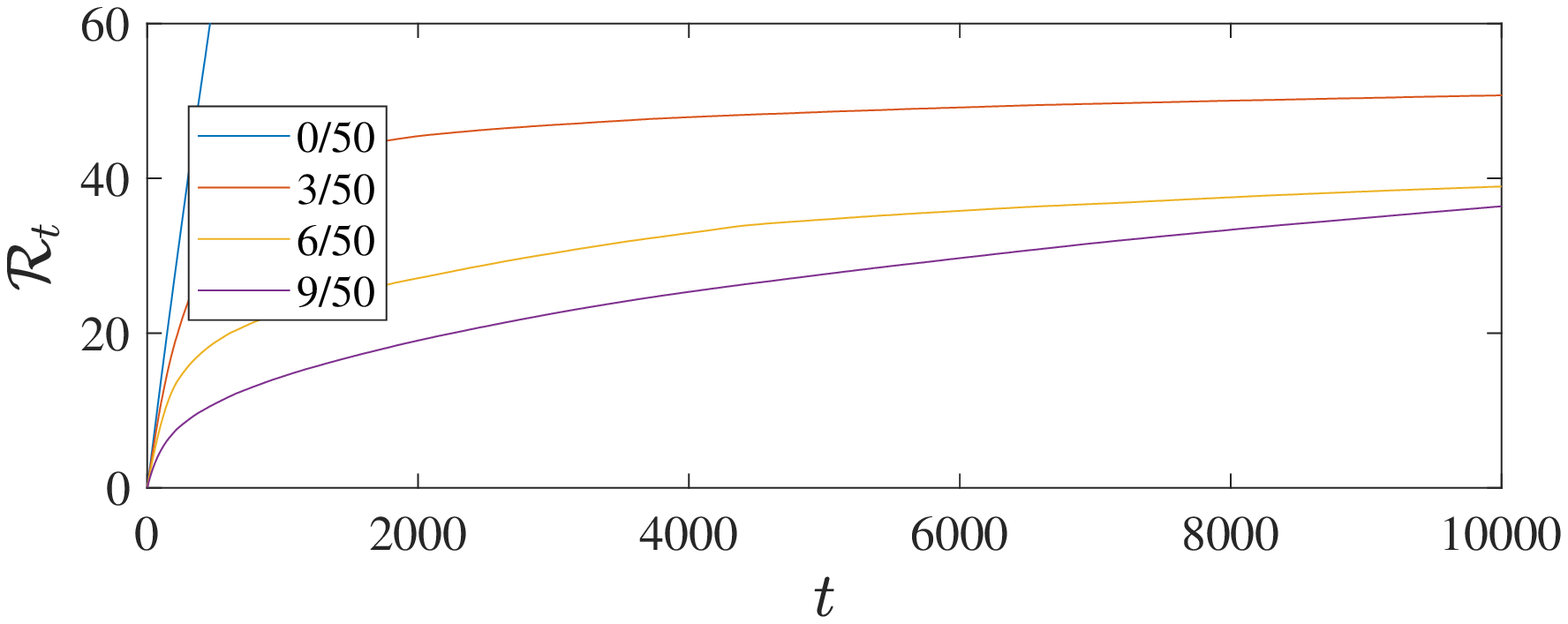}
    \caption{Regret of the scheme with a Na\"{i}ve \Thomp based safety index for various choices of $\alpha$. The legend marks $\alpha - \nu^*$. The right figure zooms in to the bottom of the left figure.}
    \label{fig:naive_ts}
\end{figure}

Further, we observe that the dependence on $\alpha - \nu^*$ scales roughly as inverse-quadratic. This is illustrated in Figure \ref{fig:dependence_on_alpha_gap}, which plots both the mean regret against $\alpha - \nu^*$, as well as the mean of $1/\sqrt(\mathcal{R}_t)$ against $\alpha-\nu^*$. The key observation is the nearly linear dependence in the second plot for small $\alpha - \nu^*$. This observation makes sense - the variance scale of a $\mathrm{Beta}(S+1, N - S+ 1)$ distribution is as $1/N,$ and so if the means $\hat{\nu}^*$ is close to the truth, then the chance of $\theta_t^k$ falling above $\alpha$ at time $t$ is roughly $1/t (\alpha - \nu^*)^2$, and so $k^* \not\in\Perm_t$ for about $\log(T)/(\alpha - \nu^*)^2 $ rounds. Of course, for large enough $\alpha - \nu^*,$ this term is dominated by the regret terms due to suboptimal arms, and the dependence is masked. This effect is further confounded in our simulation with the fact that the safety gap $\Gamma^3$ reduces as $\alpha$ is increased, which raises the regret. Nevertheless, the trend is evident, at least in the low $\alpha - \nu^*$ regime where the gap $\Gamma^3$ does not change as much, and remains much larger than $\alpha - \nu^*$.

\begin{figure}[htb]
    \centering
    \includegraphics[width = 0.49\textwidth]{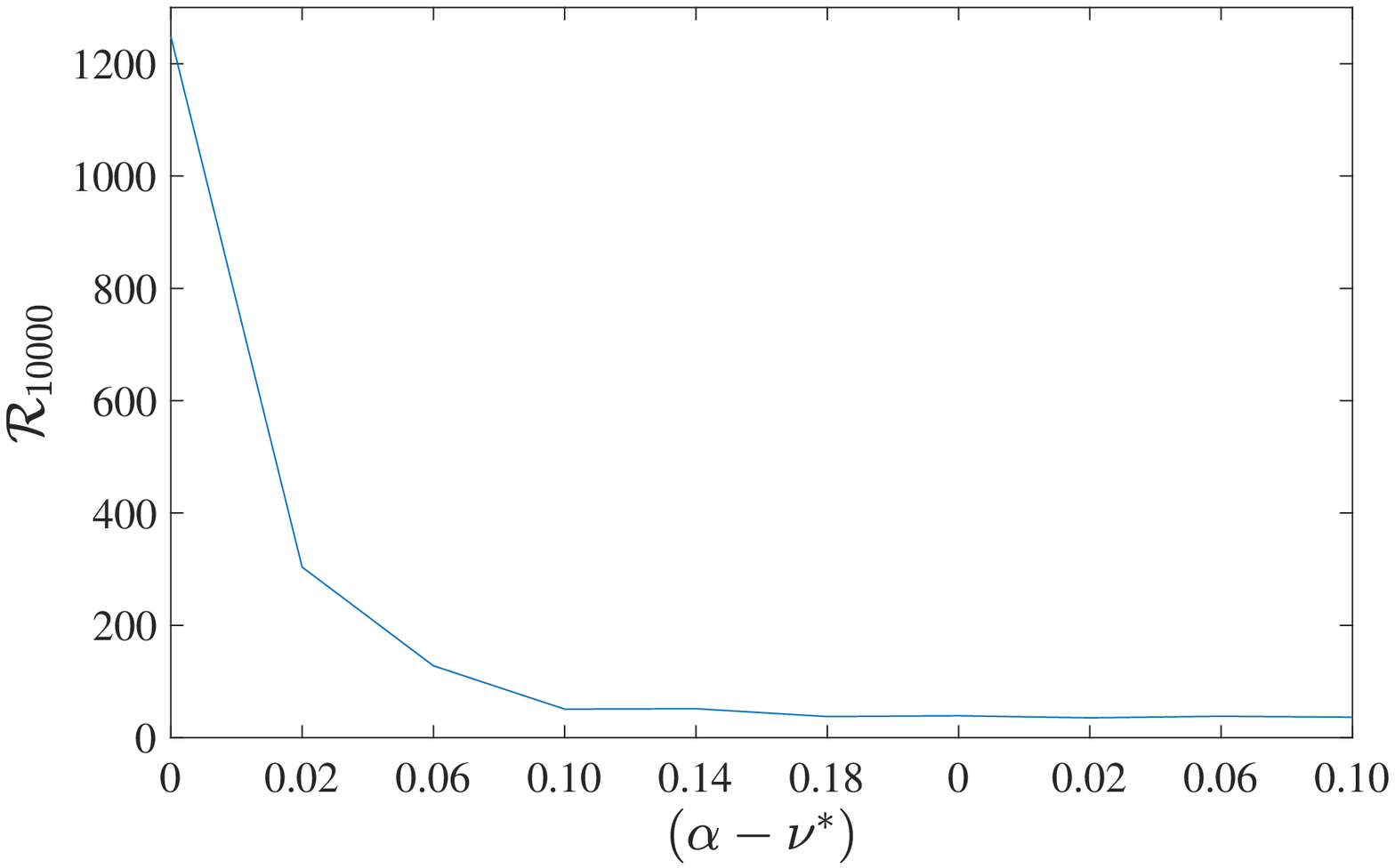}
    \includegraphics[width = 0.49\textwidth]{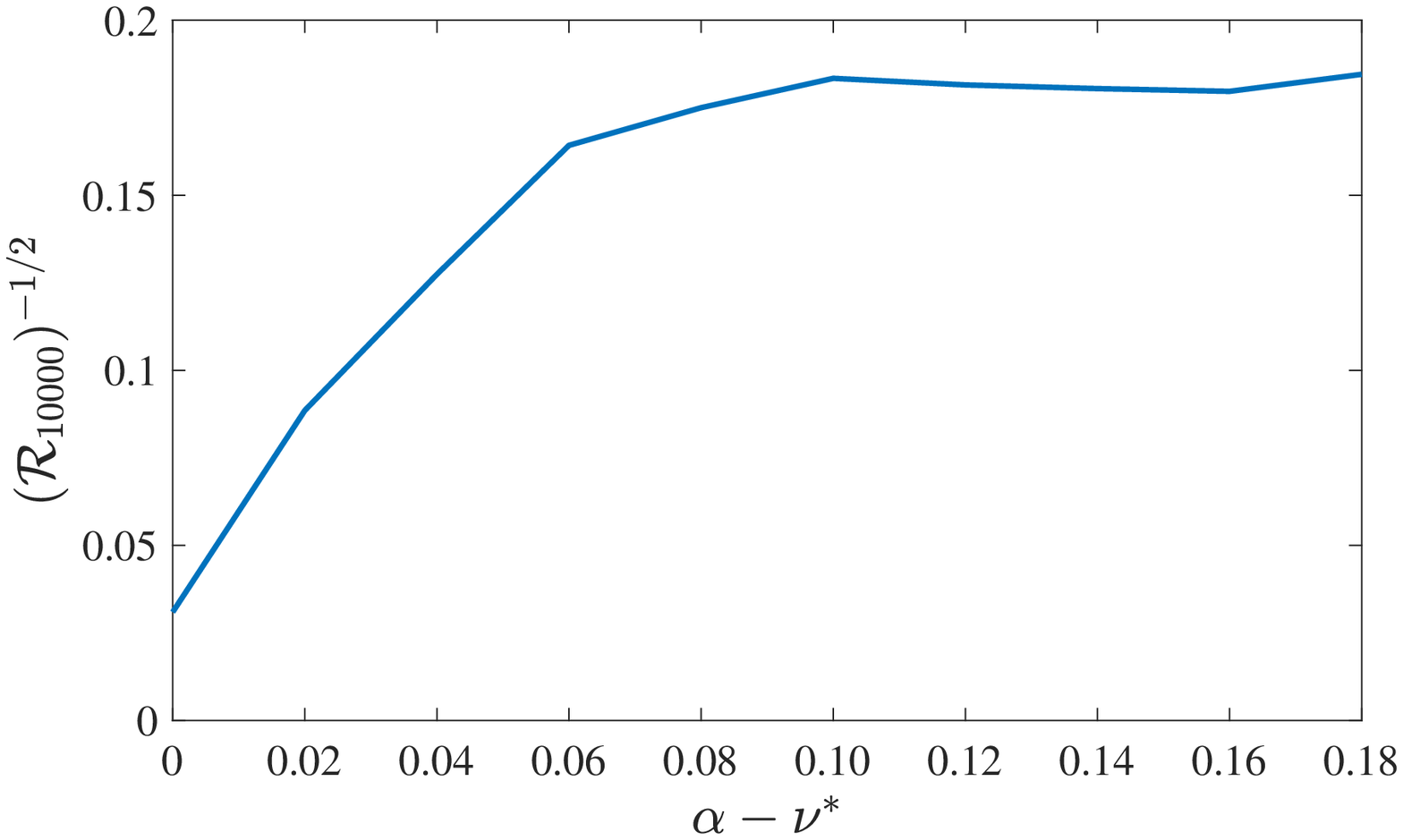}
    \caption{The mean of regret and mean of $1/\sqrt{\mathcal{R}_t}$ over 200 trials as $\alpha$ is varied, plotted against the gap of the optimal arm from the boundary, $\alpha - \nu^*.$}
    \label{fig:dependence_on_alpha_gap}
\end{figure}

Despite the ineffectiveness when $\alpha - \nu^*$ is small, a \Thomp based safety index is an attractive proposition, primarily due to wider concerns - the advantage of \Thomp for standard bandits is obtaining strong regret performance at a low computational cost, and this is specially important in cases such as combinatorial or continuously armed bandits. An alternative sampling based strategy would enable such an approach for safe bandits in such rich scenarios, and is of both practical and theoretical interest. Promisingly, when the gap is large, the effect on regret is indeed mild, showing that this is the only obstacle in the path of such a strategy. 

One natural approach to address this obstacle is to allow a slack in the safety criterion for \Thomp - we may sample $\theta_t^k$ according to the safety posterior, and then instantiate $\Pi_t = \{k: \theta_t^k \le \alpha + \varepsilon_t^k\},$ where $\varepsilon_t^k$ serves as a slack. This raises a design question of how to choose this slack. We empirically investigate the choice of slack $C \mathrm{Dev}_t^k  \sqrt{\log t},$ where $\mathrm{Dev}_t^k$ is the standard deviation of the safety posterior or arm $k$ at time $t$. This choice is natural, since this variance determines the scale of fluctuations of the score itself. Figure \ref{fig:naive_ts_safety_slack} shows the behaviour obtained as we set $C = 2^i$ for $i \in [-3:3]$ for the same data as before, but now with fixed $\alpha = 0.5.$ 

This plot, while very preliminary, shows an interesting effect in that values of $C \ge 1/2$ again result in large, linear regret. Recall that $C = 0,$ which corresponds to no slack, also gives linear regret. It is unclear how robust this effect is, but if true, this observation suggests that tuning this $C$ properly is a subtle problem, and the behaviour is quite sensitive to it, which raises an interesting challenge for further work. 

\begin{figure}[t]
    \centering
    \includegraphics[width = 0.49\textwidth]{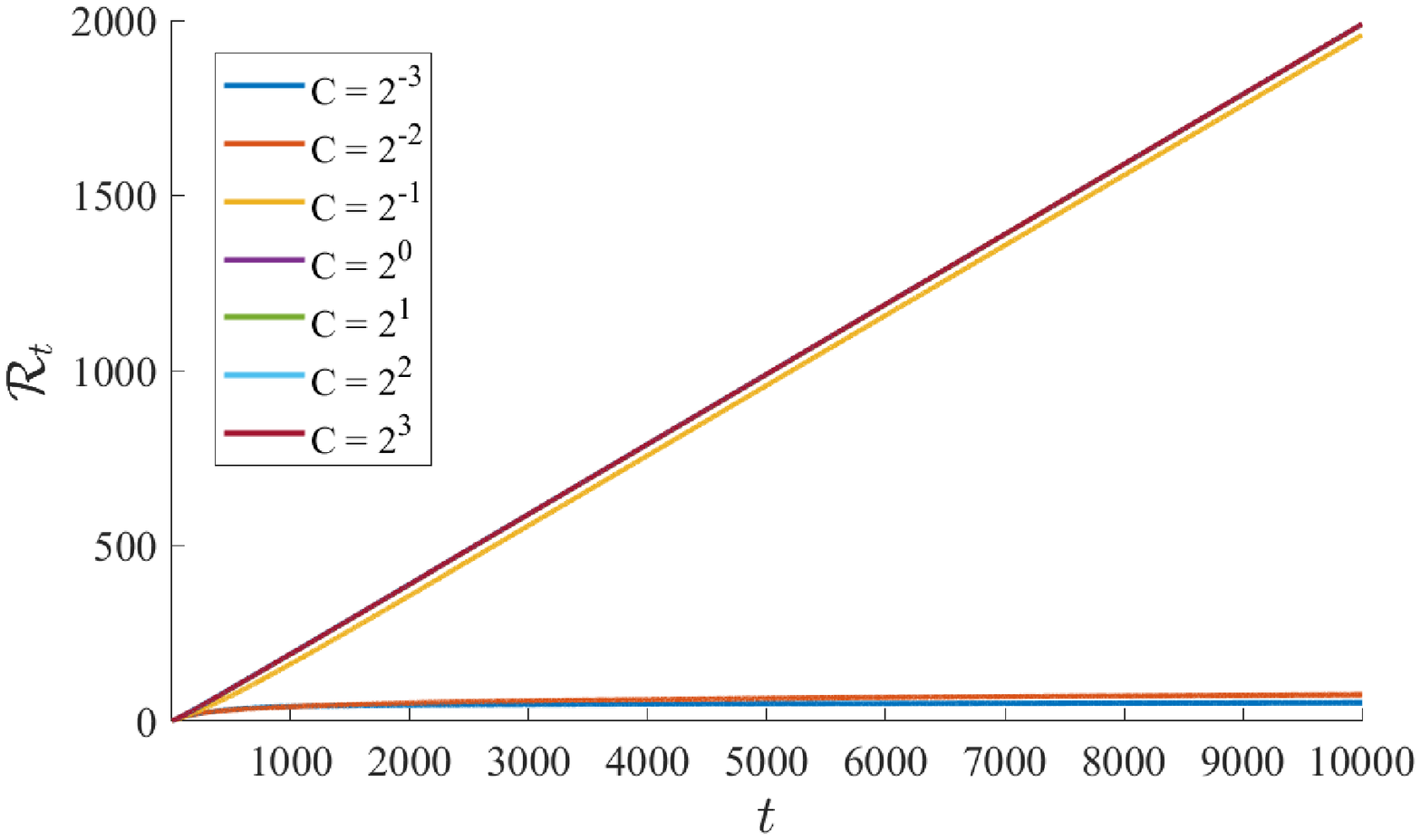}
    \includegraphics[width = 0.49\textwidth]{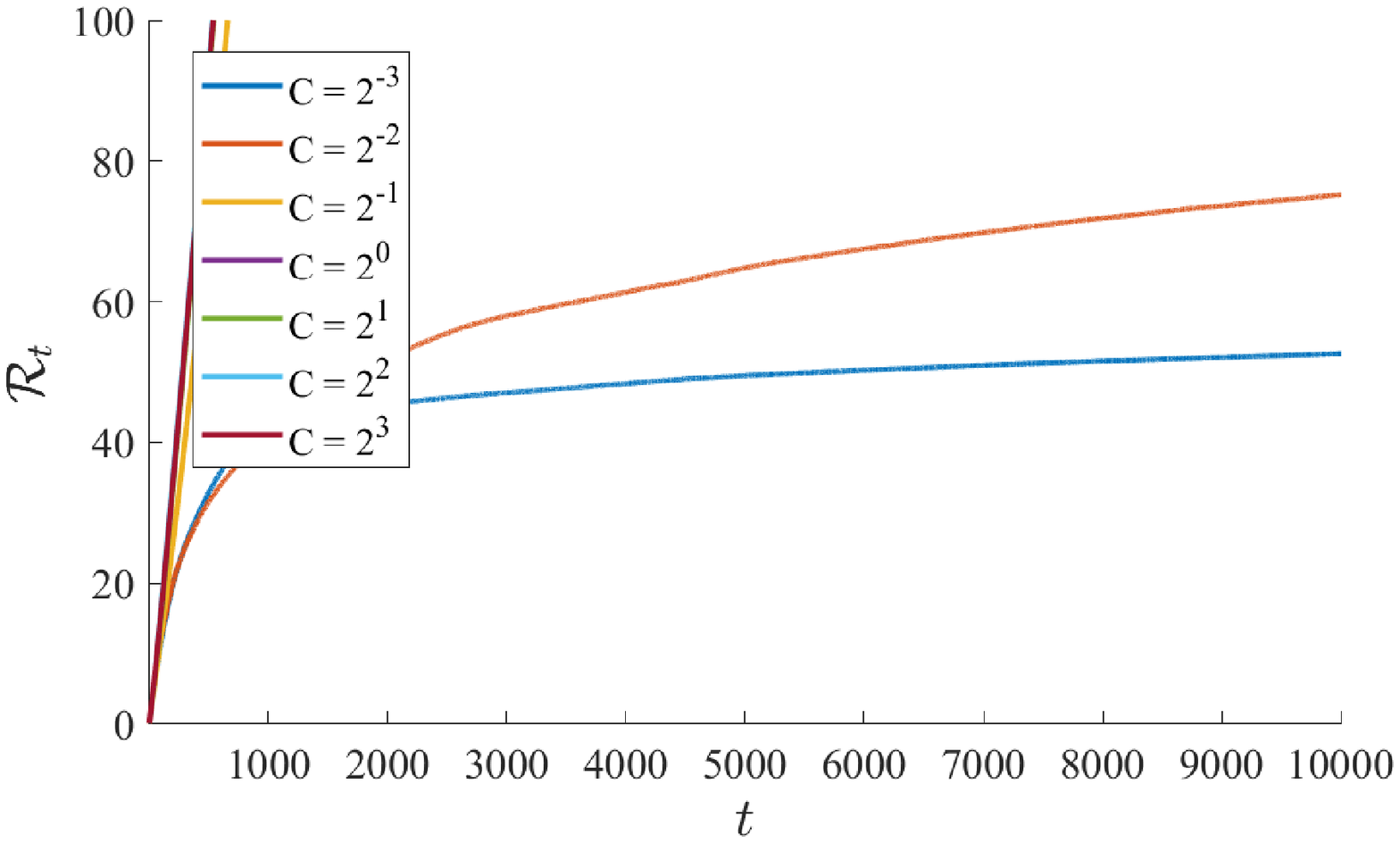}
    \caption{Regret performance as the slack factor $C$ is varied. Right zooms into the bottom half of the left plot. Merans over 200 trials are reported.}
    \label{fig:naive_ts_safety_slack}
\end{figure}

\end{document}